\documentclass{article}
\usepackage[preprint]{neurips_2025}

\usepackage[utf8]{inputenc} %
\usepackage[T1]{fontenc}    %
\usepackage{url}            %
\usepackage{booktabs}       %
\usepackage{amsfonts}       %
\usepackage{nicefrac}       %
\usepackage{microtype}      %
\usepackage{xcolor}         %

\usepackage[pagebackref]{hyperref}    %
\setcitestyle{authoryear,round,citesep={,},aysep={},yysep={;}}
\hypersetup{colorlinks=true,linkcolor=black,citecolor=black,filecolor=black,urlcolor=blue}
\renewcommand*{\backref}[1]{}
\renewcommand*{\backrefalt}[4]{
\ifcase #1
  No citations.
\or
  (p. #2.)
\else
  (pp. #2.)
\fi}

\usepackage{amssymb}
\usepackage{amsmath}
\usepackage{amsthm}
\usepackage{enumitem}
\usepackage[capitalize,noabbrev]{cleveref} %
\setitemize{noitemsep,topsep=0pt,parsep=0pt,partopsep=0pt,leftmargin=20pt}
\usepackage{graphicx}
\usepackage{bbding} %
\usepackage{wrapfig} %
\usepackage{subcaption} %
\usepackage{booktabs} %
\usepackage{diagbox}
\usepackage{algpseudocode}
\usepackage{changepage} %

\usepackage{amsmath,amsfonts,bm}

\def\eqref#1{equation~\ref{#1}}

\def\1{\bm{1}}

\def\rvx{{\mathbf{x}}}

\def\rmV{{\mathbf{V}}}

\def\vtheta{{\bm{\theta}}}

\def\vb{{\bm{b}}}

\def\vu{{\bm{u}}}

\def\vw{{\bm{w}}}
\def\vx{{\bm{x}}}

\DeclareMathAlphabet{\mathsfit}{\encodingdefault}{\sfdefault}{m}{sl}
\SetMathAlphabet{\mathsfit}{bold}{\encodingdefault}{\sfdefault}{bx}{n}

\newcommand{\R}{\mathbb{R}}

\newcommand{\sigmoid}{\sigma}

\newcommand{\norm}[1]{\|#1\|}

\newcommand{\abs}[1]{\left|#1\right|}

\newcommand{\relu}{\operatorname{ReLU}}
\newcommand{\elu}{\operatorname{ELU}}
\renewcommand{\sigmoid}{\operatorname{Sigmoid}}
\renewcommand{\tanh}{\operatorname{Tanh}}
\newcommand{\elephant}{\operatorname{Elephant}}
\newcommand{\fta}{\operatorname{FTA}}
\newcommand{\lwta}{\operatorname{LWTA}}
\newcommand{\maxout}{\operatorname{Maxout}}
\newcommand{\ntk}{\operatorname{NTK}}
\newcommand{\srnn}{\operatorname{SR-NN}}

\usepackage{thmtools,thm-restate}

\newtheorem{property}{Property}
\newtheorem{lemma}{Lemma}

\newtheorem{definition}{Definition}

\newtheorem{remark}{Remark}

\newcommand{\figwidthone}{\textwidth}
\newcommand{\figwidthtwo}{0.48\textwidth}
\newcommand{\figwidththree}{0.32\textwidth}

\title{Efficient Reinforcement Learning by Reducing Forgetting with Elephant Activation Functions}

\author{
Qingfeng Lan \\
Department of Computing Science\\
University of Alberta\\
\texttt{qlan3@ualberta.ca} \\
\And
Gautham Vasan \\
Department of Computing Science\\
University of Alberta\\
\texttt{vasan@ualberta.ca} \\
\And
A. Rupam Mahmood \\
Department of Computing Science \\
University of Alberta \\
CIFAR AI Chair, Amii \\
\texttt{armahmood@ualberta.ca} \\
}

\begin{document}
\maketitle
\begin{abstract}
Catastrophic forgetting has remained a significant challenge for efficient reinforcement learning for decades~\citep{ring1994continual,rivest2003combining}.
While recent works have proposed effective methods to mitigate this issue, they mainly focus on the algorithmic side.
Meanwhile, we do not fully understand what architectural properties of neural networks lead to catastrophic forgetting.
This study aims to fill this gap by studying the role of activation functions in the training dynamics of neural networks and their impact on catastrophic forgetting in reinforcement learning setup.
Our study reveals that, besides sparse representations, the gradient sparsity of activation functions also plays an important role in reducing forgetting.
Based on this insight, we propose a new class of activation functions, \emph{elephant activation functions}, that can generate both sparse outputs and sparse gradients.
We show that by simply replacing classical activation functions with elephant activation functions in the neural networks of value-based algorithms, we can significantly improve the resilience of neural networks to catastrophic forgetting, thus making reinforcement learning more sample-efficient and memory-efficient.~\footnote{Code release: \url{https://github.com/qlan3/ENN}}
\end{abstract}

\section{Introduction}

One of the greatest challenges to achieving efficient learning is the decades-old issue of \emph{catastrophic forgetting}~\citep{french1999catastrophic}.
Catastrophic forgetting stands for the phenomenon that, when used with backpropagation, artificial neural networks tend to forget prior knowledge drastically, which is commonly encountered in both continual supervised learning~\citep{hsu2018re,farquhar2018towards,van2019three,delange2021continual} and reinforcement learning (RL)~\citep{schwarz2018progress,atkinson2021pseudo,khetarpal2022towards}.
By its very nature, RL requires continual learning abilities even in single-task setups.
Just like continual supervised learning problems, single-task RL faces non-stationarity due to value bootstrapping, policy iterations, and the resulting shifts in data distributions~\citep{cahill2011catastrophic,rusu2016progressive,lan2023memory,vasan2024deep,elsayed2024streaming}.
Without continual learning abilities, an RL agent would override previously learned skills during later training, resulting in deteriorating performance and inefficient learning~\citep{ghiassian2020improving,pan2020fuzzy}.

In recent years, researchers have made significant progress in mitigating catastrophic forgetting and proposed many effective methods, such as replay methods~\citep{mendez2022modular}, regularization-based methods~\citep{riemer2018learning}, parameter-isolation methods~\citep{mendez2020lifelong}, and optimization-based methods~\citep{farajtabar2020orthogonal}. 
Most of these methods tackle the forgetting problem from the algorithmic approach while the other approach, alleviating forgetting by designing neural networks with specific properties, is less explored.
There is still a lack of full understanding of what properties of neural networks lead to catastrophic forgetting.
Recently, \citet{mirzadeh2022wide} find that the width of a neural network significantly affects forgetting and provide explanations from the perspectives of gradient orthogonality, gradient sparsity, and lazy training regime.
Furthermore, \citet{mirzadeh2022architecture} study the forgetting issue on large-scale benchmarks with various network architectures, demonstrating that architectures can also play an important role in reducing forgetting.

The interaction between catastrophic forgetting and neural network architectures remains under-explored.
In this work, we aim to better understand this interaction by studying the impact of various architectural choices of neural networks on catastrophic forgetting in single-task RL, where many of the continual learning issues appear.
Specifically, we focus on activation functions, one of the most important elements in neural networks.
Theoretically, we investigate the role of activation functions in the training dynamics of neural networks.
Experimentally, we study the effect of various activation functions on catastrophic forgetting under the setting of continual supervised learning.
These results suggest that not only sparse representations but also sparse gradients are essential for mitigating forgetting.
Based on this discovery, we develop a new type of activation function called \textit{elephant activation functions}.
They can generate sparse function values and gradients, thus enhancing neural networks' resilience to catastrophic forgetting.
To verify the effectiveness of our method, we substitute classical activation functions with elephant activation functions in agent neural networks and test these RL agents across a range of tasks, from basic Gymnasium tasks~\citep{towers2023gymnasium} to intricate Atari games~\citep{mnih2013playing}.
The experimental results demonstrate that, under extreme memory constraints, integrating elephant activation functions into RL agents alleviates the forgetting issue, leading to comparable or improved performance while significantly enhancing memory efficiency.
Moreover, we show that even when large replay buffers are utilized, applying elephant activation functions remains beneficial, substantially boosting the learning performance.

\section{Related Work}\label{sec:related}
\paragraph{Architecture-Based Continual Learning}
\citet{mirzadeh2022wide,mirzadeh2022architecture} are particularly notable for studying the effect of network architectures on continual learning.
Several approaches involve the selection and allocation of a subset of weights in a network for each task~\citep{ammar2014online,mallya2018packnet,sokar2021spacenet,fernando2017pathnet,serra2018overcoming,masana2021ternary,li2019learn,yoon2018lifelong,mendez2020lifelong}, or the allocation of a specific network to each task~\citep{rusu2016progressive,aljundi2017expert}.
Some methods expand networks dynamically~\citep{yoon2018lifelong,hung2019compacting,ostapenko2019learning} as training continues.
Inspired by biological neural circuits, \citet{shen2021algorithmic,bricken2023sparse}, and \citet{madireddy2023improving} propose novel networks which generate sparse representations by design.
Finally, our method is closely related to sparse activation functions, such as fuzzy tiling activation ($\fta$)~\citep{pan2020fuzzy}, $\maxout$~\citep{goodfellow2013maxout}, and local winner-take-all ($\lwta$)~\citep{srivastava2013compete}.
Specifically, inspired by tile coding~\citep{sutton2011reinforcement}, $\fta$ maps a scalar to a vector with a controllable sparsity level.
$\maxout$ outputs the maximum value across $k$ input features.
$\lwta$ is similar to $\maxout$ with a slight difference: while $\maxout$ only selects and outputs maximum values, $\lwta$ selects maximum values, sets others to zeros, and then outputs them together.
Compared to these activation functions, $\elephant$ is simpler to implement and can generate both sparse representations and gradients, as supported by our theoretical and experimental results.

\paragraph{Sparsity in Deep Learning}
Sparse representations are known to help reduce forgetting for decades~\citep{french1992semi}.
In supervised learning, dynamic sparse training~\citep{dettmers2019sparse,liu2020dynamic,sokar2021spacenet}, dropout variants~\citep{srivastava2013compete,goodfellow2013maxout,mirzadeh2020dropout,abbasi2022sparsity,sarfraz2023sparse}, and pruning methods~\citep{guo2016dynamic,frankle2019lottery,blalock2020state,zhou2020go,wang2022sparcl} are shown to speed up training and improve generalization.
Additionally, \citet{lee2021gst,sokar2022dynamic}, and \citet{tan2023rlx2} propose dynamic sparse training approaches for RL which achieve comparable or improved performance, higher sample efficiency, and better computation efficiency.
Furthermore, \citet{le2017learning} and \citet{liu2019utility} show that sparse representations stabilize training and improve performance in RL.
Finally, \citet{ceron2024value} demonstrates that increasing network sparsity by gradual magnitude pruning improves the learning performance of value-based RL agents.
Our approach differs from them in its simplicity and effectiveness---by simply replacing classical activation functions with $\elephant$, we can enhance the efficiency of training RL agents.

\paragraph{Local Elasticity and Memorization}
\citet{he2020local} propose the concept of local elasticity.
\citet{chen2020label} introduce label-aware neural tangent kernels, showing that models trained with these kernels are more locally elastic.
\citet{mehta2021extreme} prove a theoretical connection between the scale of network initialization and local elasticity, demonstrating extreme memorization using large initialization scales.
Incorporating Fourier features in the input of a network also induces local elasticity, which is greatly affected by the initial variance of the Fourier basis~\citep{li2021functional}.

\section{Investigating Catastrophic Forgetting via Training Dynamics}

First, we look into the forgetting issue via the training dynamics of neural networks.
For simplicity, we consider a regression task.
Let a scalar-valued function $f_{\vw}(\vx)$ be represented as a neural network, parameterized by $\vw$, with input $\vx$.
$F(\vx)$ is the true function and the loss function is $L(f,F,\vx)$. For example, for squared error, we have $L(f,F,\vx) = (f_{\vw}(\vx) - F(\vx))^2$.
At each time step $t$, a new sample $\{\vx_t, F(\vx_t)\}$ arrives. Given this new sample, to minimize the loss function $L(f,F,\vx_t)$, we update the weight vector by $\vw' = \vw + \Delta_\vw$ where $\Delta_\vw$ is the weight difference.
With the stochastic gradient descent (SGD) algorithm, we have $\vw' = \vw - \alpha \nabla_\vw L(f,F,\vx_t)$, where $\alpha$ is the learning rate.
So $\Delta_\vw = \vw' - \vw = - \alpha \nabla_\vw L(f,F,\vx_t) = - \alpha \nabla_f L(f,F,\vx_t) \nabla_\vw f_{\vw}(\vx_t)$.
With Taylor expansion,
\begin{equation}\label{eq:taylor}
f_{\vw'}(\vx) - f_{\vw}(\vx)
= -\alpha \nabla_f L(f,F,\vx_t) \, \langle \nabla_\vw f_{\vw}(\vx), \nabla_\vw f_{\vw}(\vx_t) \rangle + O(\Delta_\vw^2),
\end{equation}
where $\langle \cdot, \cdot \rangle$ denotes dot product or Frobenius inner product depending on the context.
In this equation, since $-\alpha \nabla_f L(f,F,\vx_t)$ is unrelated to $\vx$, we only consider the quantity $\langle \nabla_\vw f_{\vw}(\vx), \nabla_\vw f_{\vw}(\vx_t) \rangle$, which is known as the neural tangent kernel (NTK)~\citep{jacot2018neural}.
Without loss of generality, assume that the original prediction $f_{\vw}(\vx_t)$ is wrong, i.e., $f_{\vw}(\vx_t) \neq F(\vx_t)$ and $\nabla_f L(f,F,\vx_t) \neq 0$.
To correct the wrong prediction while avoiding forgetting, we expect this NTK to satisfy two properties that are essential for continual learning:

\begin{property}[error correction]\label{pro:1}
For $\vx = \vx_t$, $\langle \nabla_\vw f_{\vw}(\vx), \nabla_\vw f_{\vw}(\vx_t) \rangle \neq 0$.
\end{property}

\begin{property}[zero forgetting]\label{pro:2}
For $\vx \neq \vx_t$, $\langle \nabla_\vw f_{\vw}(\vx), \nabla_\vw f_{\vw}(\vx_t) \rangle = 0$.
\end{property}

In particular, \cref{pro:1} allows for error correction by optimizing $f_{\vw'}(\vx_t)$ towards the true value $F(\vx_t)$, so that we can learn new knowledge (i.e., update the learned function).
If $\langle \nabla_\vw f(\vx), \nabla_\vw f_{\vw}(\vx_t) \rangle = 0$, we then have $f_{\vw'}(\vx) - f_{\vw}(\vx) \approx 0$, failing to correct the wrong prediction at $\vx = \vx_t$.
Essentially, \cref{pro:1} requires the gradient norm to be non-zero.
On the other hand, \cref{pro:2} is much harder to be satisfied, especially for nonlinear approximations.
To make this property hold, except for $\vx=\vx_t$, the neural network $f$ is required to achieve zero forgetting after one step optimization, i.e., $\forall \vx \neq \vx_t, f_{\vw'}(\vx) = f_{\vw}(\vx)$.
It is the violation of \cref{pro:2} that leads to the forgetting issue.
For tabular cases (e.g., $\vx$ is a one-hot vector and $f_{\vw}(\vx)$ is a linear function), this property may hold by sacrificing the generalization ability of deep neural networks.
In order to benefit from generalization, we propose~\cref{pro:3} by relaxing~\cref{pro:2}:

\begin{property}[mild forgetting]\label{pro:3}
$\langle \nabla_\vw f_{\vw}(\vx), \nabla_\vw f_{\vw}(\vx_t) \rangle \approx 0$ for $\vx$ that is dissimilar to $\vx_t$ in a certain sense.
\end{property}

The above properties together formally define the concept of local elasticity, which is first proposed by~\cite{he2020local}.
A function $f$ is locally elastic if $f_{\vw}(\vx)$ is not significantly changed after the function is updated at $\vx_t$ that is dissimilar to $\vx$ in a certain sense, and vice versa.
For example, we can characterize the dissimilarity with the $2$-norm distance.
Although~\citet{he2020local} show that neural networks with nonlinear activation functions are locally elastic in general, there is a lack of theoretical understanding about the connection of neural network architectures and the degrees of local elasticity.
In our experiments, we find that the degrees of local elasticity of classical neural networks are not enough to address the forgetting issue, as we will show next.

\section{Understanding the Success and Failure of Sparse Representation}

Here, we use the above properties to understand the effectiveness of sparse representations in catastrophic forgetting.
To be specific, we argue that sparse representations are effective in reducing forgetting in linear function approximations but are less useful in nonlinear function approximations.

Deep neural networks can automatically generate effective representations (a.k.a. features) to extract key properties from input data.
In particular, we call a set of representations sparse when only a small part of representations is non-zero for a given input.
It is known that sparse representations reduce forgetting and interference in both continual supervised learning and RL~\citep{shen2021algorithmic,liu2019utility}.
Formally, let $\vx$ be an input and $\phi$ be an encoder that transforms the input $\vx$ into its representation $\phi(\vx)$.
The representation $\phi(\vx)$ is sparse when most of its elements are zeros.

First, consider the case of linear approximations. A linear function is defined as $f_{\vw}(\vx) =\vw^\top \phi(\vx): \R^n \mapsto \R$, where $\vx \in \R^n$ is an input, $\phi: \R^n \mapsto \R^m$ is a fixed encoder, and $\vw \in \R^m$ is a weight vector.
Assume the representation $\phi(\vx)$ is sparse and non-zero (i.e., $\norm{\phi(\vx)}_2 > 0$) for $\vx \in \R^n$.
Next, we show that both \cref{pro:1} and \cref{pro:3} are satisfied in this case.
Easy to see $\nabla_\vw f_{\vw}(\vx) = \phi(\vx)$.
Together with~\cref{eq:taylor}, we have
\begin{equation}\label{eq:linear_ntk}
f_{\vw'}(\vx) - f_{\vw}(\vx) = -\alpha \nabla_f L(f,F,\vx_t) \, \phi(\vx)^\top \phi(\vx_t).
\end{equation}
By assumption, $f_{\vw}(\vx_t) \neq F(\vx_t)$ and $\nabla_f L(f,F,\vx_t) \neq 0$.
Then \cref{pro:1} holds since $f_{\vw'}(\vx_t) - f_{\vw}(\vx_t) = -\alpha \nabla_f L(f,F,\vx_t) \, \norm{\phi(\vx_t)}_2^2 \neq 0$.
Moreover, when $\vx \neq \vx_t$, it is very likely that $\langle \phi(\vx), \phi(\vx_t) \rangle \approx 0$ due to the sparsity of $\phi(\vx)$ and $\phi(\vx_t)$.
Thus, $f_{\vw'}(\vx) - f_{\vw}(\vx) = -\alpha \nabla_f L(f,F,\vx_t) \, \langle \phi(\vx), \phi(\vx_t) \rangle \approx 0$ and~\cref{pro:3} holds.
We conclude that sparse representations successfully mitigate catastrophic forgetting in linear approximations.

However, for nonlinear approximations, sparse representations can no longer guarantee~\cref{pro:3}.
Consider a multilayer perceptron (MLP) with one hidden layer $f_{\vw}(\vx) = \vu^\top \sigma(\rmV\vx+\vb): \R^n \mapsto \R$, where $\sigma$ is a non-linear activation function, $\vx \in \R^n$, $\vu \in \R^m$, $\rmV \in \R^{m \times n}$, $\vb \in \R^m$, and $\vw=\{\vu, \rmV, \vb\}$.
We compute the NTK in this case, resulting in the following lemma.

\begin{restatable}[NTK in non-linear approximations]{lemma}{firstntklem}\label{lem:ntk}
Given a non-linear function $f_{\vw}(\vx) = \vu^\top\sigma(\rmV\vx+\vb): \R^n \mapsto \R$, where $\sigma$ is a non-linear activation function, $\vx \in \R^n$, $\vu \in \R^m$, $\rmV \in \R^{m \times n}$, $\vb \in \R^m$, and $\vw=\{\vu, \rmV, \vb\}$.
The NTK of this nonlinear function is
\begin{align*}
\langle \nabla_\vw f_{\vw}(\vx), \nabla_\vw f_{\vw}(\vx_t) \rangle
= \sigma(\rmV \vx + \vb)^\top \sigma(\rmV \vx_t + \vb) + \vu^\top \vu (\vx^\top \vx_t + 1) \sigma'(\rmV \vx + \vb)^\top \sigma'(\rmV \vx_t + \vb),
\end{align*}
where $\langle \cdot, \cdot \rangle$ denotes dot product or Frobenius inner product depending on the context.
\end{restatable}

The proof of this lemma can be found in~\cref{appendix:proof}.
Note that the encoder $\phi$ is no longer fixed, and we have $\phi_{\vtheta}(\vx) = \sigma(\rmV\vx+\vb)$, where $\vtheta = \{\rmV, \vb\}$ are learnable parameters.
By~\cref{lem:ntk}, 
\begin{align}\label{eq:nonlinear_ntk}
\langle \nabla_\vw f_{\vw}(\vx), \nabla_\vw f_{\vw}(\vx_t) \rangle
= \phi_{\vtheta}(\vx)^\top \phi_{\vtheta}(\vx_t) + \vu^\top \vu (\vx^\top \vx_t + 1) \phi_{\vtheta}'(\vx)^\top \phi_{\vtheta}'(\vx_t).
\end{align}
Compared with the NTK in linear approximations (\cref{eq:linear_ntk}), \cref{eq:nonlinear_ntk} has an additional term $\vu^\top \vu (\vx^\top \vx_t + 1) \phi_{\vtheta}'(\vx)^\top \phi_{\vtheta}'(\vx_t)$, due to a learnable encoder $\phi_{\vtheta}$.
With sparse representations, we have $\phi_{\vtheta}(\vx)^\top \phi_{\vtheta}(\vx_t) \approx 0$. 
However, it is not necessarily true that $\vu^\top \vu (\vx^\top \vx_t + 1) \phi_{\vtheta}'(\vx)^\top \phi_{\vtheta}'(\vx_t) \approx 0$ even when $\vx$ and $\vx_t$ are quite dissimilar, which violates~\cref{pro:3}.
For example, when $\tanh$ is used as the activation function, for $x_1=0, x_2>0$, we have $\tanh(x_1)\tanh(x_2)=0$ while $\tanh'(x_1)\tanh'(x_2) > 0$.
To conclude, our analysis indicates that sparse representations alone are not enough to reduce forgetting in nonlinear approximations.

\section{Obtaining Sparsity with Elephant Activation Functions}\label{sect:elephant}

Although~\cref{lem:ntk} shows that the forgetting issue can not be fully addressed with sparse representations solely in deep learning methods, it also points out a possible solution: sparse gradients. 
With sparse gradients, we could have $\phi_{\vtheta}'(\vx)^\top \phi_{\vtheta}'(\vx_t) \approx 0$.
Together with sparse representations, we may still satisfy~\cref{pro:3} in nonlinear approximations and thus reduce more forgetting.
Specifically, we aim to design new activation functions to obtain both sparse representations and sparse gradients.

To begin with, we first define the sparsity of a function, which also applies to activation functions.
\begin{definition}[sparse function]\label{def:sparse}
For a function $\sigma: \R \mapsto \R$, we define the sparsity of function $\sigma$ on input domain $[-C, C]$ as
\begin{equation*}
S_{\epsilon,C}(\sigma)
= \frac{\abs{\{x \mid \abs{\sigma(x)} \le \epsilon, x \in [-C, C]\}}}{\abs{\{x \mid x \in [-C, C]\}}}
= \frac{\abs{\{x \mid \abs{\sigma(x)} \le \epsilon, x \in [-C, C]\}}}{2C},
\end{equation*}
where $\epsilon$ is a small positive number and $C>0$.
As a special case, when $\epsilon \rightarrow 0^+$ and $C \rightarrow \infty$, define
\begin{equation*}
S(\sigma)
= \lim_{\epsilon \rightarrow 0^+} \lim_{C \rightarrow \infty} S_{\epsilon,C}(\sigma).
\end{equation*}
We call $\sigma$ a $S(\sigma)$-sparse function.
In particular, $\sigma$ is called a sparse function when $S(\sigma)=1$.
\end{definition}

\begin{remark}
Easy to verify that $0 \le S(\sigma) \le 1$. 
The sparsity of a function shows the fraction of nearly zero outputs given a symmetric input domain. 
For example, neither $\operatorname{ReLU}(x)$ nor $\operatorname{ReLU}'(x)$ is a sparse functions.
$\tanh(x)$ is not a sparse function while $\tanh'(x)$ is a sparse function.
In \cref{appendix:activation}, we present more examples as well as visualizations of the activation functions and their gradients.
It is worth noting that sparse activation functions and sparse representations are not the same---sparse activation functions are one-to-one mappings while sparse representations are vectors in which most elements are zeros.
By incorporating sparse activation functions in neural networks, we increase the likelihood of generating sparse representations given diverse inputs.
\end{remark}

\begin{figure}[tbp]
\centering
\vspace{-1em}
\subcaptionbox{A drawing of an elephant.}{
\includegraphics[width=\figwidththree]{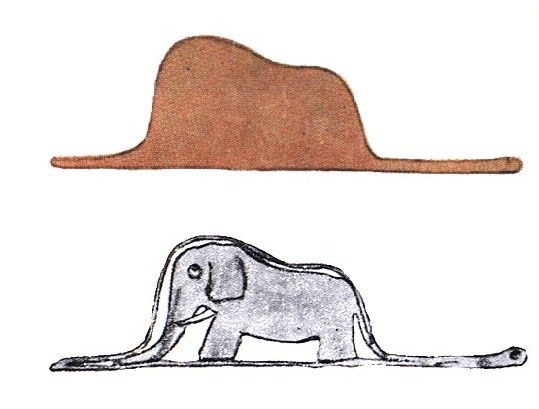}}
\subcaptionbox{$\elephant(x)$}{
\includegraphics[width=\figwidththree]{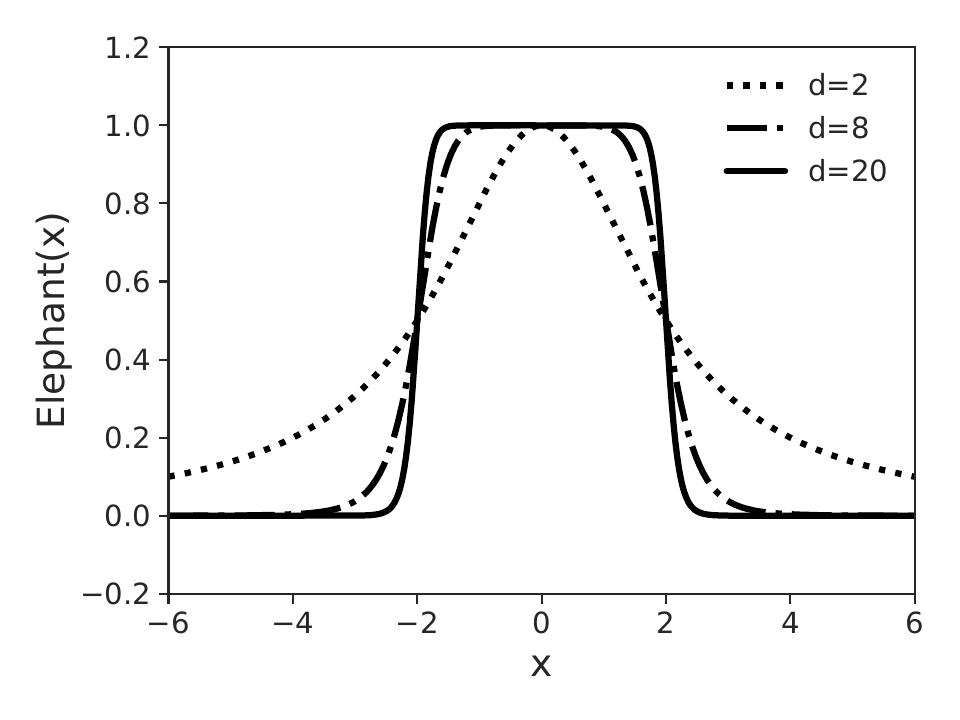}}
\subcaptionbox{$\elephant'(x)$}{
\includegraphics[width=\figwidththree]{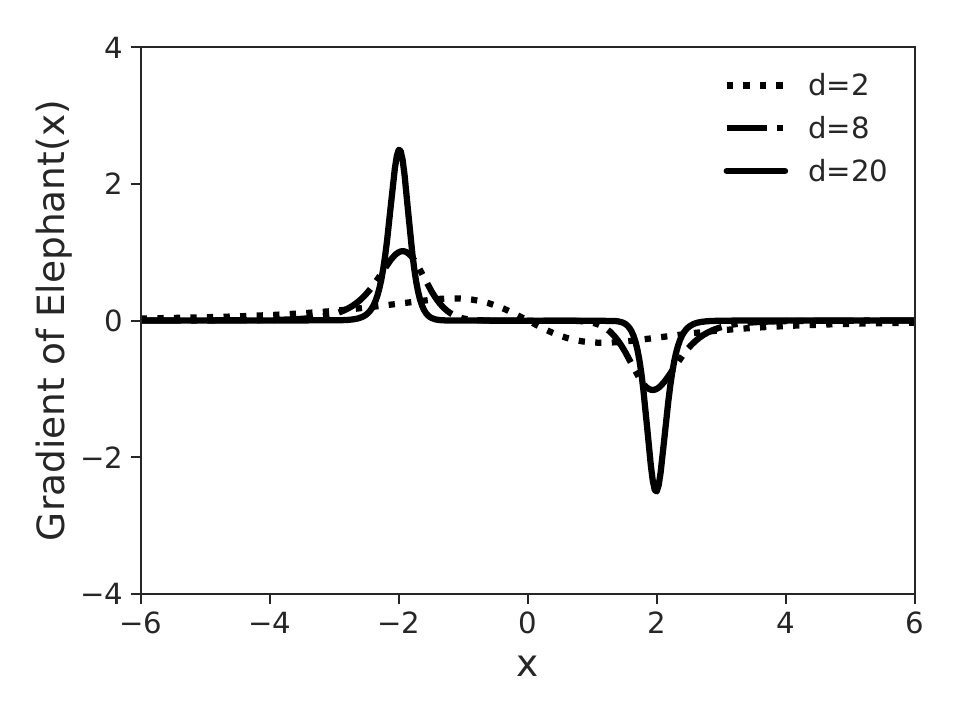}}
\caption{(a) ``My drawing was not a picture of a hat. It was a picture of a boa constrictor digesting an elephant.'' \textit{The Little Prince}, by Antoine de Saint Exup\'ery. (b) Elephant functions with $a=2$, $h=1$, and various $d$. (c) Gradients of elephant functions with $a=2$, $h=1$, and various $d$.}
\label{fig:elephant}
\vspace{-1em}
\end{figure}

Next, we propose a novel class of bell-shaped activation functions, \textit{elephant activation functions}.~\footnote{We name this bell-shaped activation function the \textit{elephant} function, as it suggests that this activation empowers neural networks with continual learning ability, echoing the saying ``an elephant never forgets.'' The bell shape also resembles the silhouette of an elephant (see \cref{fig:elephant}), paying homage to \textit{The Little Prince} by Antoine de Saint Exup\'ery.}
Formally, an elephant function is defined as
\begin{equation}\label{eq:elephant}
\elephant(x) = \frac{h}{1+\abs{\frac{x}{a}}^d},
\end{equation}
where $a$ controls the width of the function, $h$ is the height, and $d$ controls the slope.
We call a neural network that uses elephant activation functions an \textit{elephant neural network (ENN)}.
For example, for MLPs, we have \textit{elephant MLPs (EMLPs)} correspondingly.

As shown in~\cref{fig:elephant}, elephant activation functions have both sparse function values and sparse gradient values, which can be formally proved as well (see~\cref{lem:elephant} in~\cref{appendix:proof}).
Specifically, $d$ controls the sparsity of gradients for elephant functions.
The larger the value of $d$, the sharper the slope and the sparser the gradient.
On the other hand, $a$ controls the sparsity of the function itself.
Since both the gradient of an elephant function and the elephant function itself are sparse functions, we increase the likelihood of both sparse representations and gradients (see~\cref{eq:nonlinear_ntk}).
Formally, we show that \cref{pro:3} holds under certain conditions for elephant functions.

\begin{restatable}[]{thm}{firstthmmain}\label{thm:main}
Define $f_{\vw}(\vx)$ as in~\cref{lem:ntk}. Let $\sigma$ be the elephant activation function with $h=1$ and $d \rightarrow \infty$. When $\abs{\rmV (\vx - \vx_t)} \succ 2 a \mathbf{1}_{m}$, we have $\langle \nabla_\vw f_{\vw}(\vx), \nabla_\vw f_{\vw}(\vx_t) \rangle = 0$, where $\succ$ denotes an element-wise inequality symbol and $\mathbf{1}_{m} = [1, \cdots, 1]^\top \in \R^m$.
\end{restatable}

\begin{remark}
\cref{thm:main} mainly proves that when $d \rightarrow \infty$, \cref{pro:3} holds with elephant functions.
However, even when $d$ is a small integer (e.g., 8), we can still obtain this property, as we will show in the experiment section. The proof is in~\cref{appendix:proof}.
\end{remark}

\section{Experiments}

In this section, we experimentally validate a series of hypotheses regarding the effectiveness of elephant activation functions under regression and RL settings.

\subsection{Streaming Learning for Regression}

RL is relatively complex due to agent-environment interactions.
Instead, we first perform experiments in a simple regression task in the streaming learning setting to answer the following question:
\begin{adjustwidth}{20pt}{}
\textit{Can we achieve mild forgetting and local elasticity with elephant activation functions?}
\end{adjustwidth}

In this setting, a learning agent is presented with one sample only at each time step and then performs learning updates.
Moreover, the learning happens in a single pass of the whole dataset; that is, each sample only occurs once.
Furthermore, the data stream is assumed to be non-independent and identically distributed (non-iid).
Finally, the evaluation happens after each new sample arrives, which requires the agent to learn quickly while avoiding forgetting.
Streaming learning methods enable real-time adaptation; thus, they are more suitable in real-world scenarios where data is received in a continuous flow.

We consider approximating a sine function in the streaming learning setting.
In this task, there is a stream of data $(x_1,y_1), (x_2,y_2), \cdots, (x_n, y_n)$, where $0 \le x_1 < x_2 < \cdots < x_n \le 2$, $y_i = \sin(\pi x_i)$, and $n=200$.
The learning agent $f$ is an MLP with one hidden layer of size $1,000$.
At each time step $t$, the agent only receives one sample $(x_t,y_t)$.
We minimize the loss $l_t = (f(x_t) - y_t)^2$, where $f(x_t)$ is the agent's prediction.
We measure the agent performance by the mean square error (MSE) on a test dataset with $1,000$ samples, where the inputs are evenly spaced over the interval $[0,2]$.
Additional experimental details are in~\cref{appendix:streaming}.

We compare our method $\elephant$ with two kinds of baselines.
One is classical activation functions, including $\relu$, $\sigmoid$, $\elu$, and $\tanh$.
The other is the sparse representation neural network ($\srnn$)~\citep{liu2019utility}, which generates sparse representations by regularization.
We summarize test MSEs in~\cref{tb:sin} in~\cref{appendix:streaming}.
Lower is better.
Clearly, $\elephant$ has the best performance, achieving a test MSE that is two orders of magnitude lower compared to baselines, which have similar orders to each other.
Moreover, $\srnn$ performs slightly better than classical activation functions, showing the benefits of sparse representations.
Yet, compared with $\elephant$, the test MSE of $\srnn$ is still large, indicating that it fails to approximate well.
To analyze in depth, we plot the true function $\sin(\pi x)$, the learned function $f(x)$, and the NTK function $\ntk(x) = \langle \nabla_\vw f_{\vw}(x), \nabla_\vw f_{\vw}(x_t) \rangle$ at different training stages for $\elephant$ in~\cref{fig:sin}.
We normalize $\ntk(\vx)$ such that the function value is in $[-1, 1]$.
The plots in the first row show that for $\elephant$ with $d=8$, $\ntk(x)$ quickly decreases to $0$ as $x$ moves away from $x_t$, demonstrating the local elasticity of applying $\elephant$ with a small $d$.
However, $\srnn$s (and MLPs with classical activation functions) are not locally elastic; the learned function basically evolves as a linear function, a phenomenon that often appears in over-parameterized neural networks~\citep{jacot2018neural,chizat2019lazy}.

By injecting local elasticity to a neural network, we can break the inherent global generalization ability~\citep{ghiassian2020improving} of the neural networks, constraining the output changes of the neural network to small local areas.
Utilizing this phenomenon, we can update a wrong prediction by ``editing'' outputs of a neural network nearly point-wisely.
To verify, we first train a neural network to approximate $\sin(\pi x)$ well enough, calling it the old learned function.
Now assume that the original $y$ value of an input $x$ is changed to $y'$, while the true values of other inputs remain the same.
Our goal is to update the prediction for input $x$ to $y'$, while keeping the predictions of other inputs without expensive re-training on the whole dataset.
Note that this requirement is common in RL (see next section).
Specifically, we choose $x=1.5$, $y=-1.0$, and $y'=-1.5$; and perform experiments with $\elephant$ and $\relu$, showing in~\cref{fig:elastic}.
Both methods successfully update the prediction at $x=1.5$ to $y'$.
However, besides the prediction at $x=1.5$, the learned function with $\relu$ is changed globally while the changes with $\elephant$ are mainly confined in a small local area around $x=1.5$.
That is, we can successfully correct the wrong prediction nearly point-wise by ``editing'' the output value for elephant neural networks (ENNs), but not for classical neural networks.

To conclude, we showed that (1) sparse representations do not suffice to address the forgetting issue, (2) ENNs are locally elastic even when $d$ is small, and (3) ENNs can continually learn to solve regression tasks by reducing forgetting.

\begin{figure}[tbp]
\centering
\vspace{-1em}
\subcaptionbox{Approximating a sine function\label{fig:sin}}{
\includegraphics[width=0.63\textwidth]{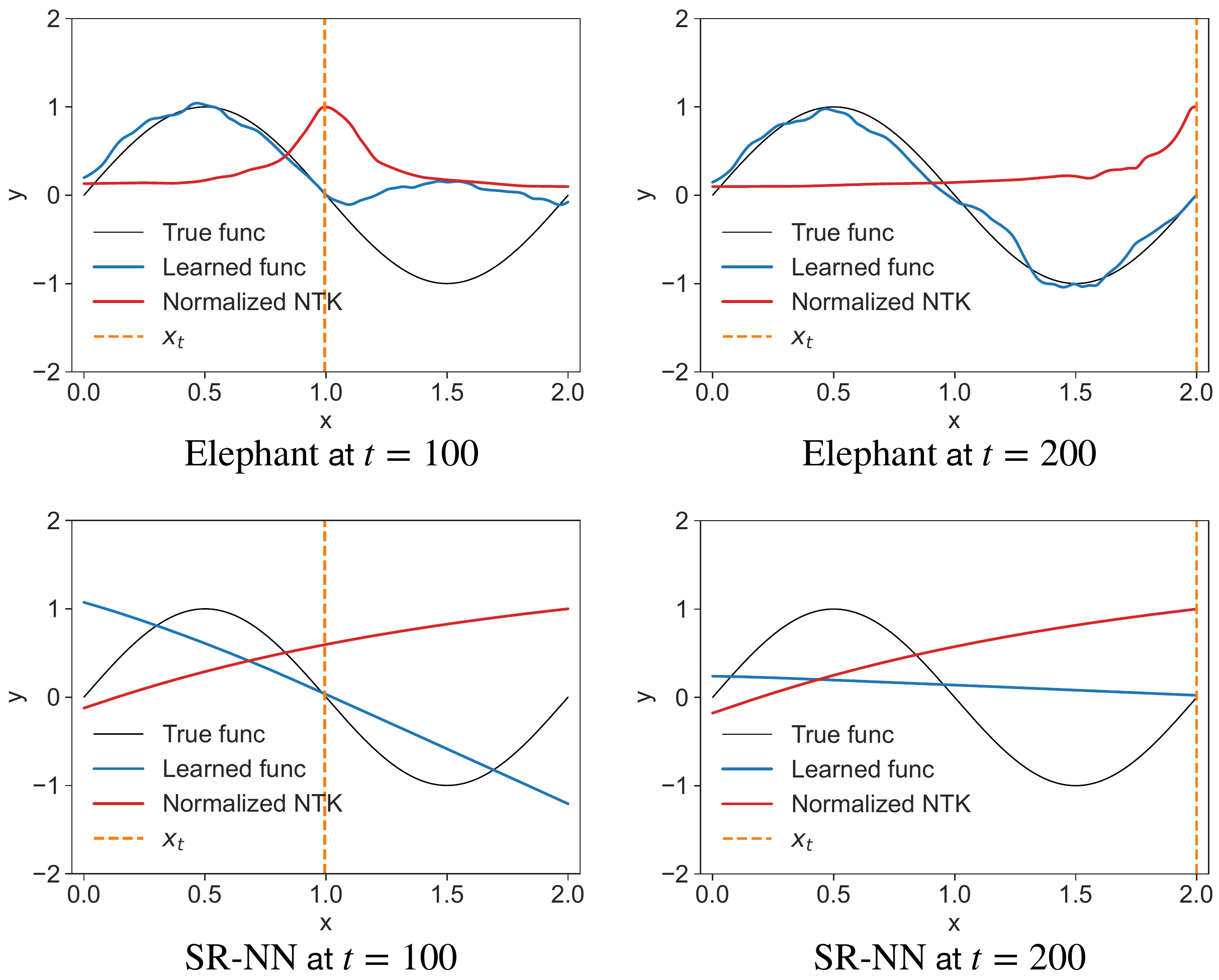}
}
\hspace{1em}
\subcaptionbox{Updating a wrong prediction\label{fig:elastic}}{
\includegraphics[width=0.3\textwidth]{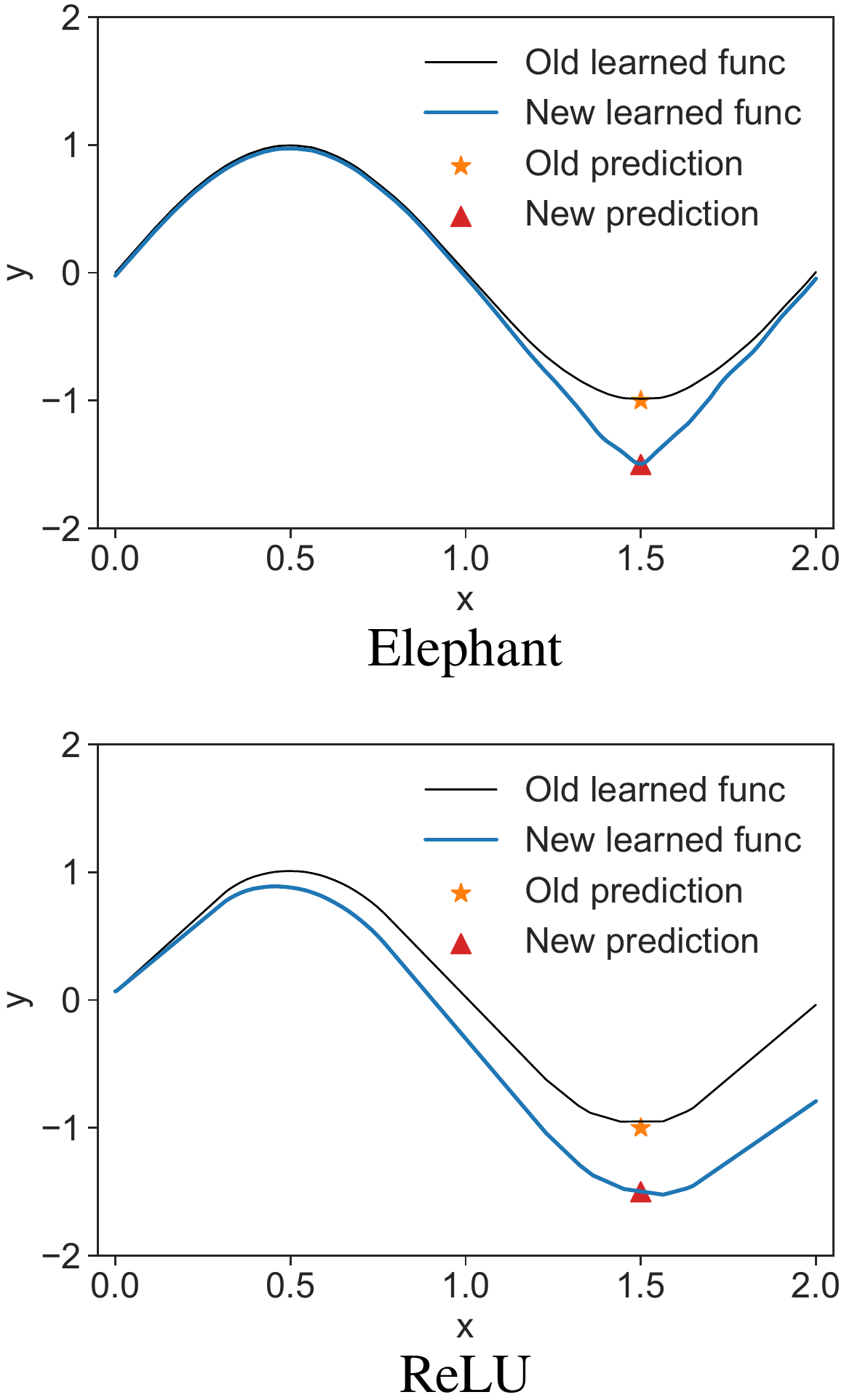}
}
\caption{(a) Plots of the true function $\sin(\pi x)$, the learned function $f(x)$, and the NTK function $\ntk(x)$ at different training stages for $\elephant$ and $\srnn$. The plots of classical activation functions are not presented since they are similar to the plots of $\srnn$. The $\ntk(x)$ of $\elephant$ quickly reduces to 0 as $x$ moves away from $x_t$, demonstrating local elasticity. (b) $\elephant$ allows updating the old prediction nearly point-wisely by ``editing'' the output value of the neural network, a capability lacking in classical activation functions.}
\label{fig:sin_elastic}
\vspace{-1em}
\end{figure}

\subsection{Reinforcement Learning}

Recently, \citet{lan2023memory} showed that the forgetting issue exists even in single RL tasks and a large replay buffer largely masks it.
Without a replay buffer, a single RL task can be viewed as a series of tasks without clear boundaries~\citep{dabney2021value}.
For example, in temporal difference (TD) learning, the true value function is approximated by $V$ with bootstrapping: $V(s_t) \leftarrow r_{t+1} + \gamma V(s_{t+1})$,
where $s_t$ and $s_{t+1}$ are two successive states and $r_{t+1} + \gamma V(s_{t+1})$ is named the TD target.
During training, the TD target constantly changes due to bootstrapping, non-stationary state distribution, and changing policy.
To speed up learning while reducing forgetting, it is crucial to update $V(S_t)$ to the new TD target without changing other state values too much, where local elasticity can help.

In the following, we aim to demonstrate that incorporating elephant activation functions helps reduce forgetting in RL, thus improving sample efficiency and learning performance.
We mainly consider two representative value-based RL algorithms --- deep Q-network (DQN)~\citep{mnih2013playing,mnih2015human} and Rainbow~\citep{hessel2018rainbow}.
We consider other activation functions as baselines, including $\relu$, $\tanh$, $\maxout$~\citep{goodfellow2013maxout}, $\lwta$~\citep{srivastava2013compete}, and $\fta$~\citep{pan2020fuzzy}.
While $\relu$ and $\tanh$ are widely used classical activation functions, $\maxout$, $\lwta$, and $\fta$ are designed to generate sparse representations (see~\cref{sec:related} for detailed introductions).
The full experimental details can be found in~\cref{appendix:rl}.

\subsubsection{Elephant Improves Memory Efficiency}

First, we focus on addressing the following question:
\begin{adjustwidth}{20pt}{}
\textit{Can elephant activation functions help achieve strong performance with a small memory budget?}
\end{adjustwidth}

Specifically, we test DQN in 4 classical RL tasks (i.e., MountainCar-v0, Acrobot-v1, Catcher, and Pixelcopter) from Gymnasium~\citep{towers2023gymnasium} and PyGame Learning Environment~\citep{tasfi2016PLE} under various buffer sizes.
Note that we select these small-scale RL tasks so that we could perform thorough hyperparameter tuning and ensure a fair comparison within a reasonable time frame.
For instance, we consider buffer sizes in $\{32, 1e2, 3e2, 1e3, 3e3, 1e4\}$, where the default buffer size is $1e4$.
And for each hyper-parameter setup, a best learning rate is selected from $\{1e-2, 3e-3, 1e-3, 3e-4, 1e-4, 3e-5, 1e-5, 3e-6\}$, while RMSProp~\citep{tieleman2012rmsprop} is applied with a decay rate of $0.999$ for optimization.
To make a fair comparison, we use the same hyper-parameters of $\elephant$ (i.e., $d=4$, $h=1$, and $a=0.2$) in all experiments, although tuning them for each task and buffer size could further boost the performance.

\begin{figure}[tbp]
\centering
\vspace{-1em}
\includegraphics[width=\figwidthone]{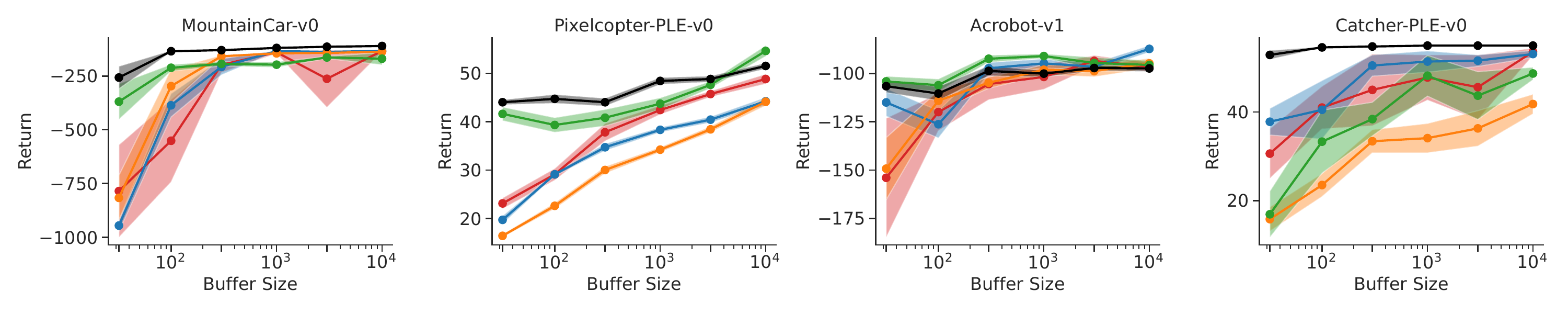}
\includegraphics[width=0.55\textwidth]{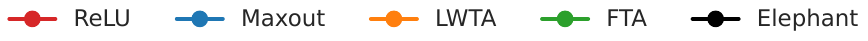}
\caption{The performance of DQN in $4$ Gymnasium and PyGame tasks with different activation functions under various buffer sizes, measured as the average return of the last 10\% episodes. 
}
\vspace{-1em}
\label{fig:gym_dqn_buffer}
\end{figure}

We plot agents' performance of different activation functions and buffer sizes in~\cref{fig:gym_dqn_buffer}.
All results are averaged over $10$ runs, and the shaded areas represent standard errors.
Clearly, $\elephant$ demonstrates its robustness to varying buffer sizes.
When the buffer size is reduced, the performance of $\elephant$ remains high in all four tasks, while the performance of all other activations drops significantly.
In summary, these results confirm the effectiveness of $\elephant$ in reducing forgetting and improving memory efficiency for DQN.

\subsubsection{Elephant Improves Sample Efficiency}

Next, we investigate the following question:
\begin{adjustwidth}{20pt}{}
\textit{Given the same amount of training samples and sufficient memory resources, can elephant activation functions outperform classical activation functions?}
\end{adjustwidth}

\begin{figure}[htbp]
\centering
\subcaptionbox{DQN \label{fig:atari_dqn_summary}}{\includegraphics[width=\figwidthtwo]{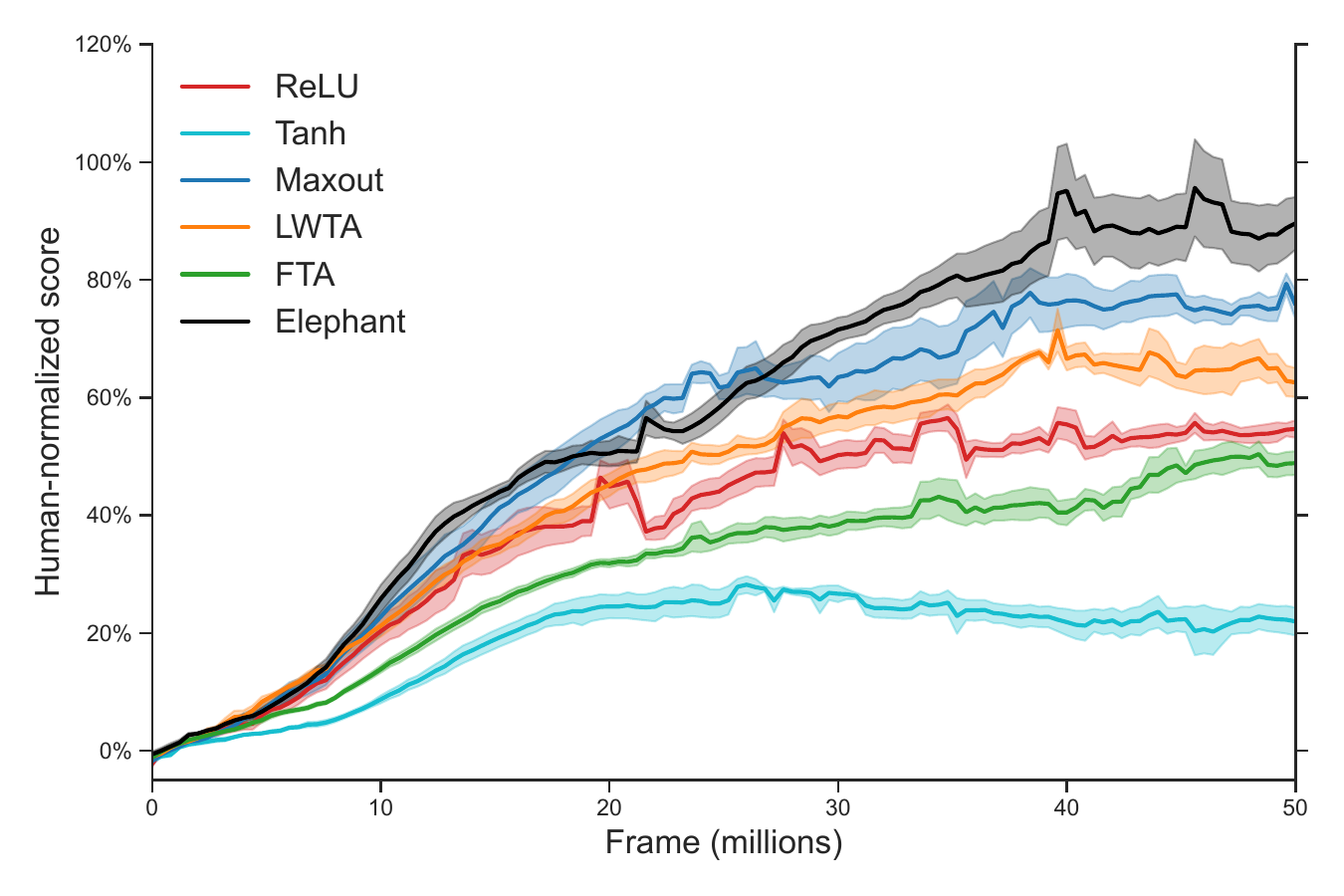}}
\subcaptionbox{Rainbow \label{fig:atari_rainbow_summary}}{\includegraphics[width=\figwidthtwo]{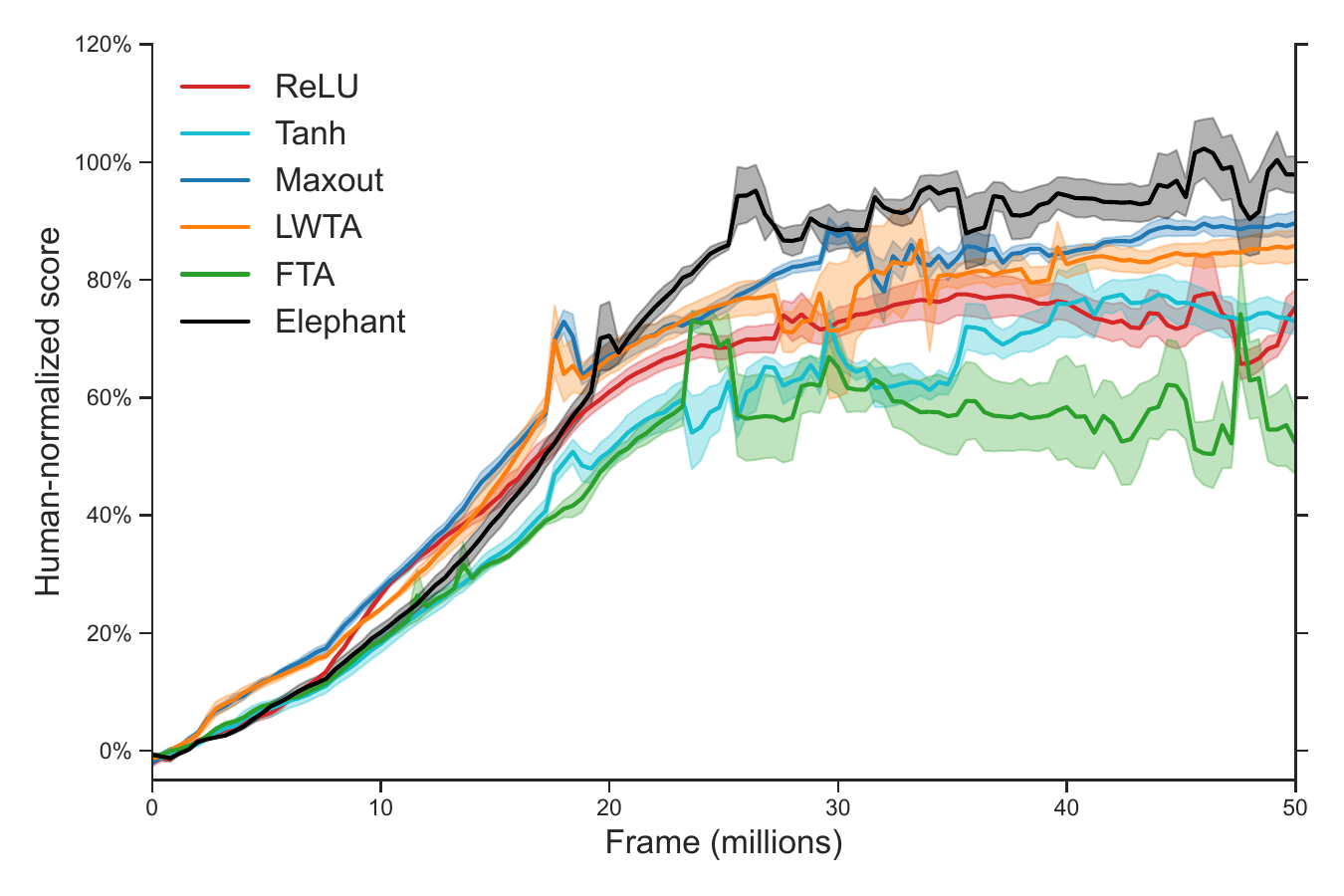}}
\caption{The learning curves of DQN and Rainbow aggregated over 10 Atari tasks for 6 activation functions.
Solid lines correspond to the average performance over 5 runs, while shaded areas correspond to standard errors.
}
\label{fig:atari_summary}
\end{figure}

To answer this question, we test DQN and Rainbow in 10 Atari games~\citep{bellemare2013arcade} with different activation functions.
Specifically, we selected 10 representative Atari games recommended by~\citet{aitchison2022atari}---Amidar, Battlezone, Bowling, Double Dunk, Frostbite, Kung-Fu Master, Name This Game, Phoenix, Q*bert, and River Raid—--which produce scores that closely correlate with the performance estimates on the full set of 57 Atari games, while requiring substantially less training cost.
The implementation of DQN and Rainbow are adapted from Tianshou~\citep{tianshou}.~\footnote{\url{https://github.com/thu-ml/tianshou/blob/v0.4.10/examples/atari/}}
We also follow the default hyper-parameters as in Tianshou unless explicitly stated otherwise.
Adam~\citep{kingma2015adam} is applied to optimize. We train all agents for 50M frames.
For DQN, the learning rate is $1e-4$ while it is $6.25e-5$ for Rainbow.
The buffer size is 1 million.

In \cref{fig:atari_summary}, we show the test human-normalized scores of DQN and Rainbow with different activation functions, aggregated over 10 Atari tasks.
Solid lines correspond to the median performance over 5 runs, while shaded areas correspond to standard errors.
Moreover, in the appendix, we present the detailed test performance of DQN and Rainbow in all 10 Atari tasks in \cref{tab:atari_individual} and the return curves in \cref{fig:atari_individual}.
Together, these results demonstrate that even when a large buffer is used, $\elephant$ still surpasses the baselines significantly.

\subsubsection{Gradient Analysis}

In \cref{sect:elephant}, we claimed that $\elephant$ induces sparse gradient and thus reduces forgetting.
Here, we perform a gradient analysis to verify the claim.
To be specific, we visualize the gradient covariance matrices at different stages of training DQN in Atari tasks.
Essentially, the gradient covariance matrix is a matrix of normalized NTK.
Formally, we estimate this matrix (denoted as $C$) by randomly sampling $k$ training samples $\rvx_1, \cdots, \rvx_k$ and compute each element as 
$C_{ij} = \frac{\langle \nabla_\theta l(\theta, \rvx_i), \nabla_\theta l(\theta, \rvx_j) \rangle}{\|\nabla_\theta l(\theta, \rvx_i)\| \|\nabla_\theta l(\theta, \rvx_j)\|}$,
where $l$ is the loss function and $\theta$ is the weight vector.
The gradient covariance matrix is strongly related to generalization and interference~\citep{fort2020stiffness,lyle2023understanding} --- negative off-diagonal entries usually indicate interference between different training samples, while positive off-diagonal entries reflect generalization.

Considering \cref{pro:3} and \cref{thm:main}, when $\elephant$ is applied, we expect the off-diagonal entries of the gradient covariance matrix to be close to zero.
Indeed, as shown in \cref{fig:atari_grad_dqn}, the gradient covariance matrix exhibits near-zero off-diagonal values throughout the entire training process when $\elephant$ is used, indicating mild generalization and reduced interference.
However, for other activation functions such as $\relu$, some off-diagonal values remain noticeably non-zero, indicating overgeneralization and strong interference.
Due to space constraints, the heatmaps for all 10 Atari tasks and 6 activation functions are included in \cref{appendix:grad}.

\begin{figure}[htbp]
\centering
\subcaptionbox{$\relu$}{\includegraphics[width=\figwidthtwo]{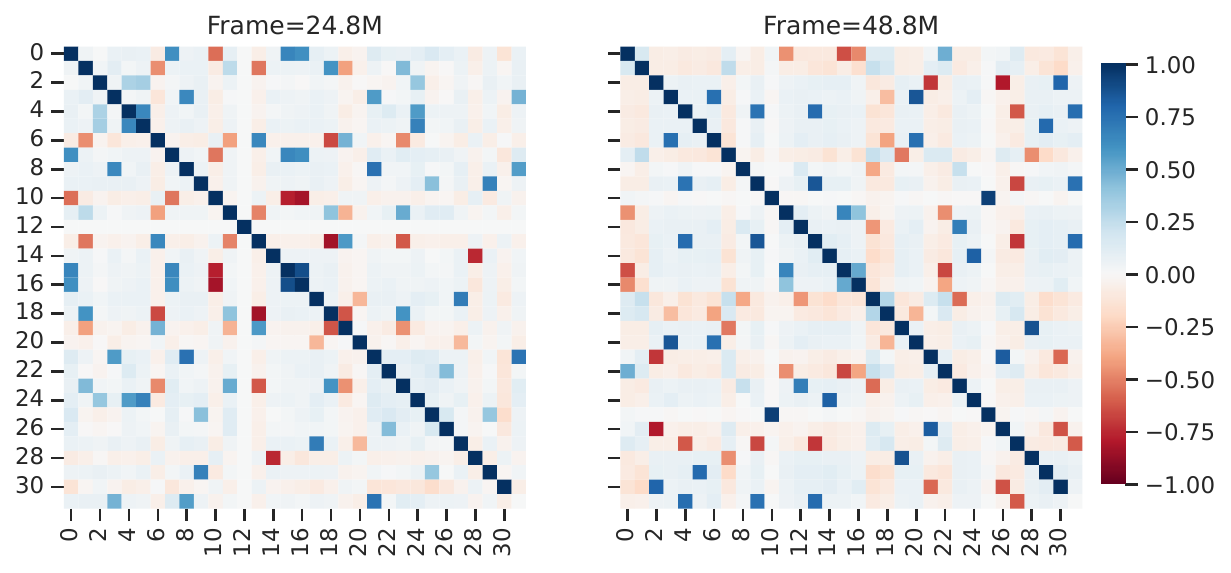}}
\subcaptionbox{$\elephant$}{\includegraphics[width=\figwidthtwo]{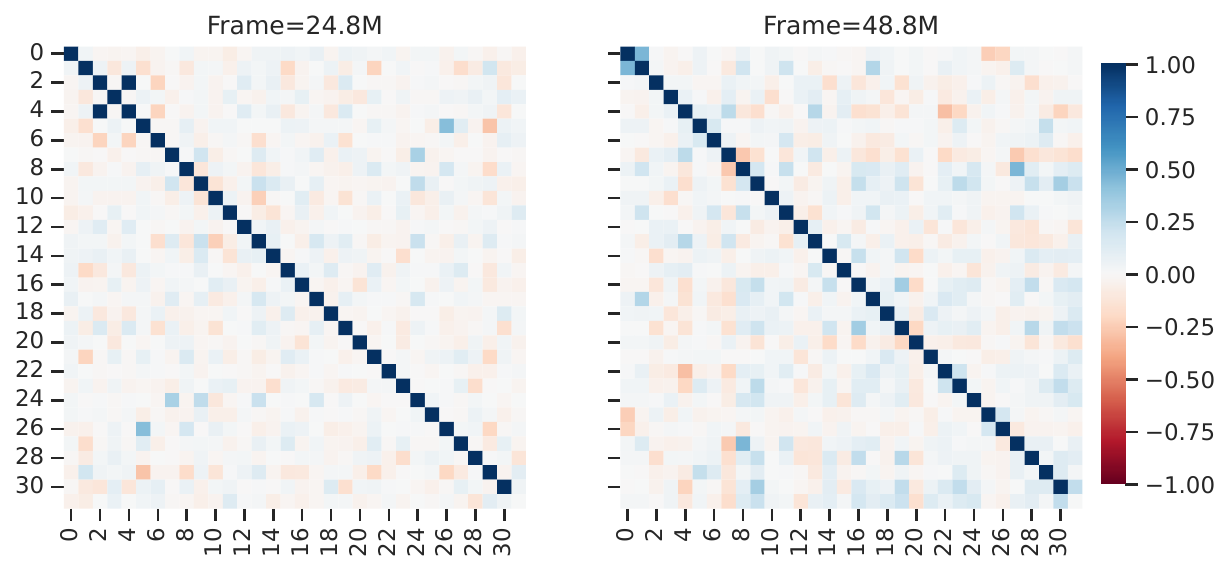}}
\hfill
\caption{Heatmaps of gradient covariance matrices for training DQN in Amidar at the midpoint (Frame = 24.8M) and end (Frame = 48.8M) of training.}
\label{fig:atari_grad_dqn}
\end{figure}

\begin{figure}[htbp]
\centering
\includegraphics[width=0.2\textwidth]{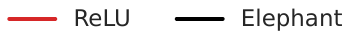} \\
\subcaptionbox{Hyper-parameter $a$}{\includegraphics[width=\figwidthtwo]{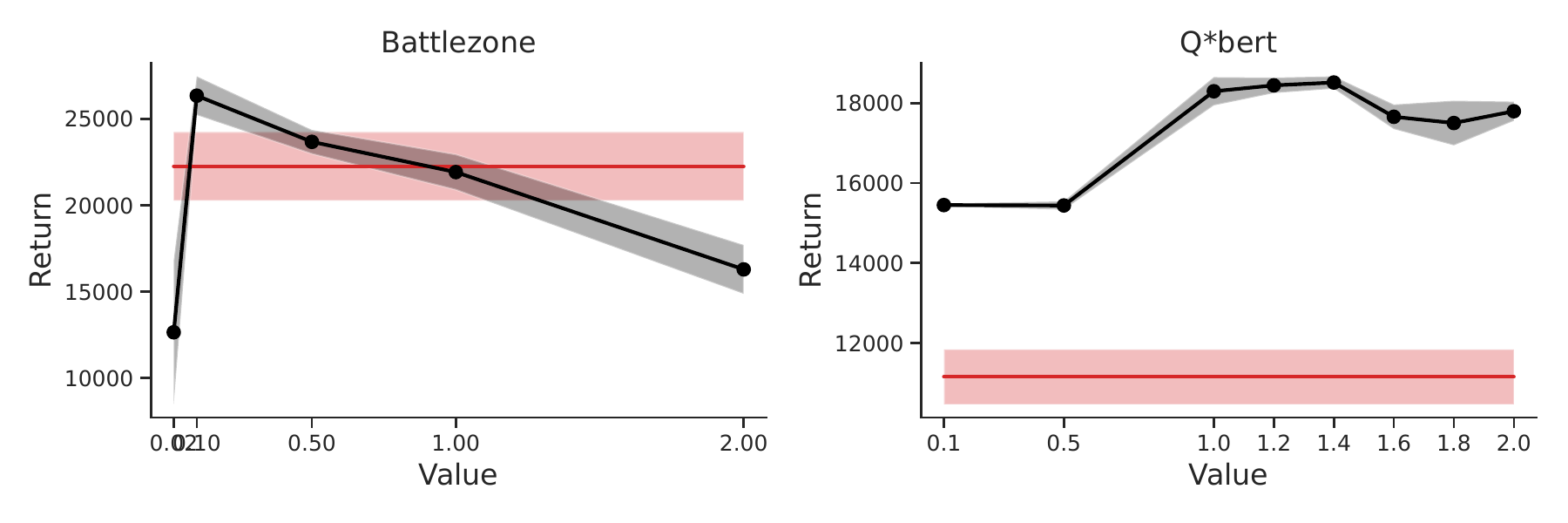}}
\subcaptionbox{Hyper-parameter $d$}{\includegraphics[width=\figwidthtwo]{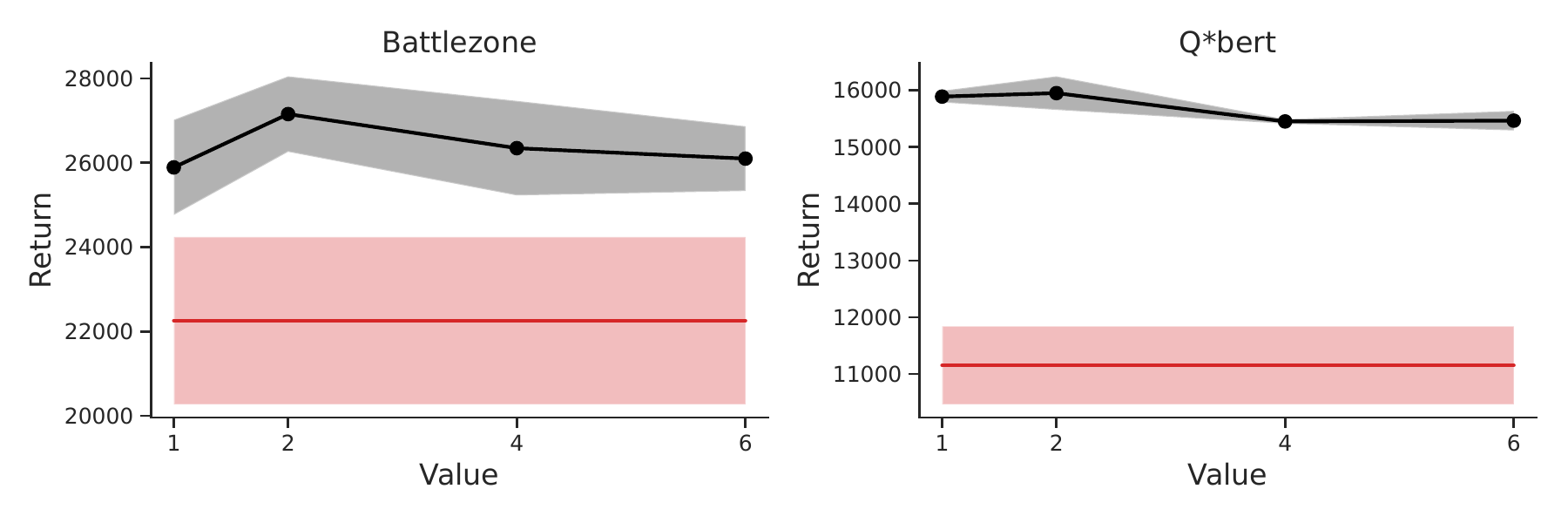}}
\caption{A sensitivity analysis of hyper-parameters $\sigma$ and $d$ in $\elephant$. We present the final test returns averaged over 5 runs, and the shaded areas represent standard errors.}
\label{fig:atari_dqn_hyper}
\end{figure}

\subsubsection{Sensitivity Analysis}

Finally, we conduct a sensitivity analysis for the hyper-parameters of $\elephant$.
Specifically, we vary $\sigma$ and $d$ (see \cref{eq:elephant}) in $\elephant$ and test DQN in Battlezone and Q*bert.
For reference, we also show the performance of $\relu$.
As shown in \cref{fig:atari_dqn_hyper}, for Battlezone, a small $a$ (around 0.1) yields the best performance, whereas for Q*bert, a relatively large $a$ (around 1.2) performs best.
As for $d$, the performance is more robust to $d$ and a small $d$ works well across tasks in general.

Please also check the appendix for additional experiments for class incremental learning (\cref{appendix:clari}) and policy gradient methods (\cref{appendix:pg}).

\section{Conclusion and Discussion}

In this work, we proposed elephant activation functions that can make neural networks more resilient to catastrophic forgetting.
Theoretically, our work provided a deeper understanding of the role of activation functions in catastrophic forgetting.
Empirically, we showed that incorporating elephant activation functions in neural networks improves memory efficiency and learning performance of value-based algorithms.

While our results demonstrate the effectiveness of elephant activation functions in model-free RL, we have not yet explored its applicability in model-based RL.
Another limitation is the lack of a principled method for selecting the hyper-parameters of elephant activation functions.
The optimal values for these parameters appear to depend on both the input data and architectural factors such as the number of input and output features in each layer.
In practice, careful tuning $a$ is required to achieve a favorable stability-plasticity trade-off.

\newpage
\bibliography{reference}

\begin{thebibliography}{}

\bibitem[Abbasi et~al., 2022]{abbasi2022sparsity}
Abbasi, A., Nooralinejad, P., Braverman, V., Pirsiavash, H., and Kolouri, S.
  (2022).
\newblock Sparsity and heterogeneous dropout for continual learning in the null
  space of neural activations.
\newblock In {\em Conference on Lifelong Learning Agents}. PMLR.

\bibitem[Aitchison et~al., 2022]{aitchison2022atari}
Aitchison, M., Sweetser, P., and Hutter, M. (2022).
\newblock Atari-5: Distilling the arcade learning environment down to five
  games.
\newblock {\em arXiv preprint arXiv:2210.02019}.

\bibitem[Aljundi et~al., 2017]{aljundi2017expert}
Aljundi, R., Chakravarty, P., and Tuytelaars, T. (2017).
\newblock Expert gate: Lifelong learning with a network of experts.
\newblock In {\em Proceedings of the IEEE Conference on Computer Vision and
  Pattern Recognition}.

\bibitem[Aljundi et~al., 2019]{aljundi2019task}
Aljundi, R., Kelchtermans, K., and Tuytelaars, T. (2019).
\newblock Task-free continual learning.
\newblock In {\em Proceedings of the IEEE/CVF Conference on Computer Vision and
  Pattern Recognition}.

\bibitem[Ammar et~al., 2014]{ammar2014online}
Ammar, H.~B., Eaton, E., Ruvolo, P., and Taylor, M. (2014).
\newblock Online multi-task learning for policy gradient methods.
\newblock In {\em International conference on machine learning}.

\bibitem[Atkinson et~al., 2021]{atkinson2021pseudo}
Atkinson, C., McCane, B., Szymanski, L., and Robins, A. (2021).
\newblock Pseudo-rehearsal: Achieving deep reinforcement learning without
  catastrophic forgetting.
\newblock {\em Neurocomputing}.

\bibitem[Ba et~al., 2016]{ba2016layer}
Ba, J.~L., Kiros, J.~R., and Hinton, G.~E. (2016).
\newblock Layer normalization.
\newblock In {\em NIPS 2016 Deep Learning Symposium}.

\bibitem[Bellemare et~al., 2013]{bellemare2013arcade}
Bellemare, M.~G., Naddaf, Y., Veness, J., and Bowling, M. (2013).
\newblock The arcade learning environment: An evaluation platform for general
  agents.
\newblock {\em Journal of Artificial Intelligence Research}, 47:253--279.

\bibitem[Blalock et~al., 2020]{blalock2020state}
Blalock, D., Gonzalez~Ortiz, J.~J., Frankle, J., and Guttag, J. (2020).
\newblock What is the state of neural network pruning?
\newblock {\em Proceedings of machine learning and systems}.

\bibitem[Bradbury et~al., 2018]{jax2018github}
Bradbury, J., Frostig, R., Hawkins, P., Johnson, M.~J., Leary, C., Maclaurin,
  D., Necula, G., Paszke, A., Vander{P}las, J., Wanderman-{M}ilne, S., and
  Zhang, Q. (2018).
\newblock {JAX}: composable transformations of {P}ython+{N}um{P}y programs.

\bibitem[Bricken et~al., 2023]{bricken2023sparse}
Bricken, T., Davies, X., Singh, D., Krotov, D., and Kreiman, G. (2023).
\newblock Sparse distributed memory is a continual learner.
\newblock In {\em International Conference on Learning Representations}.

\bibitem[Cahill, 2011]{cahill2011catastrophic}
Cahill, A. (2011).
\newblock {\em Catastrophic forgetting in reinforcement-learning environments}.
\newblock PhD thesis, University of Otago.

\bibitem[Ceron et~al., 2024]{ceron2024value}
Ceron, J. S.~O., Courville, A., and Castro, P.~S. (2024).
\newblock In value-based deep reinforcement learning, a pruned network is a
  good network.
\newblock In {\em International Conference on Machine Learning}.

\bibitem[Chen et~al., 2020]{chen2020label}
Chen, S., He, H., and Su, W. (2020).
\newblock Label-aware neural tangent kernel: Toward better generalization and
  local elasticity.
\newblock {\em Advances in Neural Information Processing Systems},
  33:15847--15858.

\bibitem[Chizat et~al., 2019]{chizat2019lazy}
Chizat, L., Oyallon, E., and Bach, F. (2019).
\newblock On lazy training in differentiable programming.
\newblock {\em Advances in neural information processing systems}.

\bibitem[Chrabaszcz et~al., 2017]{chrabaszcz2017downsampled}
Chrabaszcz, P., Loshchilov, I., and Hutter, F. (2017).
\newblock A downsampled variant of imagenet as an alternative to the cifar
  datasets.
\newblock {\em arXiv preprint arXiv:1707.08819}.

\bibitem[Dabney et~al., 2021]{dabney2021value}
Dabney, W., Barreto, A., Rowland, M., Dadashi, R., Quan, J., Bellemare, M.~G.,
  and Silver, D. (2021).
\newblock The value-improvement path: Towards better representations for
  reinforcement learning.
\newblock In {\em Proceedings of the AAAI Conference on Artificial
  Intelligence}.

\bibitem[Delange et~al., 2021]{delange2021continual}
Delange, M., Aljundi, R., Masana, M., Parisot, S., Jia, X., Leonardis, A.,
  Slabaugh, G., and Tuytelaars, T. (2021).
\newblock A continual learning survey: Defying forgetting in classification
  tasks.
\newblock {\em IEEE Transactions on Pattern Analysis and Machine Intelligence}.

\bibitem[Deng, 2012]{deng2012mnist}
Deng, L. (2012).
\newblock The {MNIST} database of handwritten digit images for machine learning
  research.
\newblock {\em IEEE Signal Processing Magazine}.

\bibitem[Dettmers and Zettlemoyer, 2019]{dettmers2019sparse}
Dettmers, T. and Zettlemoyer, L. (2019).
\newblock Sparse networks from scratch: Faster training without losing
  performance.
\newblock {\em arXiv preprint arXiv:1907.04840}.

\bibitem[Elsayed et~al., 2024]{elsayed2024streaming}
Elsayed, M., Vasan, G., and Mahmood, A.~R. (2024).
\newblock Streaming deep reinforcement learning finally works.
\newblock {\em arXiv preprint arXiv:2410.14606}.

\bibitem[Farajtabar et~al., 2020]{farajtabar2020orthogonal}
Farajtabar, M., Azizan, N., Mott, A., and Li, A. (2020).
\newblock Orthogonal gradient descent for continual learning.
\newblock In {\em International Conference on Artificial Intelligence and
  Statistics}.

\bibitem[Farquhar and Gal, 2018]{farquhar2018towards}
Farquhar, S. and Gal, Y. (2018).
\newblock Towards robust evaluations of continual learning.
\newblock {\em arXiv preprint arXiv:1805.09733}.

\bibitem[Fernando et~al., 2017]{fernando2017pathnet}
Fernando, C., Banarse, D., Blundell, C., Zwols, Y., Ha, D., Rusu, A.~A.,
  Pritzel, A., and Wierstra, D. (2017).
\newblock {PathNet}: Evolution channels gradient descent in super neural
  networks.
\newblock {\em arXiv preprint arXiv:1701.08734}.

\bibitem[Fort et~al., 2020]{fort2020stiffness}
Fort, S., Nowak, P.~K., Jastrzebski, S., and Narayanan, S. (2020).
\newblock Stiffness: A new perspective on generalization in neural networks.
\newblock {\em arXiv preprint arXiv:1901.09491}.

\bibitem[Fortunato et~al., 2018]{fortunato2018noisy}
Fortunato, M., Azar, M.~G., Piot, B., Menick, J., Hessel, M., Osband, I.,
  Graves, A., Mnih, V., Munos, R., Hassabis, D., Pietquin, O., Blundell, C.,
  and Legg, S. (2018).
\newblock Noisy networks for exploration.
\newblock In {\em International Conference on Learning Representations}.

\bibitem[Frankle and Carbin, 2019]{frankle2019lottery}
Frankle, J. and Carbin, M. (2019).
\newblock The lottery ticket hypothesis: Finding sparse, trainable neural
  networks.
\newblock In {\em International Conference on Learning Representations}.

\bibitem[French, 1992]{french1992semi}
French, R.~M. (1992).
\newblock Semi-distributed representations and catastrophic forgetting in
  connectionist networks.
\newblock {\em Connection Science}.

\bibitem[French, 1999]{french1999catastrophic}
French, R.~M. (1999).
\newblock Catastrophic forgetting in connectionist networks.
\newblock {\em Trends in cognitive sciences}.

\bibitem[Ghiassian et~al., 2020]{ghiassian2020improving}
Ghiassian, S., Rafiee, B., Lo, Y.~L., and White, A. (2020).
\newblock Improving performance in reinforcement learning by breaking
  generalization in neural networks.
\newblock In {\em Proceedings of the 19th International Conference on
  Autonomous Agents and MultiAgent Systems}.

\bibitem[Goodfellow et~al., 2013]{goodfellow2013maxout}
Goodfellow, I., Warde-Farley, D., Mirza, M., Courville, A., and Bengio, Y.
  (2013).
\newblock Maxout networks.
\newblock In {\em International conference on machine learning}.

\bibitem[Guo et~al., 2016]{guo2016dynamic}
Guo, Y., Yao, A., and Chen, Y. (2016).
\newblock Dynamic network surgery for efficient {DNNs}.
\newblock {\em Advances in neural information processing systems}.

\bibitem[Haarnoja et~al., 2018]{haarnoja2018soft}
Haarnoja, T., Zhou, A., Abbeel, P., and Levine, S. (2018).
\newblock Soft actor-critic: Off-policy maximum entropy deep reinforcement
  learning with a stochastic actor.
\newblock In {\em International Conference on Machine Learning}, pages
  1861--1870.

\bibitem[He and Su, 2020]{he2020local}
He, H. and Su, W. (2020).
\newblock The local elasticity of neural networks.
\newblock In {\em International Conference on Learning Representations}.

\bibitem[Hessel et~al., 2018]{hessel2018rainbow}
Hessel, M., Modayil, J., Van~Hasselt, H., Schaul, T., Ostrovski, G., Dabney,
  W., Horgan, D., Piot, B., Azar, M., and Silver, D. (2018).
\newblock Rainbow: Combining improvements in deep reinforcement learning.
\newblock In {\em Proceedings of the AAAI conference on artificial
  intelligence}.

\bibitem[Hsu et~al., 2018]{hsu2018re}
Hsu, Y.-C., Liu, Y.-C., Ramasamy, A., and Kira, Z. (2018).
\newblock Re-evaluating continual learning scenarios: A categorization and case
  for strong baselines.
\newblock {\em arXiv preprint arXiv:1810.12488}.

\bibitem[Huang et~al., 2022]{huang2022cleanrl}
Huang, S., Dossa, R. F.~J., Ye, C., Braga, J., Chakraborty, D., Mehta, K., and
  Araújo, J.~G. (2022).
\newblock {CleanRL}: High-quality single-file implementations of deep
  reinforcement learning algorithms.
\newblock {\em Journal of Machine Learning Research}.

\bibitem[Hung et~al., 2019]{hung2019compacting}
Hung, C.-Y., Tu, C.-H., Wu, C.-E., Chen, C.-H., Chan, Y.-M., and Chen, C.-S.
  (2019).
\newblock Compacting, picking and growing for unforgetting continual learning.
\newblock {\em Advances in Neural Information Processing Systems}.

\bibitem[Jacot et~al., 2018]{jacot2018neural}
Jacot, A., Gabriel, F., and Hongler, C. (2018).
\newblock Neural tangent kernel: Convergence and generalization in neural
  networks.
\newblock {\em Advances in neural information processing systems}.

\bibitem[Jung et~al., 2023]{jung2022new}
Jung, D., Lee, D., Hong, S., Jang, H., Bae, H., and Yoon, S. (2023).
\newblock New insights for the stability-plasticity dilemma in online continual
  learning.
\newblock In {\em International Conference on Learning Representations}.

\bibitem[Khetarpal et~al., 2022]{khetarpal2022towards}
Khetarpal, K., Riemer, M., Rish, I., and Precup, D. (2022).
\newblock Towards continual reinforcement learning: A review and perspectives.
\newblock {\em Journal of Artificial Intelligence Research}.

\bibitem[Kingma and Ba, 2015]{kingma2015adam}
Kingma, D.~P. and Ba, J. (2015).
\newblock Adam: A method for stochastic optimization.
\newblock In {\em International Conference on Learning Representations}.

\bibitem[Kirkpatrick et~al., 2017]{kirkpatrick2017overcoming}
Kirkpatrick, J., Pascanu, R., Rabinowitz, N., Veness, J., Desjardins, G., Rusu,
  A.~A., Milan, K., Quan, J., Ramalho, T., Grabska-Barwinska, A., et~al.
  (2017).
\newblock Overcoming catastrophic forgetting in neural networks.
\newblock {\em Proceedings of the national academy of sciences}.

\bibitem[Kostrikov, 2021]{jaxrl}
Kostrikov, I. (2021).
\newblock {JAXRL}: Implementations of reinforcement learning algorithms in
  {JAX}.

\bibitem[Krizhevsky, 2009]{krizhevsky2009learning}
Krizhevsky, A. (2009).
\newblock Learning multiple layers of features from tiny images.
\newblock {\em Master's thesis, University of Toronto}.

\bibitem[Lan, 2019]{Explorer}
Lan, Q. (2019).
\newblock A pytorch reinforcement learning framework for exploring new ideas.
\newblock \url{https://github.com/qlan3/Explorer}.

\bibitem[Lan et~al., 2023]{lan2023memory}
Lan, Q., Pan, Y., Luo, J., and Mahmood, A.~R. (2023).
\newblock Memory-efficient reinforcement learning with value-based knowledge
  consolidation.
\newblock {\em Transactions on Machine Learning Research}.

\bibitem[Le et~al., 2017]{le2017learning}
Le, L., Kumaraswamy, R., and White, M. (2017).
\newblock Learning sparse representations in reinforcement learning with sparse
  coding.
\newblock {\em International Joint Conferences on Artificial Intelligence}.

\bibitem[Le and Yang, 2015]{le2015tiny}
Le, Y. and Yang, X. (2015).
\newblock Tiny imagenet visual recognition challenge.
\newblock {\em CS 231N}.

\bibitem[Lee et~al., 2021]{lee2021gst}
Lee, J., Kim, S., Kim, S., Jo, W., and Yoo, H.-J. (2021).
\newblock Gst: Group-sparse training for accelerating deep reinforcement
  learning.
\newblock {\em arXiv preprint arXiv:2101.09650}.

\bibitem[Li and Pathak, 2021]{li2021functional}
Li, A. and Pathak, D. (2021).
\newblock Functional regularization for reinforcement learning via learned
  {Fourier} features.
\newblock {\em Advances in Neural Information Processing Systems}.

\bibitem[Li et~al., 2019]{li2019learn}
Li, X., Zhou, Y., Wu, T., Socher, R., and Xiong, C. (2019).
\newblock Learn to grow: A continual structure learning framework for
  overcoming catastrophic forgetting.
\newblock In {\em International Conference on Machine Learning}.

\bibitem[Lin et~al., 2022]{lin2022towards}
Lin, G., Chu, H., and Lai, H. (2022).
\newblock Towards better plasticity-stability trade-off in incremental
  learning: A simple linear connector.
\newblock In {\em Conference on Computer Vision and Pattern Recognition}.

\bibitem[Liu et~al., 2020]{liu2020dynamic}
Liu, J., Xu, Z., Shi, R., Cheung, R. C.~C., and So, H.~K. (2020).
\newblock Dynamic sparse training: Find efficient sparse network from scratch
  with trainable masked layers.
\newblock In {\em International Conference on Learning Representations}.

\bibitem[Liu et~al., 2019]{liu2019utility}
Liu, V., Kumaraswamy, R., Le, L., and White, M. (2019).
\newblock The utility of sparse representations for control in reinforcement
  learning.
\newblock In {\em Proceedings of the AAAI Conference on Artificial
  Intelligence}.

\bibitem[Lyle et~al., 2023]{lyle2023understanding}
Lyle, C., Zheng, Z., Nikishin, E., Pires, B.~A., Pascanu, R., and Dabney, W.
  (2023).
\newblock Understanding plasticity in neural networks.
\newblock In {\em International Conference on Machine Learning}.

\bibitem[Madireddy et~al., 2023]{madireddy2023improving}
Madireddy, S., {Yanguas-Gil}, A., and Balaprakash, P. (2023).
\newblock Improving performance in continual learning tasks using bio-inspired
  architectures.
\newblock In {\em Conference on Lifelong Learning Agents}.

\bibitem[Mahmood et~al., 2018]{mahmood2018benchmarking}
Mahmood, A.~R., Korenkevych, D., Vasan, G., Ma, W., and Bergstra, J. (2018).
\newblock Benchmarking reinforcement learning algorithms on real-world robots.
\newblock In {\em Conference on robot learning}.

\bibitem[Mallya and Lazebnik, 2018]{mallya2018packnet}
Mallya, A. and Lazebnik, S. (2018).
\newblock {PackNet}: Adding multiple tasks to a single network by iterative
  pruning.
\newblock In {\em Proceedings of the IEEE conference on Computer Vision and
  Pattern Recognition}.

\bibitem[Masana et~al., 2021]{masana2021ternary}
Masana, M., Tuytelaars, T., and Van~de Weijer, J. (2021).
\newblock Ternary feature masks: zero-forgetting for task-incremental learning.
\newblock In {\em Proceedings of the IEEE/CVF Conference on Computer Vision and
  Pattern Recognition}.

\bibitem[Mehta et~al., 2021]{mehta2021extreme}
Mehta, H., Cutkosky, A., and Neyshabur, B. (2021).
\newblock Extreme memorization via scale of initialization.
\newblock In {\em International Conference on Learning Representations}.

\bibitem[Mendez et~al., 2020]{mendez2020lifelong}
Mendez, J., Wang, B., and Eaton, E. (2020).
\newblock Lifelong policy gradient learning of factored policies for faster
  training without forgetting.
\newblock {\em Advances in Neural Information Processing Systems}.

\bibitem[Mendez et~al., 2022]{mendez2022modular}
Mendez, J.~A., {van}~Seijen, H., and Eaton, E. (2022).
\newblock Modular lifelong reinforcement learning via neural composition.
\newblock In {\em International Conference on Learning Representations}.

\bibitem[Mermillod et~al., 2013]{mermillod2013stabilityplasticity}
Mermillod, M., Bugaiska, A., and BONIN, P. (2013).
\newblock The stability-plasticity dilemma: investigating the continuum from
  catastrophic forgetting to age-limited learning effects.
\newblock {\em Frontiers in Psychology}.

\bibitem[Mirzadeh et~al., 2022a]{mirzadeh2022wide}
Mirzadeh, S.~I., Chaudhry, A., Yin, D., Hu, H., Pascanu, R., Gorur, D., and
  Farajtabar, M. (2022a).
\newblock Wide neural networks forget less catastrophically.
\newblock In {\em International Conference on Machine Learning}.

\bibitem[Mirzadeh et~al., 2022b]{mirzadeh2022architecture}
Mirzadeh, S.~I., Chaudhry, A., Yin, D., Nguyen, T., Pascanu, R., Gorur, D., and
  Farajtabar, M. (2022b).
\newblock Architecture matters in continual learning.
\newblock {\em arXiv preprint arXiv:2202.00275}.

\bibitem[Mirzadeh et~al., 2020]{mirzadeh2020dropout}
Mirzadeh, S.~I., Farajtabar, M., and Ghasemzadeh, H. (2020).
\newblock Dropout as an implicit gating mechanism for continual learning.
\newblock In {\em Proceedings of the IEEE/CVF Conference on Computer Vision and
  Pattern Recognition Workshops}.

\bibitem[Mnih et~al., 2013]{mnih2013playing}
Mnih, V., Kavukcuoglu, K., Silver, D., Graves, A., and Antonoglou, I. (2013).
\newblock Playing {Atari} with deep reinforcement learning.
\newblock In {\em NIPS Deep Learning Workshop}.

\bibitem[Mnih et~al., 2015]{mnih2015human}
Mnih, V., Kavukcuoglu, K., Silver, D., Rusu, A.~A., Veness, J., Bellemare,
  M.~G., Graves, A., Riedmiller, M., Fidjeland, A.~K., Ostrovski, G., Petersen,
  S., Beattie, C., Sadik, A., Antonoglou, I., King, H., Kumaran, D., Wierstra,
  D., Legg, S., and Hassabis, D. (2015).
\newblock Human-level control through deep reinforcement learning.
\newblock {\em Nature}.

\bibitem[Ostapenko et~al., 2019]{ostapenko2019learning}
Ostapenko, O., Puscas, M., Klein, T., Jahnichen, P., and Nabi, M. (2019).
\newblock Learning to remember: A synaptic plasticity driven framework for
  continual learning.
\newblock In {\em Proceedings of the IEEE/CVF conference on computer vision and
  pattern recognition}.

\bibitem[Pan et~al., 2022]{pan2020fuzzy}
Pan, Y., Banman, K., and White, M. (2022).
\newblock Fuzzy tiling activations: A simple approach to learning sparse
  representations online.
\newblock In {\em International Conference on Learning Representations}.

\bibitem[Riemer et~al., 2018]{riemer2018learning}
Riemer, M., Cases, I., Ajemian, R., Liu, M., Rish, I., Tu, Y., and Tesauro, G.
  (2018).
\newblock Learning to learn without forgetting by maximizing transfer and
  minimizing interference.
\newblock In {\em International Conference on Learning Representations}.

\bibitem[Ring, 1994]{ring1994continual}
Ring, M.~B. (1994).
\newblock {\em Continual learning in reinforcement environments}.
\newblock The University of Texas at Austin.

\bibitem[Rivest and Precup, 2003]{rivest2003combining}
Rivest, F. and Precup, D. (2003).
\newblock Combining td-learning with cascade-correlation networks.
\newblock In {\em International Conference on Machine Learning}.

\bibitem[Russakovsky et~al., 2015]{imagenet15russakovsky}
Russakovsky, O., Deng, J., Su, H., Krause, J., Satheesh, S., Ma, S., Huang, Z.,
  Karpathy, A., Khosla, A., Bernstein, M., Berg, A.~C., and Fei-Fei, L. (2015).
\newblock {ImageNet Large Scale Visual Recognition Challenge}.
\newblock {\em International Journal of Computer Vision}.

\bibitem[Rusu et~al., 2016]{rusu2016progressive}
Rusu, A.~A., Rabinowitz, N.~C., Desjardins, G., Soyer, H., Kirkpatrick, J.,
  Kavukcuoglu, K., Pascanu, R., and Hadsell, R. (2016).
\newblock Progressive neural networks.
\newblock {\em arXiv preprint arXiv:1606.04671}.

\bibitem[Sarfraz et~al., 2023]{sarfraz2023sparse}
Sarfraz, F., Arani, E., and Zonooz, B. (2023).
\newblock Sparse coding in a dual memory system for lifelong learning.
\newblock In {\em Proceedings of the AAAI Conference on Artificial
  Intelligence}.

\bibitem[Schulman et~al., 2017]{schulman2017proximal}
Schulman, J., Wolski, F., Dhariwal, P., Radford, A., and Klimov, O. (2017).
\newblock Proximal policy optimization algorithms.
\newblock {\em arXiv preprint arXiv:1707.06347}.

\bibitem[Schwarz et~al., 2018]{schwarz2018progress}
Schwarz, J., Czarnecki, W., Luketina, J., Grabska-Barwinska, A., Teh, Y.~W.,
  Pascanu, R., and Hadsell, R. (2018).
\newblock Progress \& compress: A scalable framework for continual learning.
\newblock In {\em International Conference on Machine Learning}.

\bibitem[Serra et~al., 2018]{serra2018overcoming}
Serra, J., Suris, D., Miron, M., and Karatzoglou, A. (2018).
\newblock Overcoming catastrophic forgetting with hard attention to the task.
\newblock In {\em International Conference on Machine Learning}.

\bibitem[Shen et~al., 2021]{shen2021algorithmic}
Shen, Y., Dasgupta, S., and Navlakha, S. (2021).
\newblock Algorithmic insights on continual learning from fruit flies.
\newblock {\em arXiv preprint arXiv:2107.07617}.

\bibitem[Sokar et~al., 2021]{sokar2021spacenet}
Sokar, G., Mocanu, D.~C., and Pechenizkiy, M. (2021).
\newblock {SpaceNet}: Make free space for continual learning.
\newblock {\em Neurocomputing}.

\bibitem[Sokar et~al., 2022]{sokar2022dynamic}
Sokar, G., Mocanu, E., Mocanu, D.~C., Pechenizkiy, M., and Stone, P. (2022).
\newblock Dynamic sparse training for deep reinforcement learning.
\newblock In {\em International Joint Conference on Artificial Intelligence}.

\bibitem[Srivastava et~al., 2013]{srivastava2013compete}
Srivastava, R.~K., Masci, J., Kazerounian, S., Gomez, F., and Schmidhuber, J.
  (2013).
\newblock Compete to compute.
\newblock {\em Advances in neural information processing systems}.

\bibitem[Sutton and Barto, 2018]{sutton2011reinforcement}
Sutton, R.~S. and Barto, A.~G. (2018).
\newblock {\em {Reinforcement Learning: An Introduction}}.
\newblock MIT Press, second edition.

\bibitem[Tan et~al., 2023]{tan2023rlx2}
Tan, Y., Hu, P., Pan, L., Huang, J., and Huang, L. (2023).
\newblock {RL}x2: Training a sparse deep reinforcement learning model from
  scratch.
\newblock In {\em International Conference on Learning Representations}.

\bibitem[Tasfi, 2016]{tasfi2016PLE}
Tasfi, N. (2016).
\newblock {PyGame} learning environment.
\newblock \url{https://github.com/ntasfi/PyGame-Learning-Environment}.

\bibitem[Tieleman and Hinton, 2012]{tieleman2012rmsprop}
Tieleman, T. and Hinton, G. (2012).
\newblock Lecture 6.5-{RMSProp}: Divide the gradient by a running average of
  its recent magnitude.
\newblock {\em COURSERA Neural Networks Neural Networks for Machine Learning}.

\bibitem[Towers et~al., 2023]{towers2023gymnasium}
Towers, M., Terry, J.~K., Kwiatkowski, A., Balis, J.~U., Cola, G.~d., Deleu,
  T., Goulão, M., Kallinteris, A., KG, A., Krimmel, M., Perez-Vicente, R.,
  Pierré, A., Schulhoff, S., Tai, J.~J., Shen, A. T.~J., and Younis, O.~G.
  (2023).
\newblock Gymnasium.
\newblock \url{https://zenodo.org/record/8127025}.

\bibitem[Trockman and Kolter, 2022]{trockman2022patches}
Trockman, A. and Kolter, J.~Z. (2022).
\newblock Patches are all you need?
\newblock {\em arXiv preprint arXiv:2201.09792}.

\bibitem[Van~de Ven and Tolias, 2019]{van2019three}
Van~de Ven, G.~M. and Tolias, A.~S. (2019).
\newblock Three scenarios for continual learning.
\newblock {\em arXiv preprint arXiv:1904.07734}.

\bibitem[Vasan et~al., 2024]{vasan2024deep}
Vasan, G., Elsayed, M., Azimi, S.~A., He, J., Shahriar, F., Bellinger, C.,
  White, M., and Mahmood, R. (2024).
\newblock Deep policy gradient methods without batch updates, target networks,
  or replay buffers.
\newblock {\em Advances in Neural Information Processing Systems}.

\bibitem[Wang et~al., 2022]{wang2022sparcl}
Wang, Z., Zhan, Z., Gong, Y., Yuan, G., Niu, W., Jian, T., Ren, B., Ioannidis,
  S., Wang, Y., and Dy, J. (2022).
\newblock {SparCL}: Sparse continual learning on the edge.
\newblock In {\em Advances in Neural Information Processing Systems}.

\bibitem[Weng et~al., 2022]{tianshou}
Weng, J., Chen, H., Yan, D., You, K., Duburcq, A., Zhang, M., Su, Y., Su, H.,
  and Zhu, J. (2022).
\newblock Tianshou: A highly modularized deep reinforcement learning library.
\newblock {\em Journal of Machine Learning Research}.

\bibitem[Wolfe and Kyrillidis, 2022]{wolfe2022cold}
Wolfe, C.~R. and Kyrillidis, A. (2022).
\newblock Cold start streaming learning for deep networks.
\newblock {\em arXiv preprint arXiv:2211.04624}.

\bibitem[Yarats and Kostrikov, 2020]{pytorch_sac}
Yarats, D. and Kostrikov, I. (2020).
\newblock Soft actor-critic (sac) implementation in pytorch.
\newblock \url{https://github.com/denisyarats/pytorch_sac}.

\bibitem[Yoon et~al., 2018]{yoon2018lifelong}
Yoon, J., Yang, E., Lee, J., and Hwang, S.~J. (2018).
\newblock Lifelong learning with dynamically expandable networks.
\newblock In {\em International Conference on Learning Representations}.

\bibitem[Zhou et~al., 2020]{zhou2020go}
Zhou, D., Ye, M., Chen, C., Meng, T., Tan, M., Song, X., Le, Q., Liu, Q., and
  Schuurmans, D. (2020).
\newblock Go wide, then narrow: Efficient training of deep thin networks.
\newblock In {\em International Conference on Machine Learning}.

\end{thebibliography}
\bibliographystyle{apalike}
\newpage
\appendix

\section{Proofs}\label{appendix:proof}

\paragraph{Derivation of~\cref{eq:taylor}}
By Taylor expansion, we have
\begin{align*}
& f_{\vw'}(\vx) - f_{\vw}(\vx) \\
=& f_{\vw + \Delta_\vw}(\vx) - f_{\vw}(\vx) \\
=& f_{\vw}(\vx) + \langle \nabla_\vw f_{\vw}(\vx), \Delta_\vw \rangle + O(\Delta_\vw^2) - f_{\vw}(\vx) \\
=& \langle \nabla_\vw f_{\vw}(\vx), \Delta_\vw \rangle + O(\Delta_\vw^2) \\
=& -\alpha \nabla_f L(f,F,\vx_t) \, \langle \nabla_\vw f_{\vw}(\vx), \nabla_\vw f_{\vw}(\vx_t) \rangle + O(\Delta_\vw^2).
\end{align*}

\firstntklem*
\begin{proof}
Let $\circ$ denote Hadamard product. By definition, we have
\begin{align*}
&\langle \nabla_\vw f_{\vw}(\vx), \nabla_\vw f_{\vw}(\vx_t) \rangle \\
&= \langle \nabla_u f_{\vw}(\vx), \nabla_u f_{\vw}(\vx_t) \rangle + \langle \nabla_V f_{\vw}(\vx), \nabla_V f_{\vw}(\vx_t) \rangle + \langle \nabla_b f_{\vw}(\vx), \nabla_b f_{\vw}(\vx_t) \rangle \\
&= \langle \sigma(\rmV \vx + \vb), \sigma(\rmV \vx_t + \vb) \rangle + \langle \vu \circ \sigma'(\rmV \vx + \vb) \vx^\top, \vu \circ \sigma'(\rmV \vx_t + \vb) \vx_t^\top \rangle \\
&+ \langle \vu \circ \sigma'(\rmV \vx + \vb), \vu \circ \sigma'(\rmV \vx_t + \vb) \rangle \\
&= \sigma(\rmV \vx + \vb)^\top \sigma(\rmV \vx_t + \vb) + (\vx^\top \vx_t + 1) \left(\vu \circ \sigma'(\rmV \vx + \vb) \right)^\top \left( \vu \circ \sigma'(\mathbf{\rmV} \vx_t + \vb) \right), \\
&= \sigma(\rmV \vx + \vb)^\top \sigma(\rmV \vx_t + \vb) + \vu^\top \vu (\vx^\top \vx_t + 1) \sigma'(\rmV \vx + \vb)^\top \sigma'(\rmV \vx_t + \vb).
\end{align*}
\end{proof}

\begin{lemma}\label{lem:elephant}
$\elephant(x)$ and $\elephant'(x)$ are sparse functions.
\end{lemma}

\begin{proof}
Without loss of generality, we set $h=1$.
So $\elephant(x) = \frac{1}{1+|\frac{x}{a}|^d}$ and $\abs{\elephant'(x)} = \frac{d}{a} \abs{\frac{x}{a}}^{d-1} (\frac{1}{1+\abs{\frac{x}{a}}^d})^2$.
For $0 < \epsilon < 1$, easy to verify that
\begin{equation*}
\abs{x} \ge a (\frac{1}{\epsilon} - 1)^{1/d} \Longrightarrow \elephant(x) \le \epsilon
\quad \text{and} \quad
\abs{x} \ge \frac{d}{2\epsilon} \Longrightarrow \abs{\elephant'(x)} \le \epsilon.
\end{equation*}

For $C > a (\frac{1}{\epsilon} - 1)^{1/d}$,
we have $S_{\epsilon,C}(\elephant) \ge \frac{C-a (\frac{1}{\epsilon} - 1)^{1/d}}{C}$,
thus
\begin{align*}
S(\elephant)
= \lim_{\epsilon \rightarrow 0^+} \lim_{C \rightarrow \infty} S_{\epsilon,C}(\elephant)
\ge \lim_{\epsilon \rightarrow 0^+} \lim_{C \rightarrow \infty} \frac{C-a (\frac{1}{\epsilon} - 1)^{1/d}}{C} 
= \lim_{\epsilon \rightarrow 0^+} 1
= 1.
\end{align*}

Similarly, for $C > \frac{d}{2\epsilon}$,
we have $S_{\epsilon,C}(\elephant') \ge \frac{C - \frac{d}{2\epsilon}}{C}$,
thus
\begin{align*}
S(\elephant')
= \lim_{\epsilon \rightarrow 0^+} \lim_{C \rightarrow \infty} S_{\epsilon,C}(\elephant')
\ge \lim_{\epsilon \rightarrow 0^+} \lim_{C \rightarrow \infty} \frac{C - \frac{d}{2\epsilon}}{C} 
= \lim_{\epsilon \rightarrow 0^+} 1
= 1.
\end{align*}

Note that $S(\elephant) \le 1$ and $S(\elephant') \le 1$. Together, we conclude that $S(\elephant)=1$ and $S(\elephant')=1$; $\elephant(x)$ and $\elephant'(x)$ are sparse functions.
\end{proof}

\firstthmmain*
\begin{proof}
When $d \rightarrow \infty$, the elephant function is a rectangular function, i.e.
\begin{align*}
\sigma(x) = \operatorname{rect}(x) =
\begin{cases}
1, & \abs{x} < a, \\
\frac{1}{2}, & \abs{x} = a, \\
0, & \abs{x} > a. \\
\end{cases}
\end{align*}
In this case, it is easy to verify that $\forall \, x, y \in \R$, $\abs{x-y} > 2a$, we have $\sigma(x) \sigma(y) = 0$ and $\sigma'(x) \sigma'(y) = 0$.
Denote $\Delta_\vx = \vx - \vx_t$.
Then when $\abs{\rmV \Delta_\vx} \succ 2 a \mathbf{1}_{m}$, we have 
$\sigma(\rmV \vx + \vb)^\top \sigma(\rmV \vx_t + \vb) = 0$ and 
$\sigma'(\rmV \vx + \vb)^\top \sigma'(\rmV \vx_t + \vb) = 0$.
In other words, when $\vx$ and $\vx_t$ are dissimilar in the sense that $\abs{\rmV (\vx - \vx_t)} \succ 2 a \mathbf{1}_{m}$, we have $\langle \nabla_\vw f_{\vw}(\vx), \nabla_\vw f_{\vw}(\vx_t) \rangle = 0$.
\end{proof}

\section{Function and Gradient Sparsity of Activation Functions}\label{appendix:activation}

\begin{figure}[htbp]
\centering
\subcaptionbox{$\relu$}{
\includegraphics[width=\figwidthtwo]{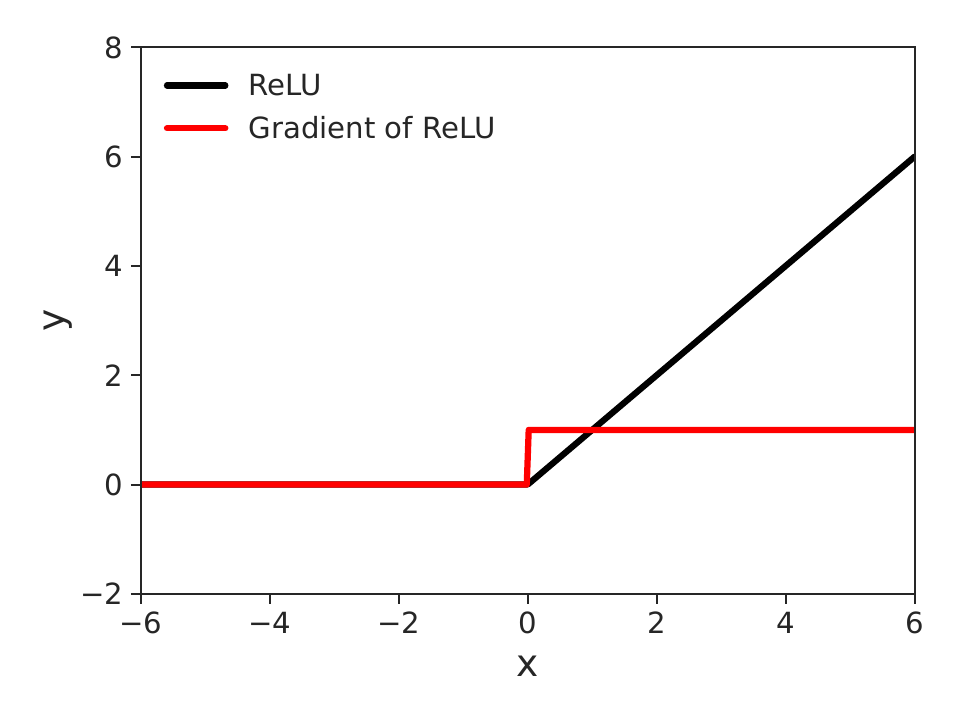}}
\subcaptionbox{$\elu$}{
\includegraphics[width=\figwidthtwo]{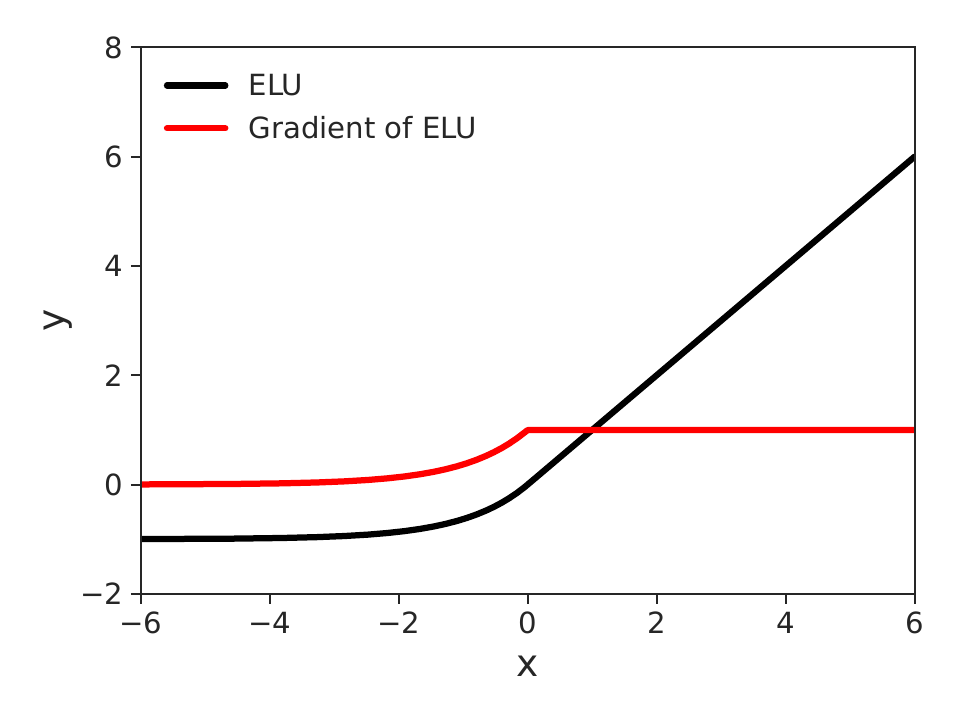}}
\subcaptionbox{$\sigmoid$}{
\includegraphics[width=\figwidthtwo]{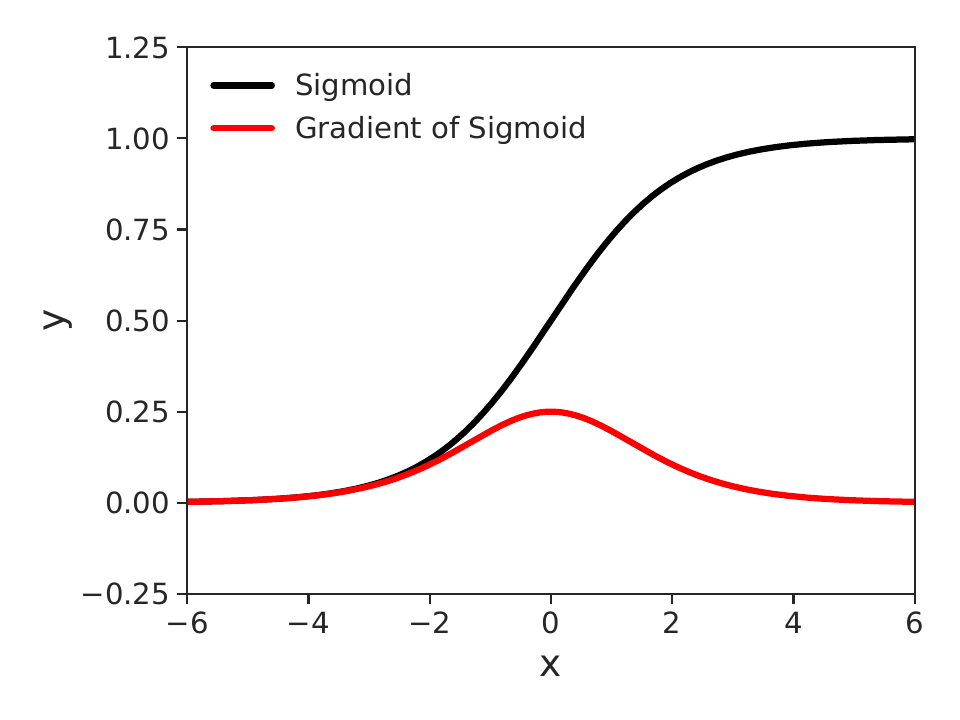}}
\subcaptionbox{$\tanh$}{
\includegraphics[width=\figwidthtwo]{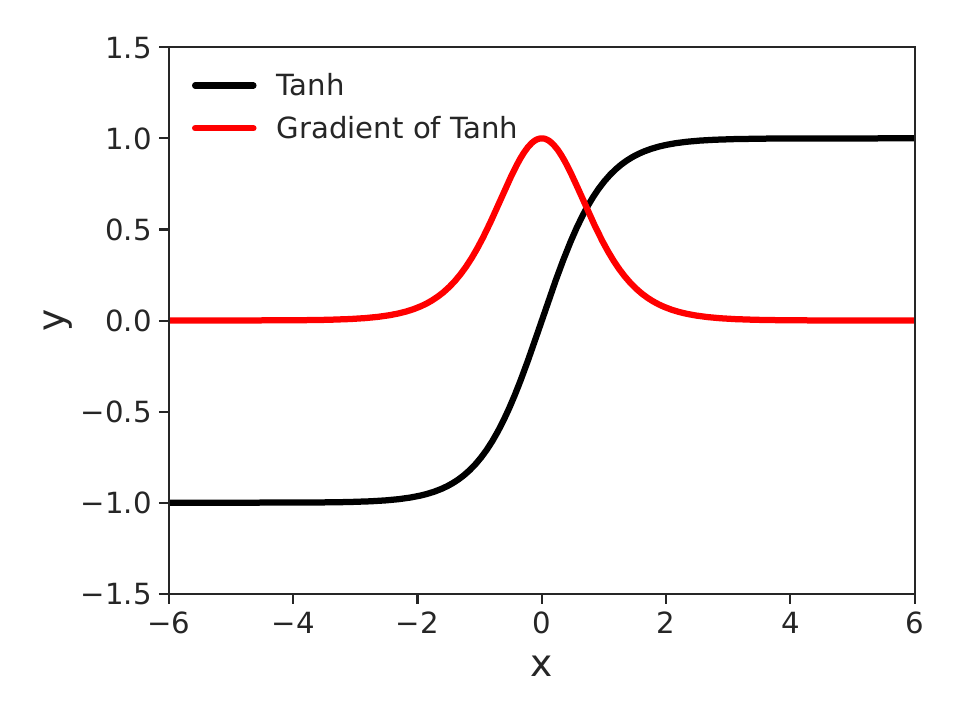}}
\caption{Visualizations of common activation functions and their gradients.}
\label{fig:act}
\end{figure}

\begin{table}[htbp]
\vspace{-1em}
\caption{The function sparsity and gradient sparsity of various activation functions. Among them, only $\elephant$ is sparse in terms of both function values and gradient values, according to \cref{def:sparse} and \cref{lem:elephant}.}
\label{tb:activation}
\centering
\begin{tabular}{lcc}
\toprule
Activation & Function Sparsity & Gradient Sparsity \\
\midrule
$\relu$     & 1/2 & 1/2 \\
$\sigmoid$  & 1/2 & 1   \\
$\tanh$     & 0 & 1 \\
$\elu$      & 0 & 1/2 \\
$\elephant$ & 1 & 1 \\
\bottomrule
\end{tabular}
\vspace{-1em}
\end{table}

\section{Experimental Details}\label{appendix:exp_detail}

All simulation experiments are conducted on Intel Gold 6148 Skylake CPUs and V100-16GB GPUs.
The total computation required to reproduce all experiments is around $18$ CPU months and $16$ GPU months.
However, the exact amount of the full research project is hard to estimate but it should be greater than this number due to preliminary and failed experiments.

For elephant activation functions, we use the following setup unless explicitly stated otherwise.
Specifically, since the activation area of $\elephant$ is narrow and centered around zero (i.e., from $-a$ to $a$), we apply layer normalization~\citep{ba2016layer} to normalize the input vector in order to get a better trade-off between stability and plasticity~\citep{mermillod2013stabilityplasticity,lin2022towards,jung2022new,lyle2023understanding}.
We also make some parameters of $\elephant$ adaptive and learnable, assigning an individual elephant activation function to each unit of a neural network.
Particularly, all elephant functions are initialized with the same hyper-parameters at the beginning of training. During training, we optimize $h$ and $a$ using gradient descent while keeping $d$ fixed for each elephant function (see~\cref{eq:elephant}).
Compared to fixing $h$ and $a$ during training, this approach is shown to match or even improve performance in our preliminary experiments.
Furthermore, when $\elephant$ is applied, to improve the diversity of initially generated features, we initialize bias values in a linear layer with evenly spaced numbers over the interval $[-\sqrt{3} \sigma_{bias}, \sqrt{3} \sigma_{bias}]$, where $\sigma_{bias}$ is a positive hyper-parameter.
When other activation functions are used, bias values are initialized with zeros.
For weight values in a linear layer, we follow the default initialization as PyTorch by sampling weight values from a uniform distribution $U[-\sqrt{k}, \sqrt{k}]$, where $k=1/\text{in\_features}$.

\subsection{Streaming Learning for Regression}\label{appendix:streaming}

In this experiment, we use an MLP with one hidden layer of size $1,000$.
For each new arriving example, we perform $10$ updates with Adam~\citep{kingma2015adam} optimizer.
The best learning rate is selected from $\{3e-3, 1e-3, 3e-4, 1e-4, 3e-5, 1e-5\}$.
For $\elephant$, we set $d=8$, $h=1$, $a=0.08$, and $\sigma_{bias}=1.28$.
For clarity of demonstration, we omit layer normalization before applying $\elephant$ and fix $h$ and $a$ during training, avoiding their influence on the visualization of the NTK functions.
For $\srnn$, we apply various classical activation functions ($\relu$, $\sigmoid$, $\elu$, and $\tanh$), tune Set KL loss weight $\lambda$ and $\beta$ (see~\citet{liu2019utility} for details), and present the best result in this work.
Specifically, we choose $\lambda$ from $\{0, 0.1, 0.01, 0.001\}$ and choose $\beta$ from $\{0.05, 0.1, 0.2\}$.

We present the test MSEs of various methods in \cref{tb:sin}. Lower is better. All results are averaged over $5$ runs, reported with standard errors. For $\srnn$, we present the best result among the combinations of $\srnn$s and classical activation functions.

\begin{table}[htbp]
\vspace{-1em}
\caption{The test MSEs of various methods in streaming learning for a simple regression task.}
\label{tb:sin}
\centering
\begin{tabular}{lc}
\toprule
Method & Test Performance (MSE) \\
\midrule
$\relu$      & $0.4729 \pm 0.0110$ \\
$\sigmoid$   & $0.4583 \pm 0.0008$ \\
$\tanh$      & $0.4461 \pm 0.0013$ \\
$\elu$       & $0.4521 \pm 0.0019$ \\
$\srnn$      & $0.4061 \pm 0.0036$ \\
$\elephant$  & \textbf{0.0081} $\pm$ \textbf{0.0009} \\
\bottomrule
\end{tabular}
\vspace{-1em}
\end{table}

\subsection{Reinforcement Learning}\label{appendix:rl}

\subsubsection{Classical RL Tasks}\label{appendix:classical}

We implement DQN with Jax~\citep{jax2018github} from scratch based on \citet{Explorer}.
We apply RMSProp~\citep{tieleman2012rmsprop} with a decay rate of $0.999$ for optimization.
A best learning rate is selected from $\{1e-2, 3e-3, 1e-3, 3e-4, 1e-4, 3e-5, 1e-5, 3e-6\}$ for each hyper-parameter setup.
For buffer sizes, we consider values in $\{32, 1e2, 3e2, 1e3, 3e3, 1e4\}$, where the default buffer size is $1e4$.
The discount factor $\gamma=0.99$. The mini-batch size is $32$.

For $\maxout$ and $\lwta$, we set $k=5$.
For $\fta$, we set $k=20$ and $[l,u]=[-20,20]$ following~\citet{pan2020fuzzy}.
For $\elephant$, to make a fair comparison, we use the same set of hyper-parameters for all DQN experiments in classical RL tasks, i.e., $d=4$, $h=1$, $a=0.2$, and $\sigma_{bias}=0.4$.
Note that tuning the hyper-parameters of $\elephant$ for each task and buffer size could further improve the performance.
The default network is an MLP with one hidden layer of size $1,000$ in all tasks.
However, when different activation functions are applied, we adjust the width of the hidden layer so that the total number of parameters is close to each other.

\subsubsection{Atari Tasks}\label{appendix:atari}

Our implementations of Atari DQN and Rainbow are adapted from the implementation of Atari DQN in Tianshou~\citep{tianshou}.
We also follow the default hyper-parameters and training setups as in Tianshou unless explicitly stated otherwise.
The mini-batch size is $32$. The discount factor is $0.99$.
Adam~\citep{kingma2015adam} is applied to optimize.
We train all agents for 50M frames (i.e., 12.5M steps).
For DQN, the learning rate is $1e-4$ while it is $6.25e-5$ for Rainbow.
The buffer size is 1 million.

For $\maxout$ and $\lwta$, we set $k=4$.
For $\fta$, we set $k=20$ and $[l,u]=[-20,20]$ following~\citet{pan2020fuzzy}.
We adjust the width of the hidden layer so that the total number of parameters is close to each other when different activation functions are applied.
To make a fair comparison, we use the same hyper-parameters for $\elephant$ (i.e., $d=4$, $h=1$, $a=0.1$, and $\sigma_{bias}=2$) in all Atari experiments, although tuning them for each task could further boost the performance.
Furthermore, since Rainbow uses noisy linear layers~\citep{fortunato2018noisy}, we follow the default initialization in Rainbow and do not reinitialize the weights and biases when $\elephant$ is applied in Rainbow.
Note that for both DQN and Rainbow, we only apply these activation functions to the penultimate layer; the CNN feature network still uses $\relu$ as the default activation function.

We present the detailed test performance of all 10 Atari tasks in \cref{tab:atari_individual} and the return curves in \cref{fig:atari_individual}.

\begin{table}[htbp]
\centering
\caption{The performance comparison of DQN and Rainbow with different activation functions across 10 Atari tasks. We report the final test returns averaged over 5 runs as well as the correspond 95\% confidence intervals.}
\label{tab:atari_individual}
\subcaptionbox{DQN \label{tab:atari_dqn_individual}}{
\resizebox{\linewidth}{!}{
\begin{tabular}{lcccccc}
\toprule
\bf \diagbox{Task}{Activation} & $\relu$ & $\tanh$ & $\maxout$ & $\lwta$ & $\fta$ & $\elephant$ \\
\midrule
Amidar & 309 $\pm$ 30 & 171 $\pm$ 9 & 620 $\pm$ 133 & 462 $\pm$ 40 & 203 $\pm$ 2 & 912 $\pm$ 88 \\
Battlezone & 21664 $\pm$ 3696 & 4148 $\pm$ 3084 & 19028 $\pm$ 6012 & 21852 $\pm$ 1924 & 4760 $\pm$ 1504 & 26176 $\pm$ 2192 \\
Bowling & 44 $\pm$ 8 & 2 $\pm$ 2 & 22 $\pm$ 6 & 33 $\pm$ 7 & 20 $\pm$ 7 & 44 $\pm$ 11 \\
Double Dunk & -1.9 $\pm$ 0.5 & -1.4 $\pm$ 0.2 & -5.2 $\pm$ 2.4 & -6.0 $\pm$ 1.3 & -2.2 $\pm$ 0.5 & -5.0 $\pm$ 1.4 \\
Frostbite & 3074 $\pm$ 167 & 1501 $\pm$ 628 & 3734 $\pm$ 505 & 3510 $\pm$ 208 & 2613 $\pm$ 433 & 6001 $\pm$ 724 \\
Kung-Fu Master & 12996 $\pm$ 9809 & 0 $\pm$ 0 & 22563 $\pm$ 2230 & 10731 $\pm$ 7698 & 12251 $\pm$ 3113 & 28411 $\pm$ 2065 \\
Name This Game & 3573 $\pm$ 484 & 3889 $\pm$ 488 & 5261 $\pm$ 761 & 5890 $\pm$ 551 & 4751 $\pm$ 383 & 4847 $\pm$ 435 \\
Phoenix & 4346 $\pm$ 246 & 3060 $\pm$ 632 & 8421 $\pm$ 829 & 7659 $\pm$ 1544 & 14055 $\pm$ 1792 & 11110 $\pm$ 2387 \\
Q*bert & 11081 $\pm$ 1271 & 7697 $\pm$ 1824 & 14469 $\pm$ 898 & 15228 $\pm$ 638 & 10195 $\pm$ 1007 & 15454 $\pm$ 84 \\
River Raid & 9560 $\pm$ 290 & 4802 $\pm$ 319 & 9405 $\pm$ 243 & 9669 $\pm$ 288 & 6875 $\pm$ 110 & 14339 $\pm$ 1148 \\
\bottomrule
\end{tabular}
}}
\subcaptionbox{Rainbow \label{tab:atari_rainbow_individual}}{
\resizebox{\linewidth}{!}{
\begin{tabular}{lcccccc}
\toprule
\bf \diagbox{Task}{Activation} & $\relu$ & $\tanh$ & $\maxout$ & $\lwta$ & $\fta$ & $\elephant$ \\
\midrule
Amidar & 300 $\pm$ 37 & 285 $\pm$ 27 & 354 $\pm$ 64 & 309 $\pm$ 96 & 211 $\pm$ 12 & 402 $\pm$ 71 \\
BattleZone & 24104 $\pm$ 2224 & 9844 $\pm$ 8764 & 22320 $\pm$ 1712 & 24308 $\pm$ 3284 & 16176 $\pm$ 1520 & 19600 $\pm$ 4088 \\
Bowling & 27 $\pm$ 3 & 8 $\pm$ 5 & 31 $\pm$ 1 & 30 $\pm$ 1 & 26 $\pm$ 4 & 24 $\pm$ 7 \\
DoubleDunk & -1.9 $\pm$ 0.6 & -1.9 $\pm$ 0.4 & -2.0 $\pm$ 0.3 & -1.9 $\pm$ 0.6 & -2.0 $\pm$ 0.3 & -2.0 $\pm$ 0.3 \\
Frostbite & 2825 $\pm$ 308 & 2604 $\pm$ 662 & 3747 $\pm$ 399 & 2706 $\pm$ 913 & 292 $\pm$ 17 & 3961 $\pm$ 362 \\
KungFuMaster & 24583 $\pm$ 1642 & 17225 $\pm$ 1413 & 21275 $\pm$ 1324 & 23064 $\pm$ 2794 & 18502 $\pm$ 1282 & 25421 $\pm$ 2004 \\
NameThisGame & 12321 $\pm$ 713 & 10833 $\pm$ 1155 & 12746 $\pm$ 199 & 13425 $\pm$ 734 & 9616 $\pm$ 536 & 11384 $\pm$ 444 \\
Phoenix & 6442 $\pm$ 474 & 10649 $\pm$ 1170 & 15024 $\pm$ 1894 & 14408 $\pm$ 951 & 8207 $\pm$ 3642 & 26033 $\pm$ 5962 \\
Qbert & 14477 $\pm$ 467 & 14616 $\pm$ 514 & 15537 $\pm$ 377 & 15110 $\pm$ 880 & 11309 $\pm$ 4481 & 16160 $\pm$ 794 \\
Riverraid & 10054 $\pm$ 564 & 8109 $\pm$ 455 & 11127 $\pm$ 540 & 10504 $\pm$ 309 & 7427 $\pm$ 979 & 9563 $\pm$ 487 \\
\bottomrule
\end{tabular}
}}
\end{table}

\begin{figure}[htbp]
\centering
\includegraphics[width=0.7\textwidth]{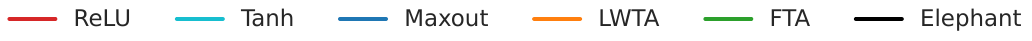}
\subcaptionbox{DQN \label{fig:atari_dqn_individual}}{\includegraphics[width=\figwidthtwo]{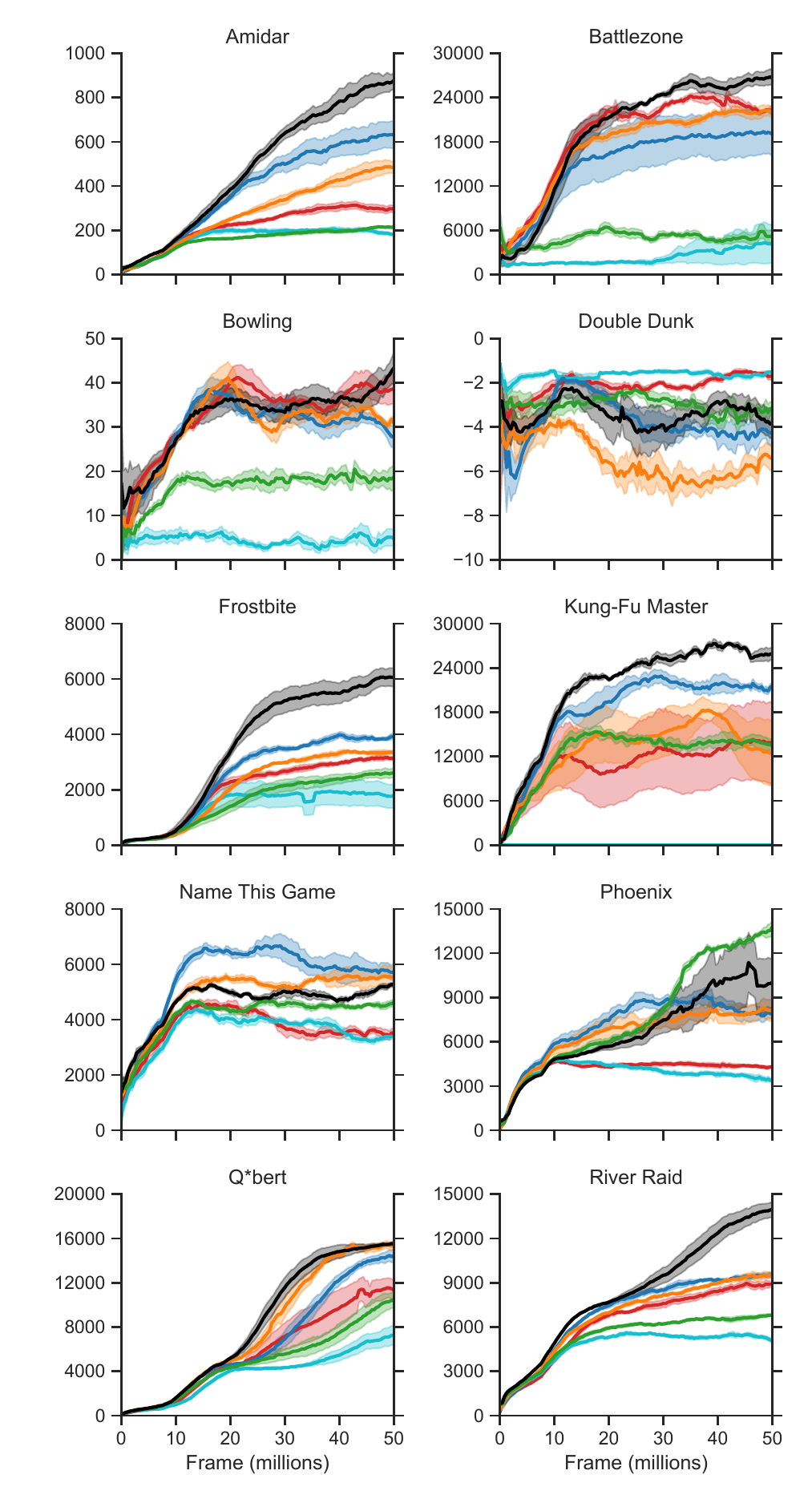}}
\subcaptionbox{Rainbow \label{fig:atari_rainbow_individual}}{\includegraphics[width=\figwidthtwo]{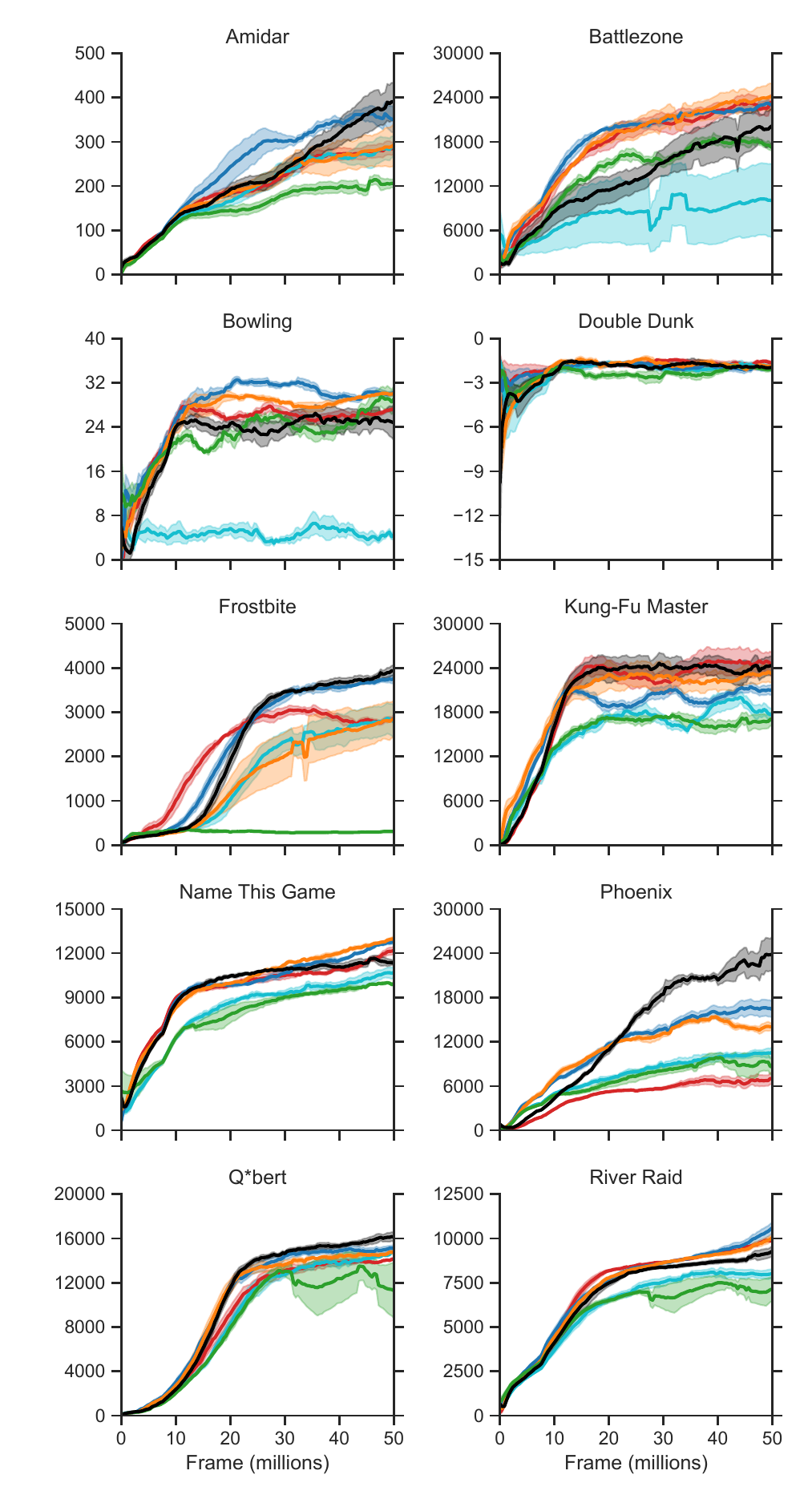}}
\caption{The return curves of DQN and Rainbow with various activations in 10 Atari tasks over 50 million training frames during test. Solid lines correspond to the median performance over 5 random seeds, and the shaded areas correspond to standard errors.}
\label{fig:atari_individual}
\end{figure}

\subsection{The Gradient Covariance Heatmaps of Training DQN in Atari Games}\label{appendix:grad}

From \cref{fig:atari_grad_dqn:Amidar} to \cref{fig:atari_grad_dqn:Riverraid}, we present the heatmaps for training DQN in all 10 Atari tasks with 6 activation functions.
Specifically, we set $k=32$ and estimate the gradient covariance matrices at the midpoint (Frame = 24.8M) and end (Frame = 48.8M) of training.

\begin{figure}[htbp]
\vspace{-2em}
\centering
\subcaptionbox{$\relu$}{\includegraphics[width=\figwidthtwo]{figures/rl/grad_atari_dqn/AtariDQN_Amidar_ReLU.pdf}}
\subcaptionbox{$\tanh$}{\includegraphics[width=\figwidthtwo]{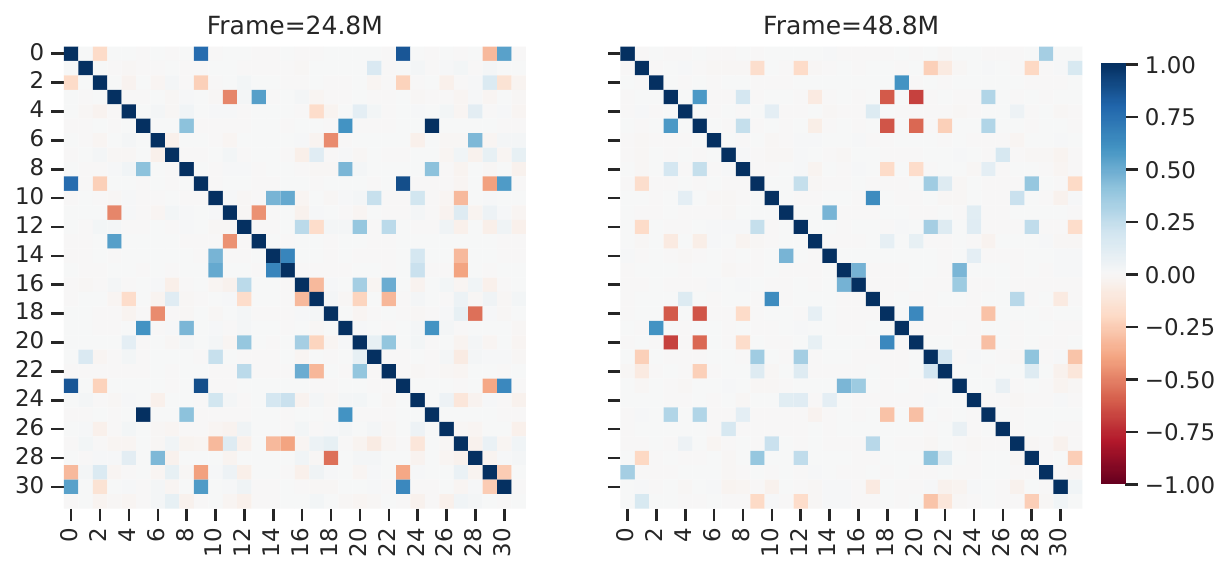}}
\subcaptionbox{$\maxout$}{\includegraphics[width=\figwidthtwo]{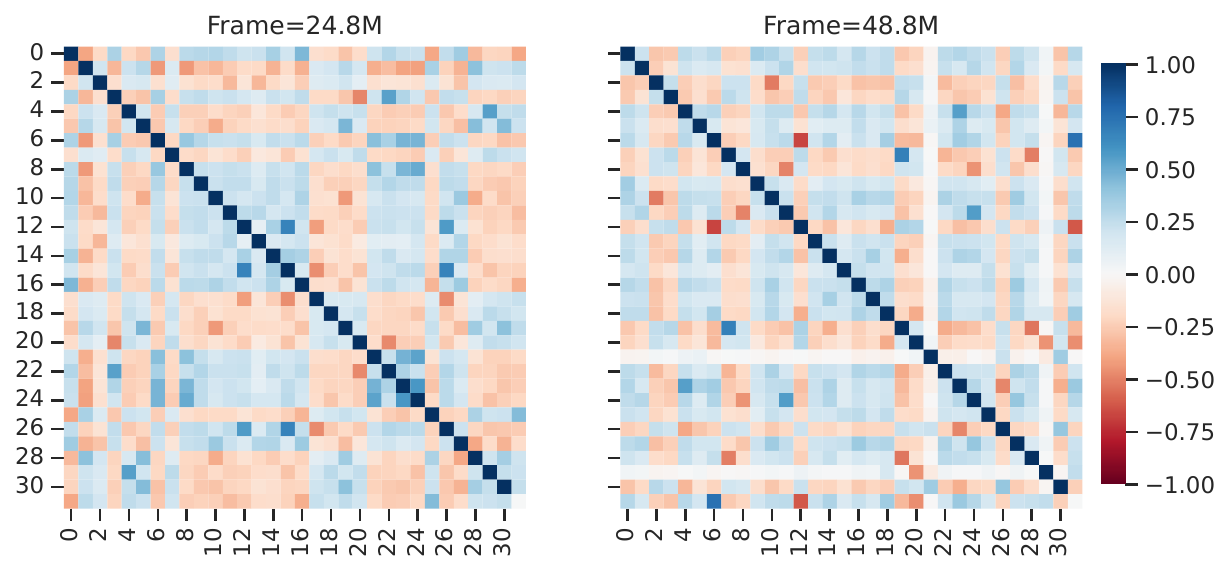}}
\subcaptionbox{$\lwta$}{\includegraphics[width=\figwidthtwo]{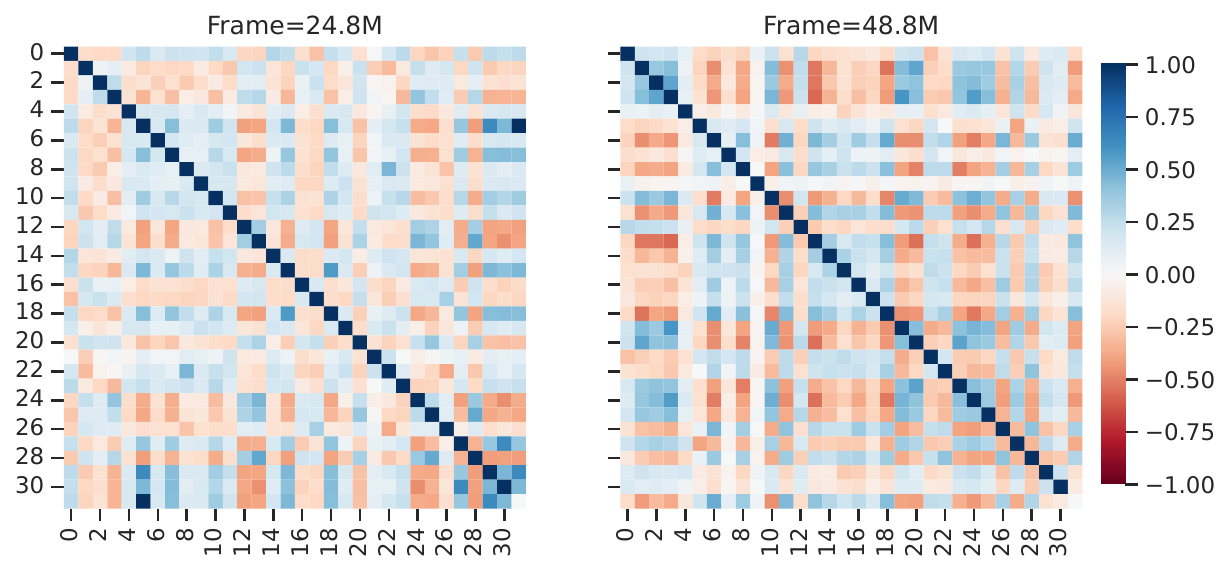}}
\subcaptionbox{$\fta$}{\includegraphics[width=\figwidthtwo]{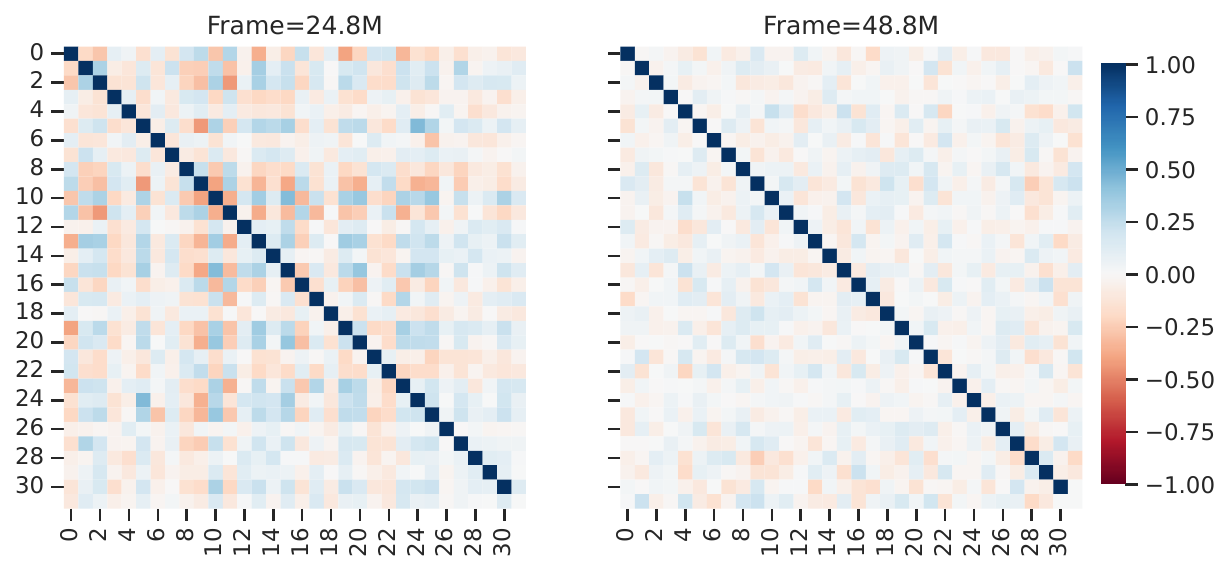}}
\subcaptionbox{$\elephant$}{\includegraphics[width=\figwidthtwo]{figures/rl/grad_atari_dqn/AtariDQN_Amidar_AdaptiveElementwiseElephant.pdf}}
\hfill
\caption{Heatmaps of gradient covariance matrices for training DQN in Amidar.}
\label{fig:atari_grad_dqn:Amidar}
\hfill \\
\subcaptionbox{$\relu$}{\includegraphics[width=\figwidthtwo]{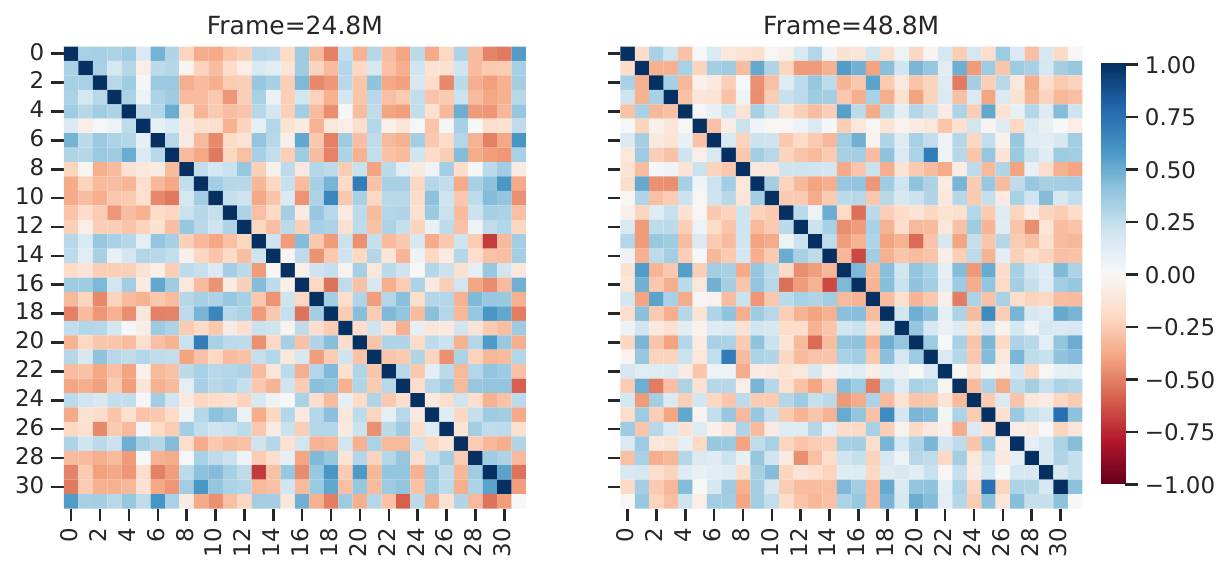}}
\subcaptionbox{$\tanh$}{\includegraphics[width=\figwidthtwo]{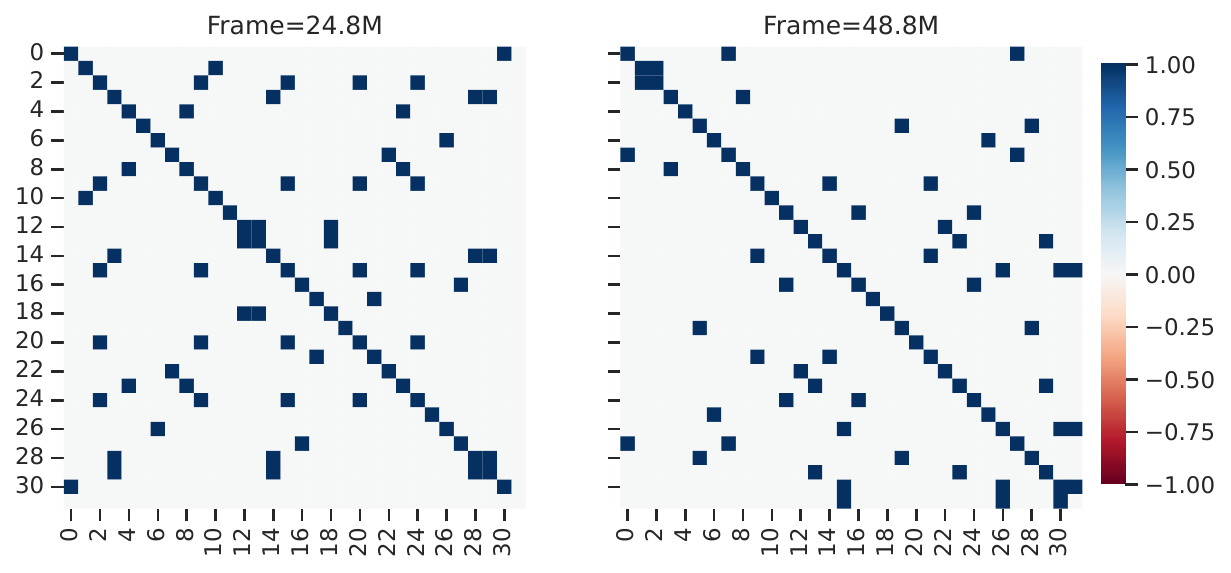}}
\subcaptionbox{$\maxout$}{\includegraphics[width=\figwidthtwo]{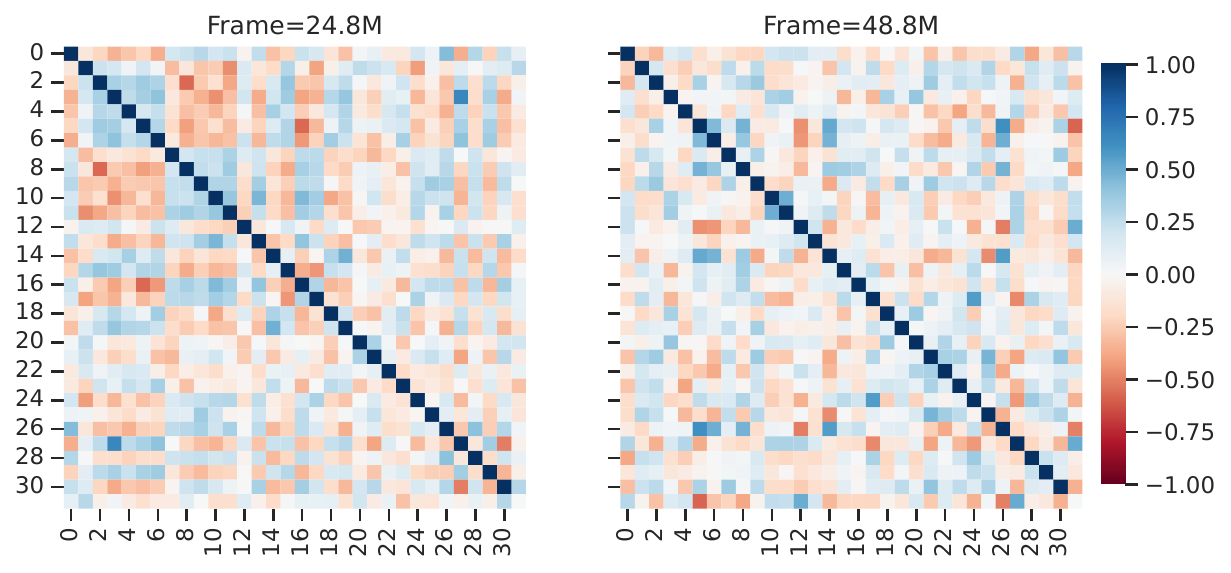}}
\subcaptionbox{$\lwta$}{\includegraphics[width=\figwidthtwo]{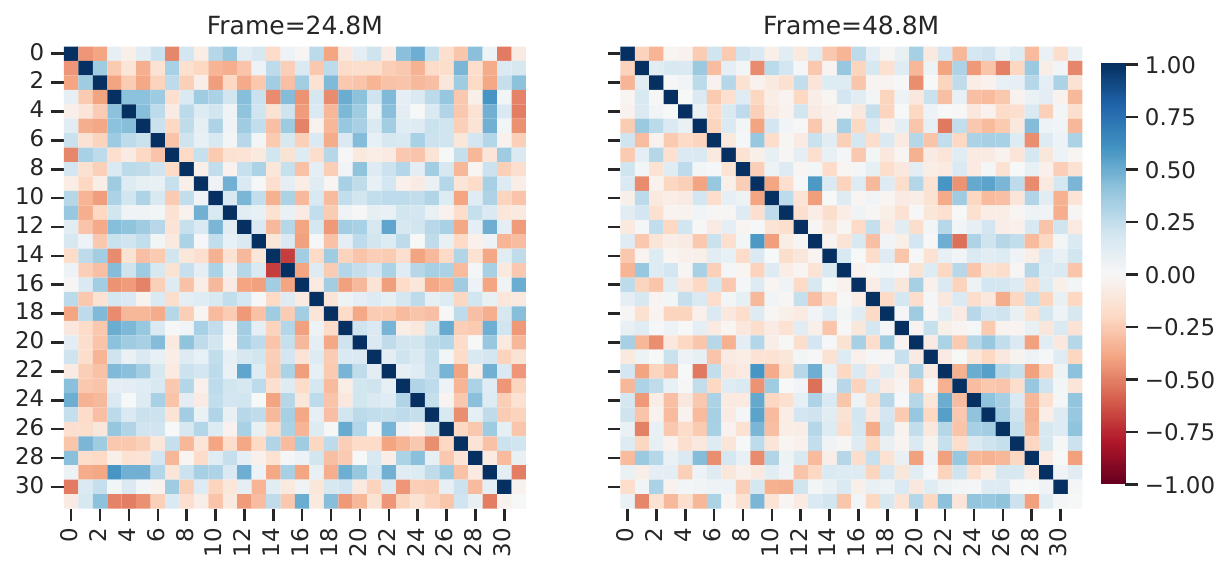}}
\subcaptionbox{$\fta$}{\includegraphics[width=\figwidthtwo]{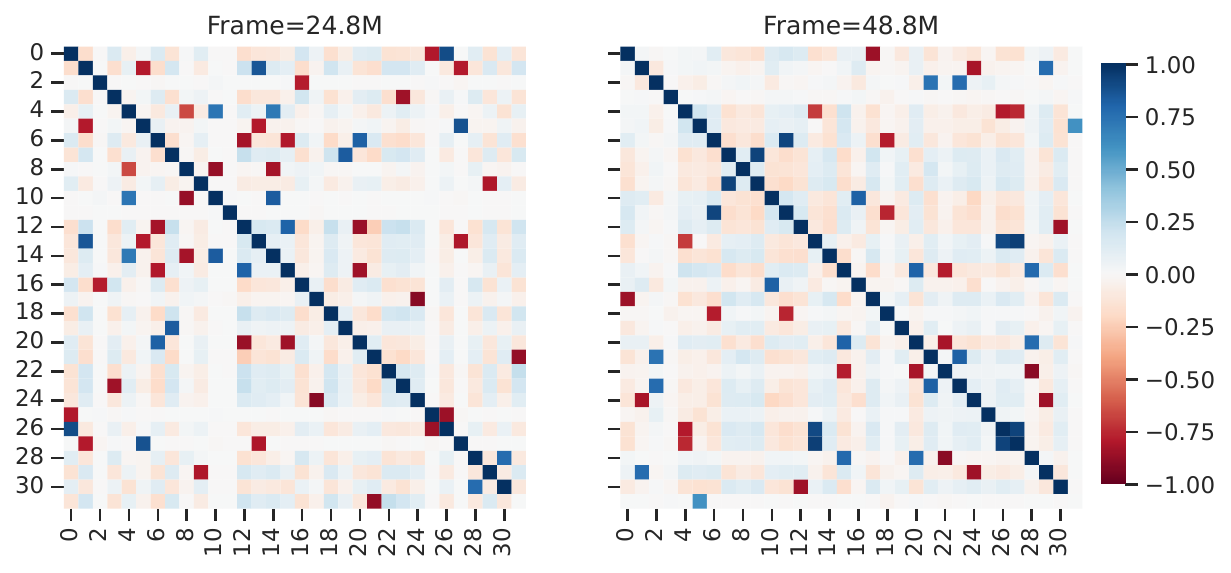}}
\subcaptionbox{$\elephant$}{\includegraphics[width=\figwidthtwo]{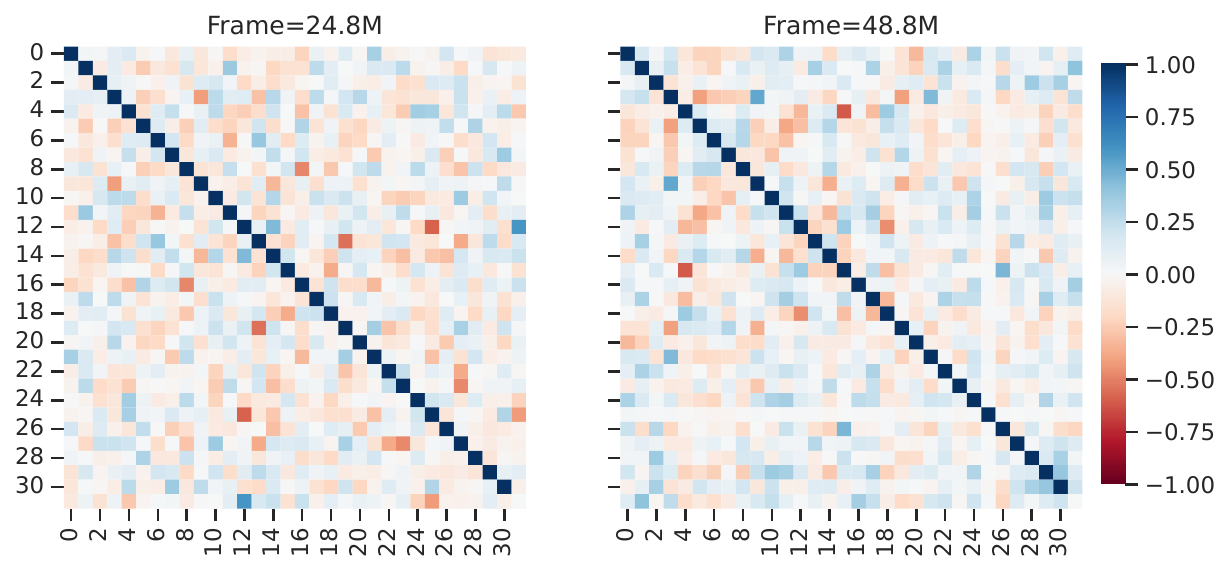}}
\hfill
\caption{Heatmaps of gradient covariance matrices for training DQN in BattleZone.}
\label{fig:atari_grad_dqn:BattleZone}
\end{figure}

\begin{figure}[htbp]
\vspace{-2em}
\centering
\subcaptionbox{$\relu$}{\includegraphics[width=\figwidthtwo]{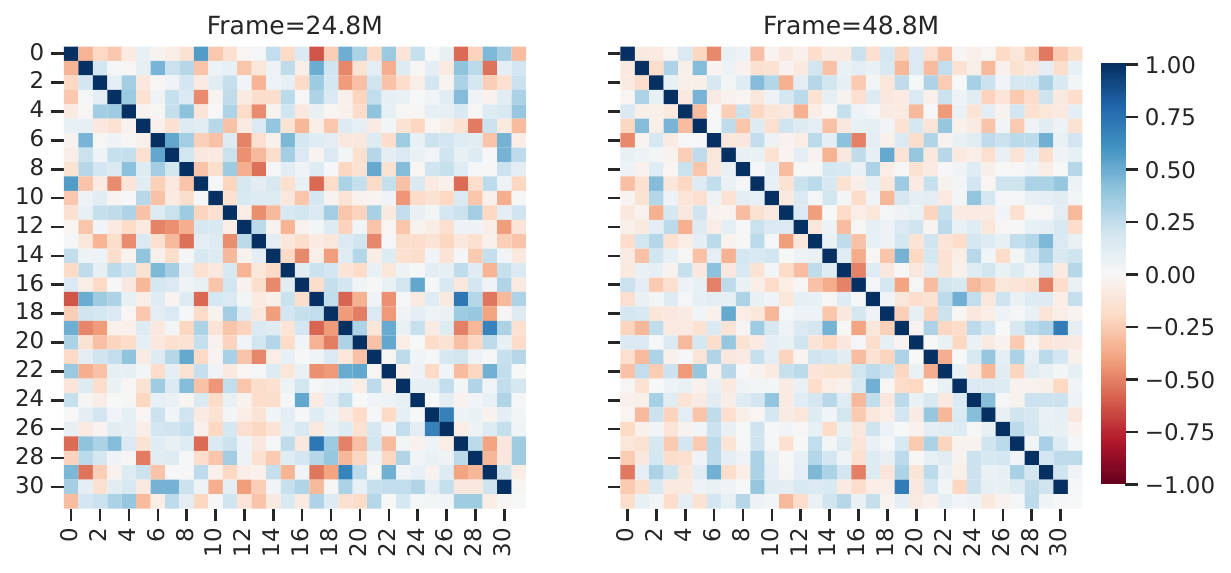}}
\subcaptionbox{$\tanh$}{\includegraphics[width=\figwidthtwo]{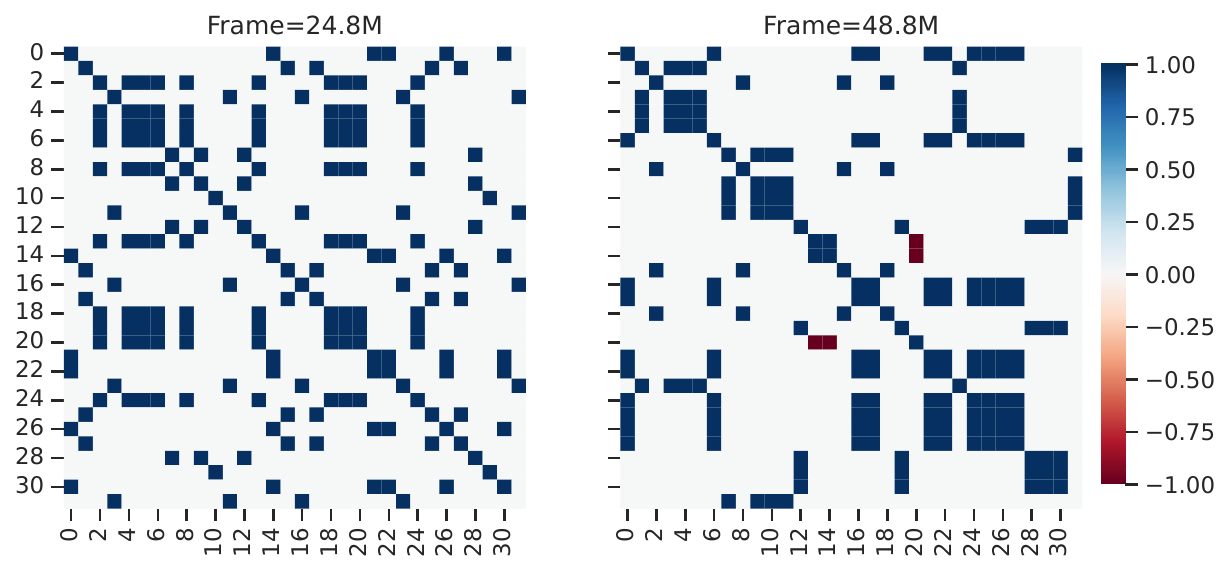}}
\subcaptionbox{$\maxout$}{\includegraphics[width=\figwidthtwo]{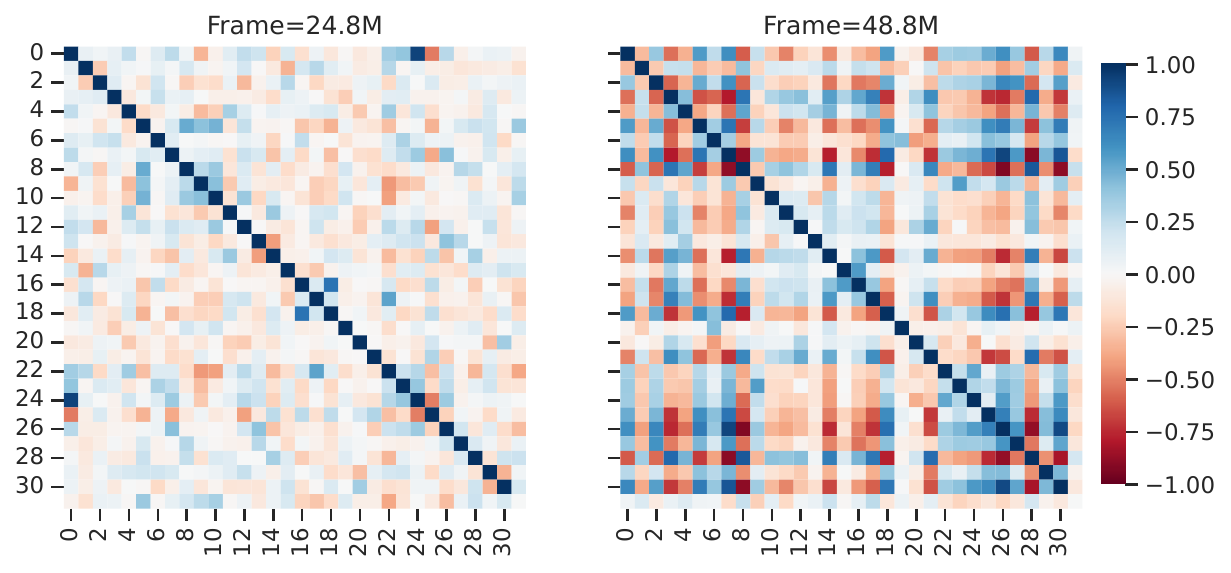}}
\subcaptionbox{$\lwta$}{\includegraphics[width=\figwidthtwo]{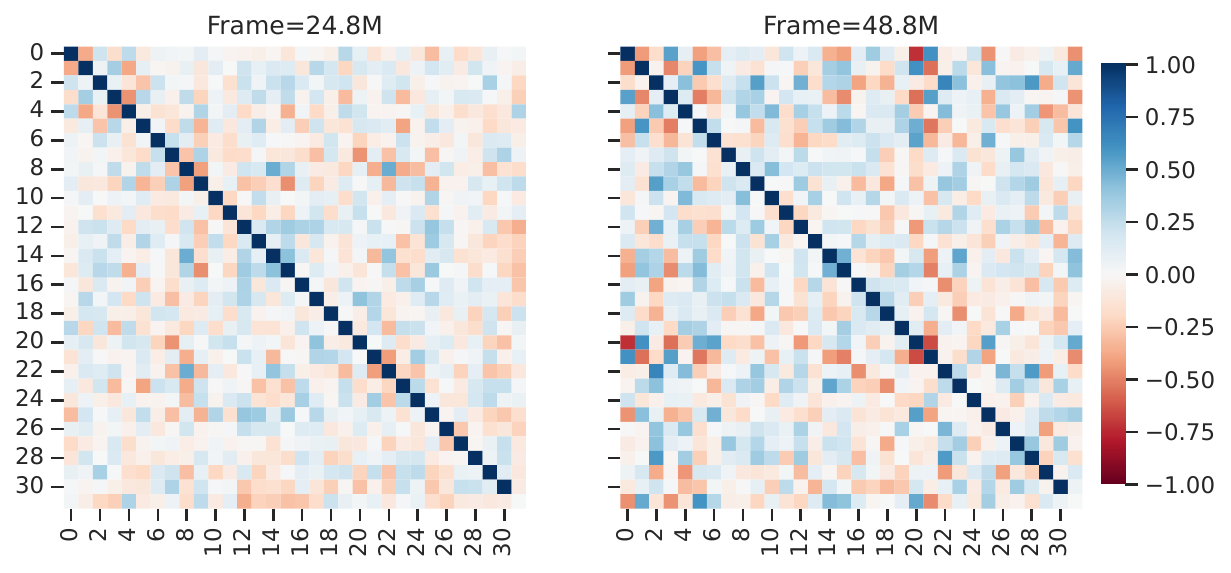}}
\subcaptionbox{$\fta$}{\includegraphics[width=\figwidthtwo]{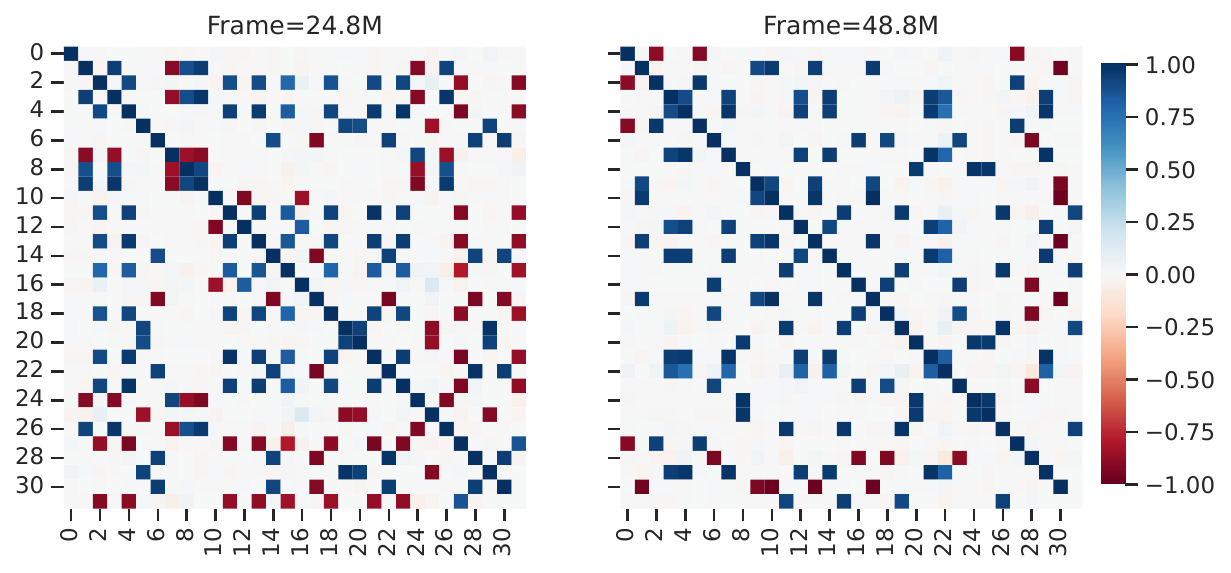}}
\subcaptionbox{$\elephant$}{\includegraphics[width=\figwidthtwo]{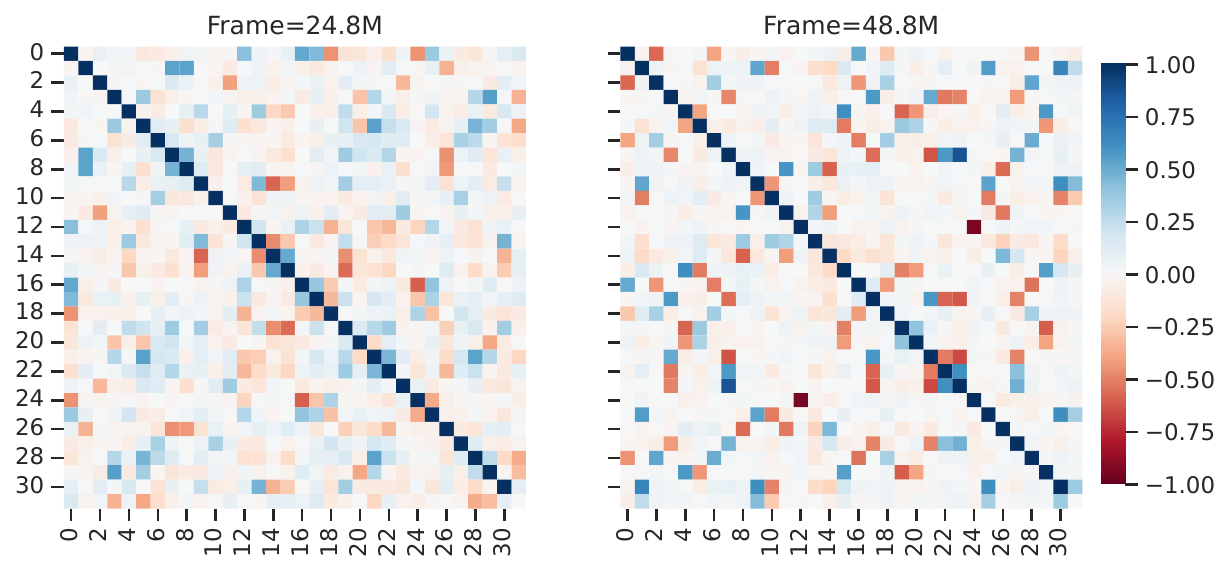}}
\hfill
\caption{Heatmaps of gradient covariance matrices for training DQN in Bowling.}
\label{fig:atari_grad_dqn:Bowling}
\hfill \\
\subcaptionbox{$\relu$}{\includegraphics[width=\figwidthtwo]{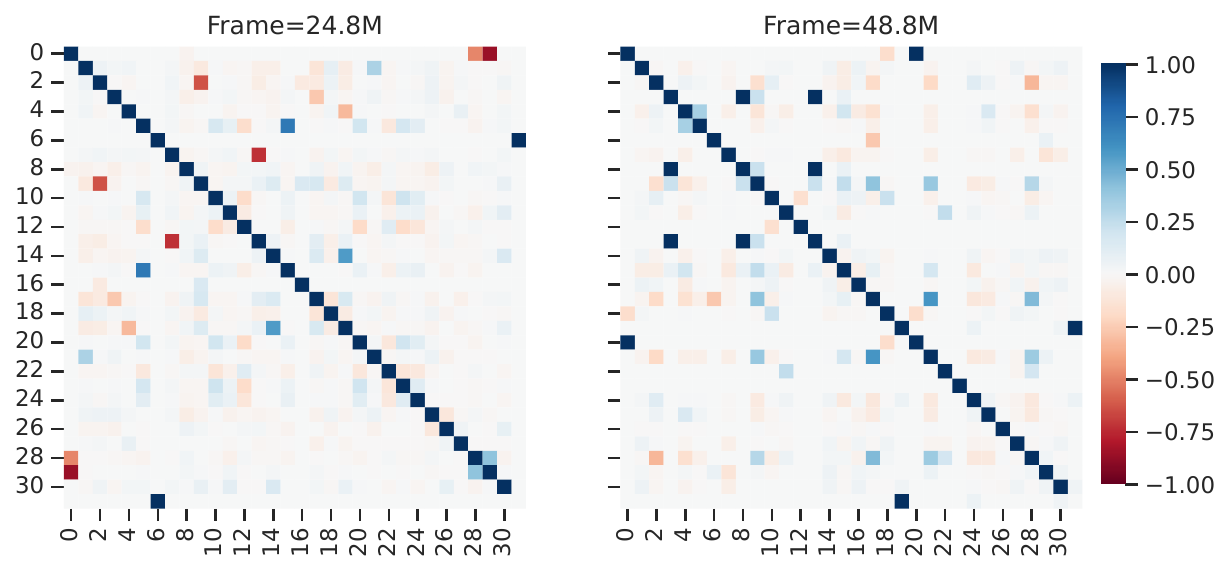}}
\subcaptionbox{$\tanh$}{\includegraphics[width=\figwidthtwo]{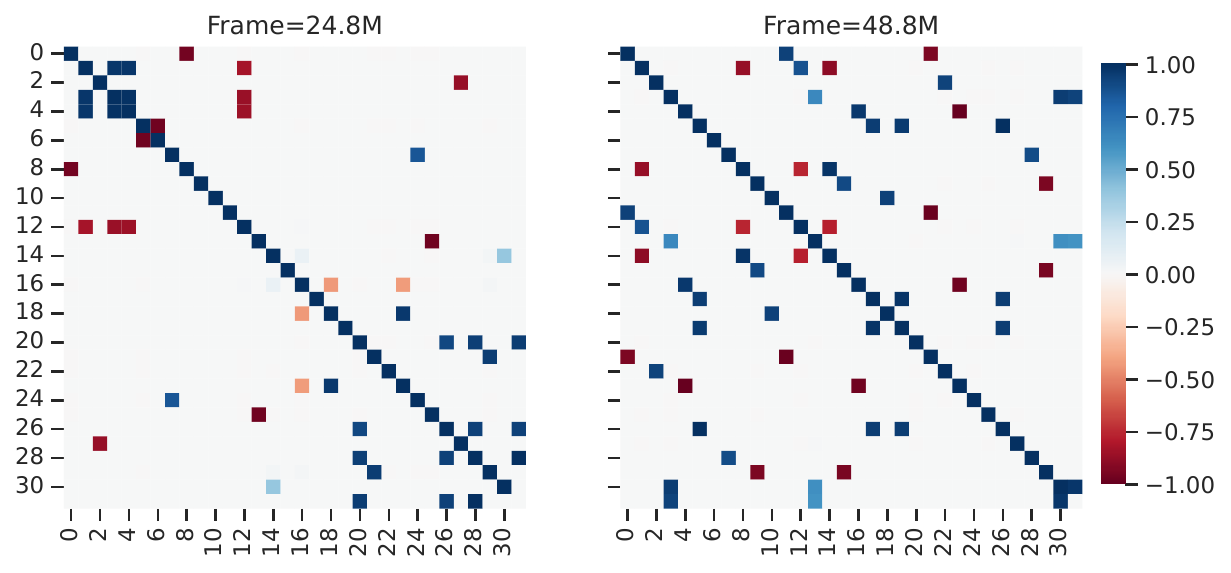}}
\subcaptionbox{$\maxout$}{\includegraphics[width=\figwidthtwo]{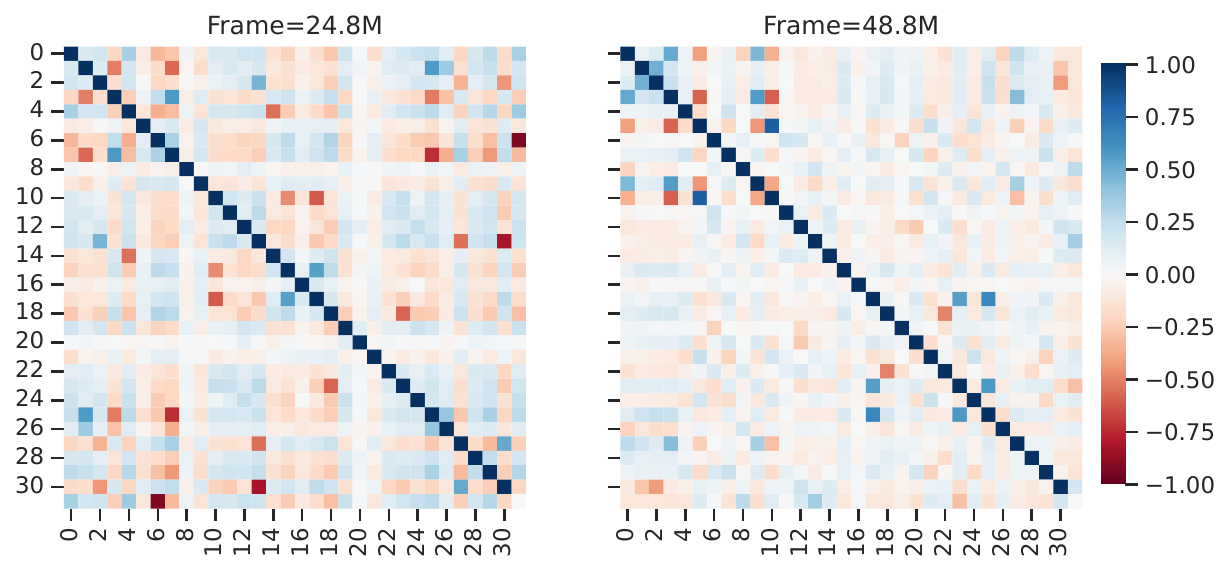}}
\subcaptionbox{$\lwta$}{\includegraphics[width=\figwidthtwo]{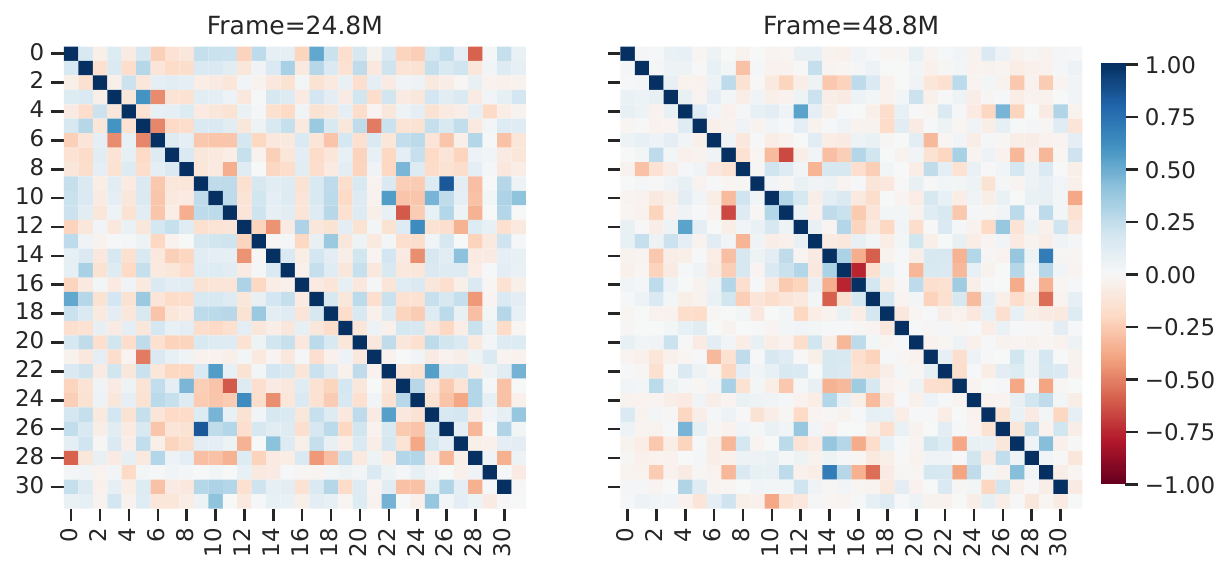}}
\subcaptionbox{$\fta$}{\includegraphics[width=\figwidthtwo]{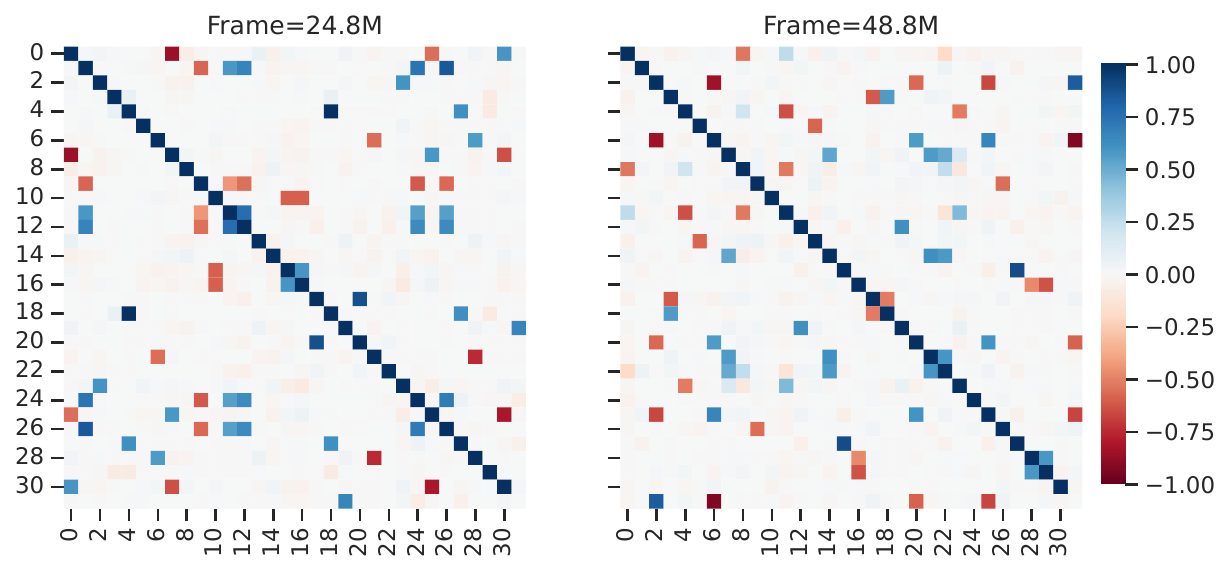}}
\subcaptionbox{$\elephant$}{\includegraphics[width=\figwidthtwo]{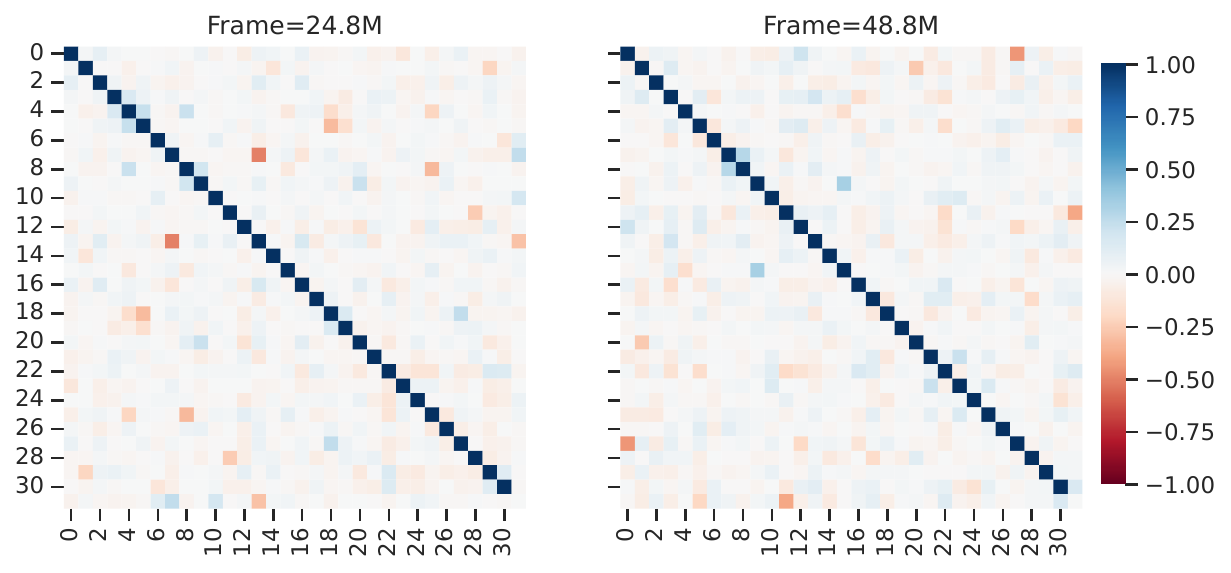}}
\hfill
\caption{Heatmaps of gradient covariance matrices for training DQN in DoubleDunk.}
\label{fig:atari_grad_dqn:DoubleDunk}
\end{figure}

\begin{figure}[htbp]
\vspace{-2em}
\centering
\subcaptionbox{$\relu$}{\includegraphics[width=\figwidthtwo]{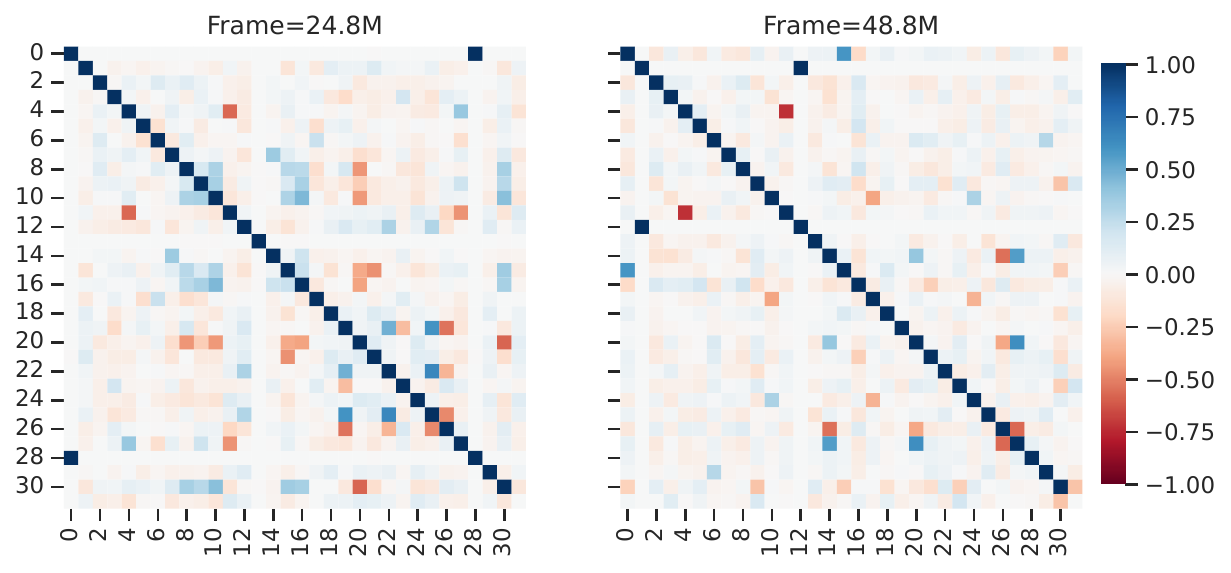}}
\subcaptionbox{$\tanh$}{\includegraphics[width=\figwidthtwo]{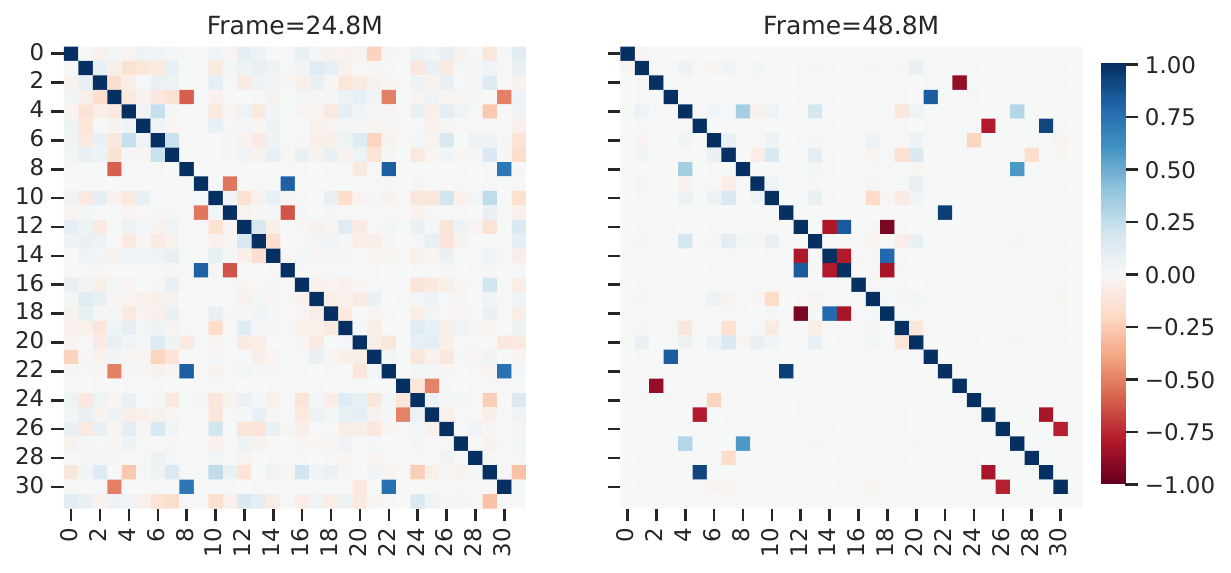}}
\subcaptionbox{$\maxout$}{\includegraphics[width=\figwidthtwo]{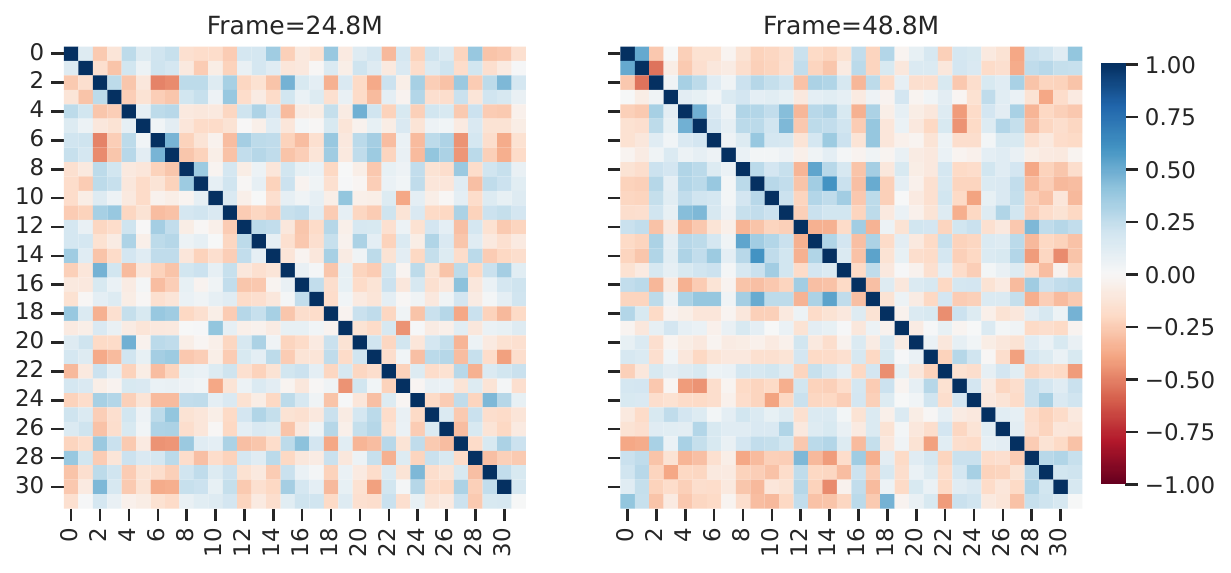}}
\subcaptionbox{$\lwta$}{\includegraphics[width=\figwidthtwo]{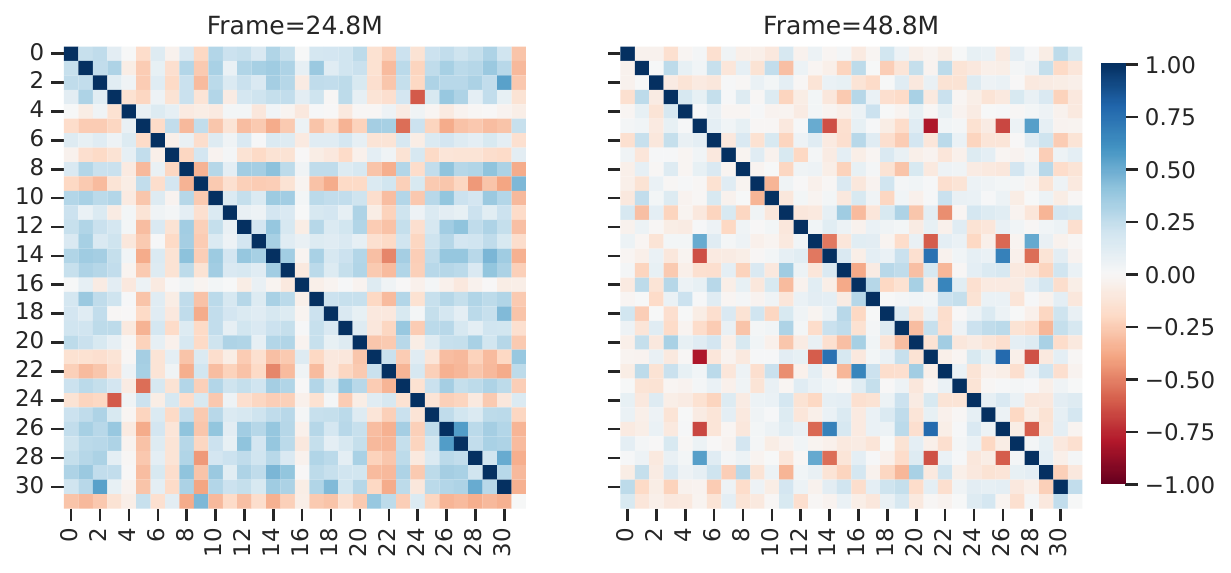}}
\subcaptionbox{$\fta$}{\includegraphics[width=\figwidthtwo]{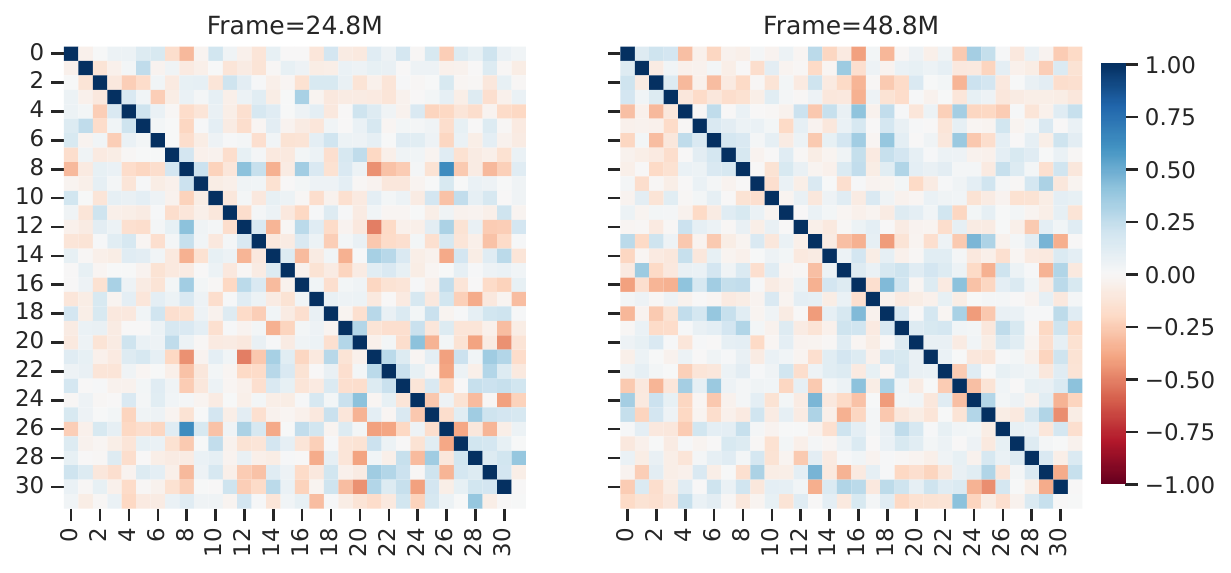}}
\subcaptionbox{$\elephant$}{\includegraphics[width=\figwidthtwo]{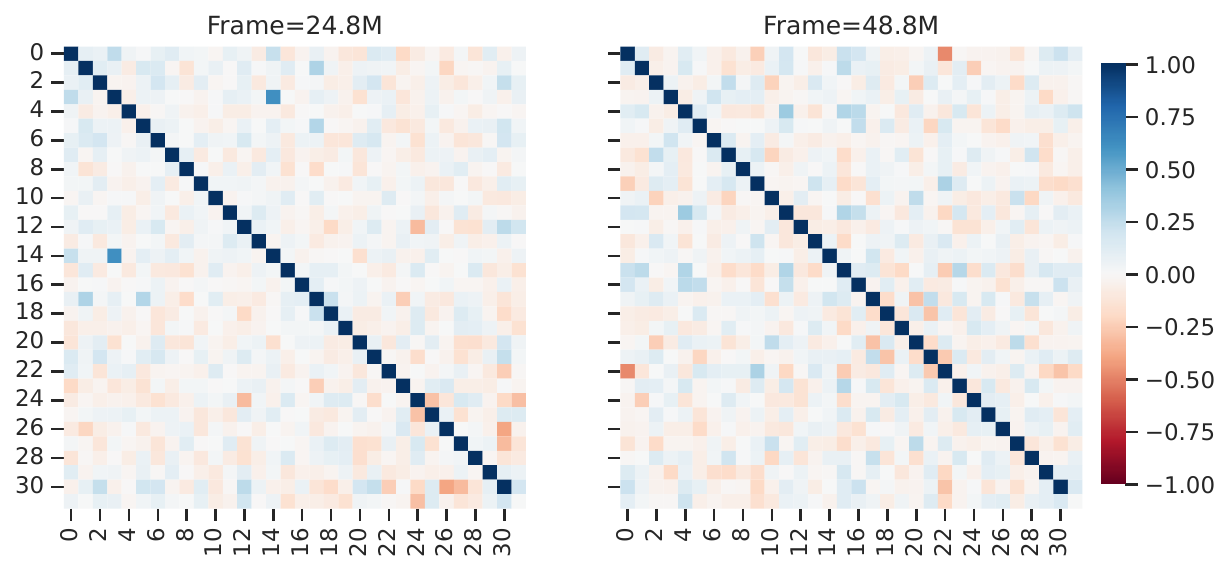}}
\hfill
\caption{Heatmaps of gradient covariance matrices for training DQN with in Frostbite.}
\label{fig:atari_grad_dqn:Frostbite}
\hfill \\
\subcaptionbox{$\relu$}{\includegraphics[width=\figwidthtwo]{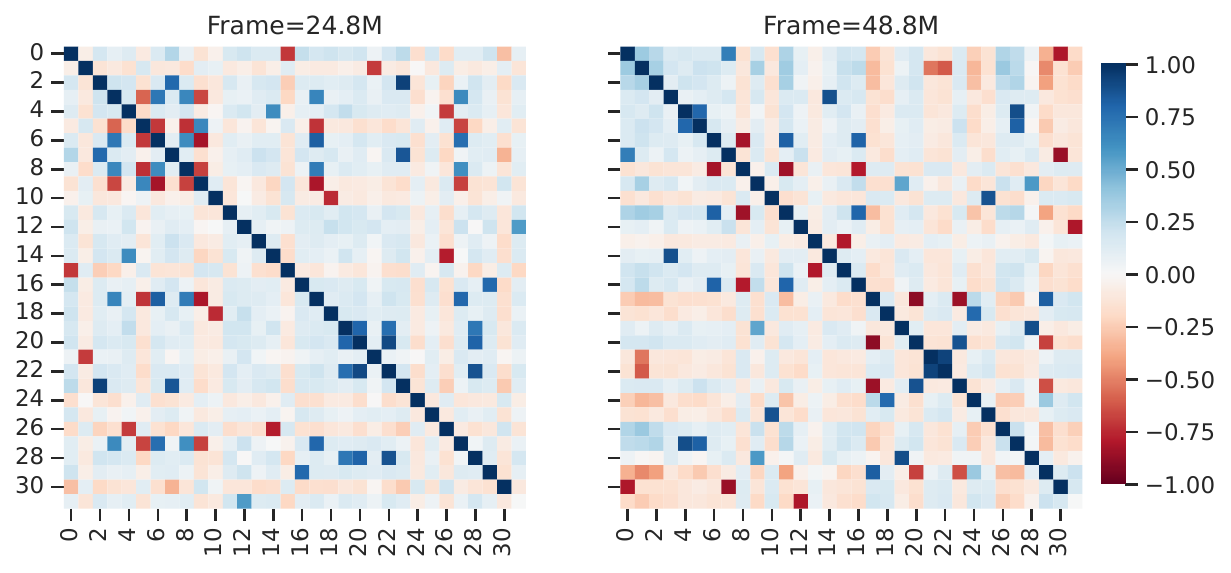}}
\subcaptionbox{$\tanh$}{\includegraphics[width=\figwidthtwo]{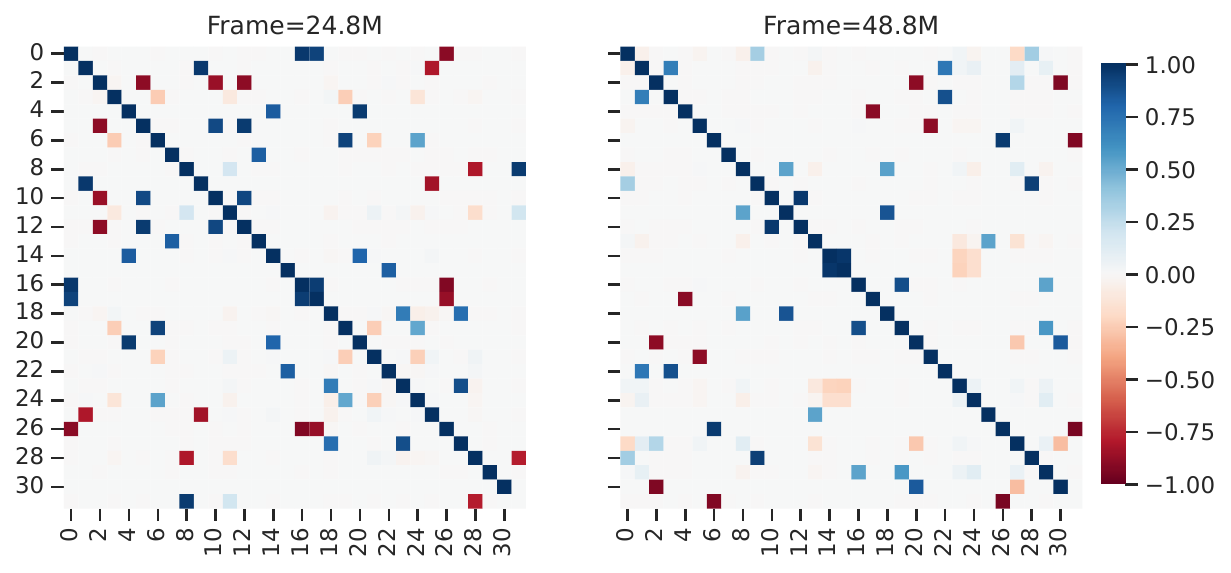}}
\subcaptionbox{$\maxout$}{\includegraphics[width=\figwidthtwo]{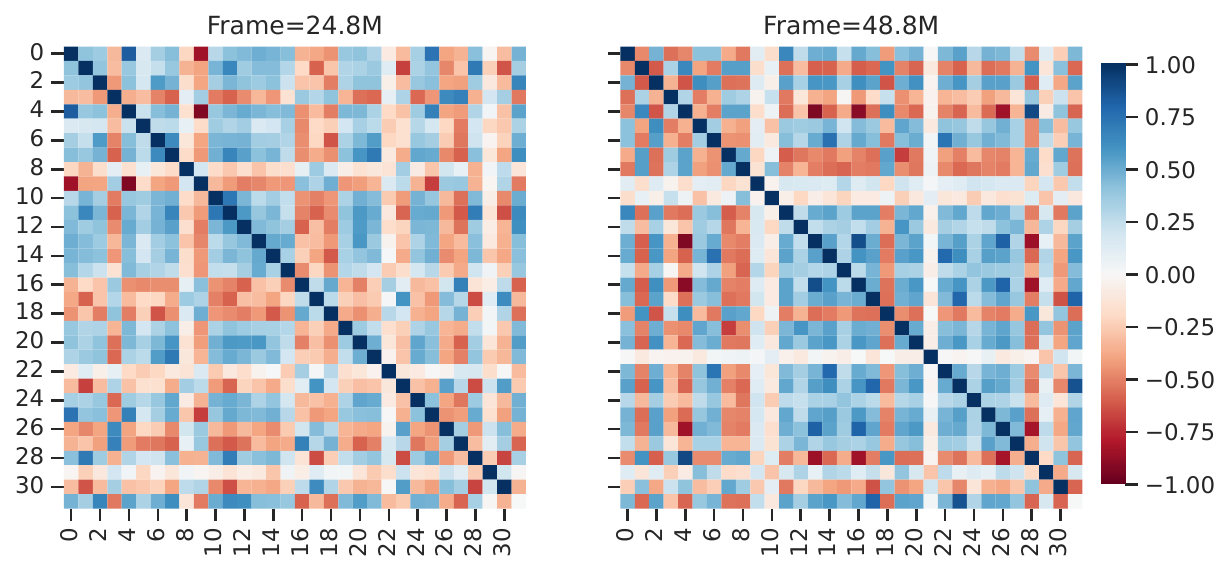}}
\subcaptionbox{$\lwta$}{\includegraphics[width=\figwidthtwo]{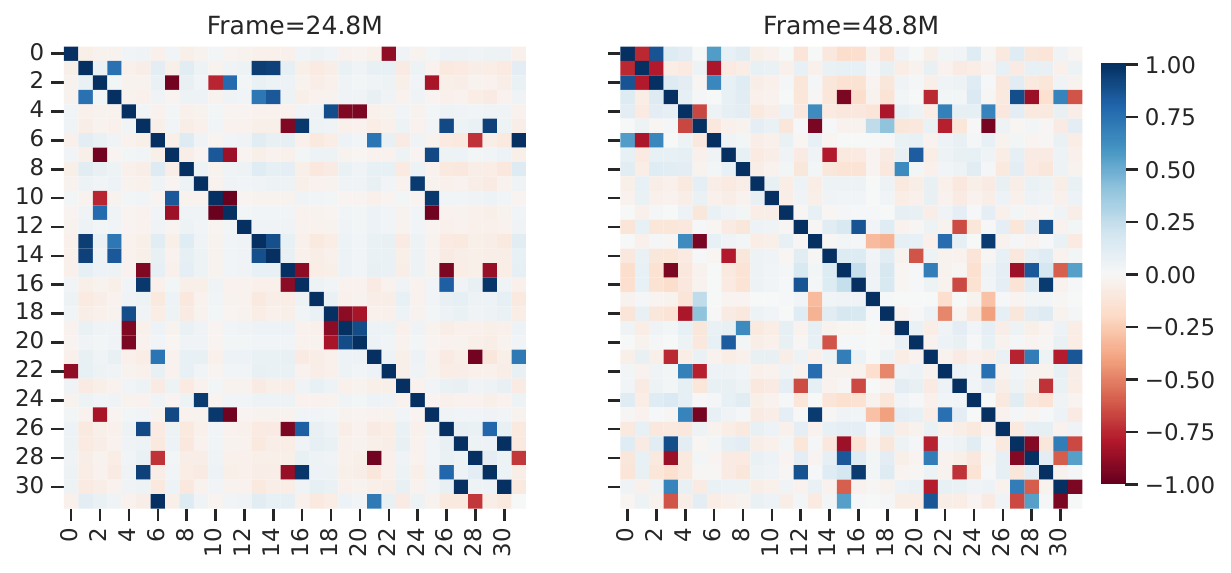}}
\subcaptionbox{$\fta$}{\includegraphics[width=\figwidthtwo]{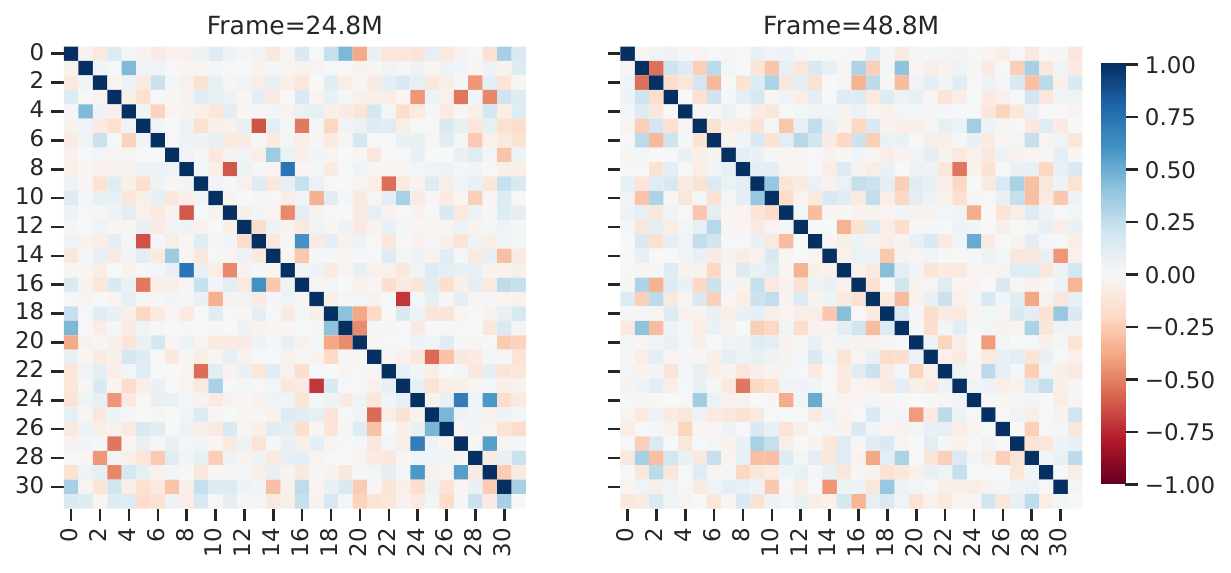}}
\subcaptionbox{$\elephant$}{\includegraphics[width=\figwidthtwo]{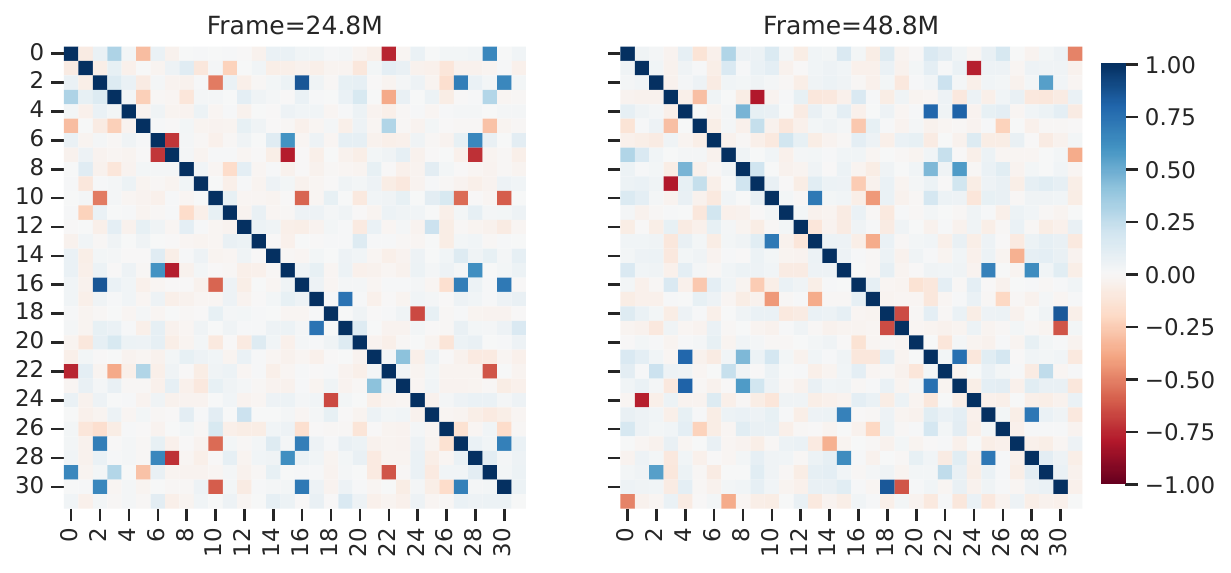}}
\hfill
\caption{Heatmaps of gradient covariance matrices for training DQN in Kung-Fu Master.}
\label{fig:atari_grad_dqn:KungFuMaster}
\end{figure}

\begin{figure}[htbp]
\vspace{-2em}
\centering
\subcaptionbox{$\relu$}{\includegraphics[width=\figwidthtwo]{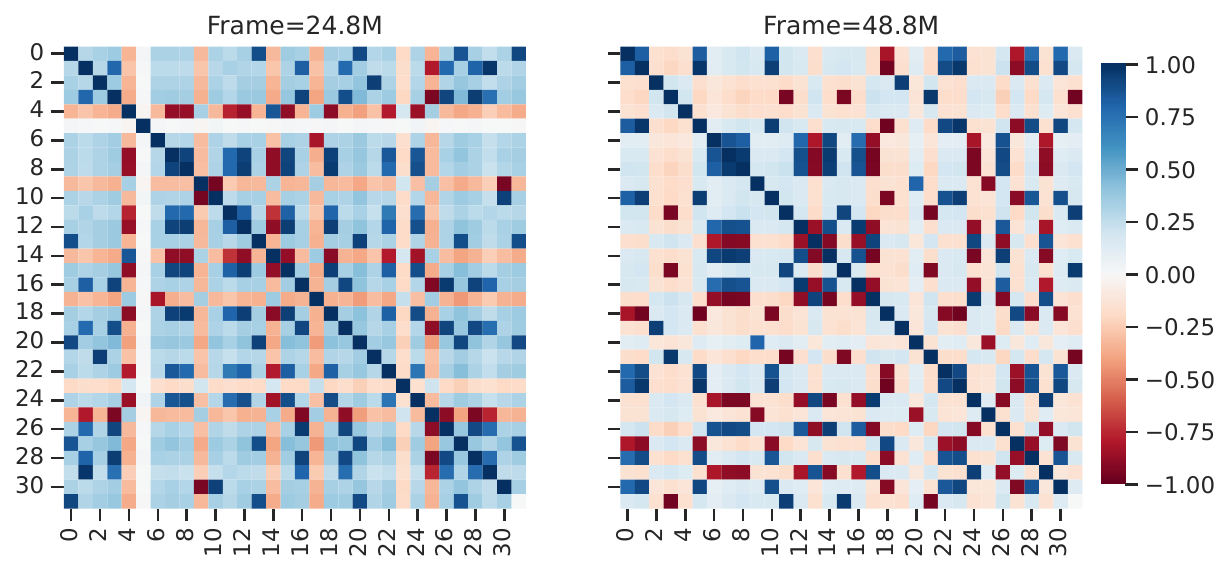}}
\subcaptionbox{$\tanh$}{\includegraphics[width=\figwidthtwo]{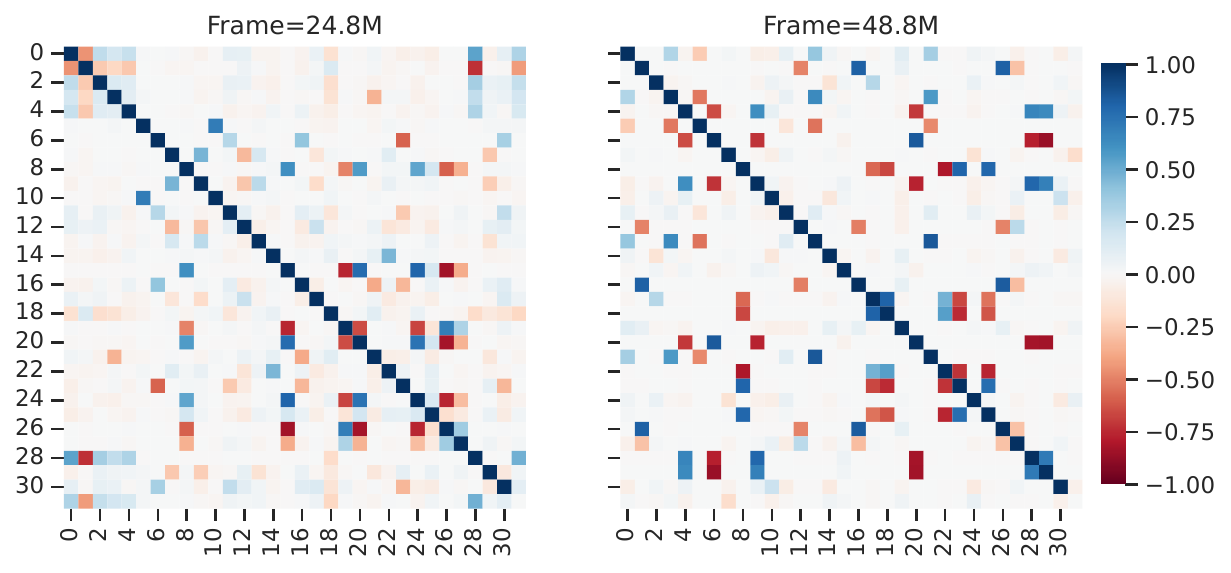}}
\subcaptionbox{$\maxout$}{\includegraphics[width=\figwidthtwo]{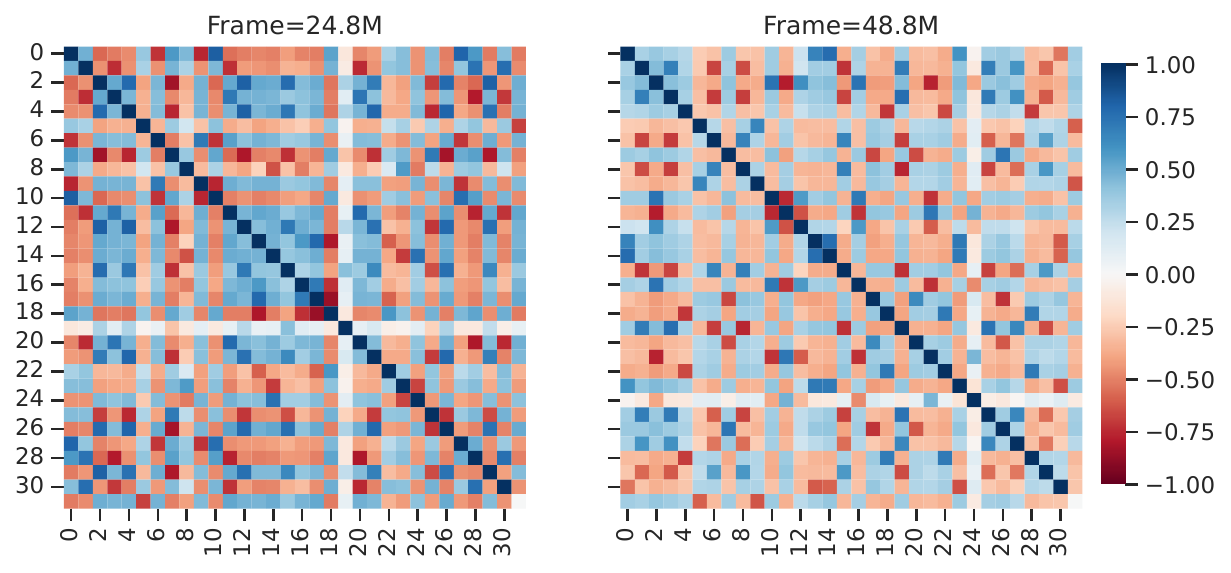}}
\subcaptionbox{$\lwta$}{\includegraphics[width=\figwidthtwo]{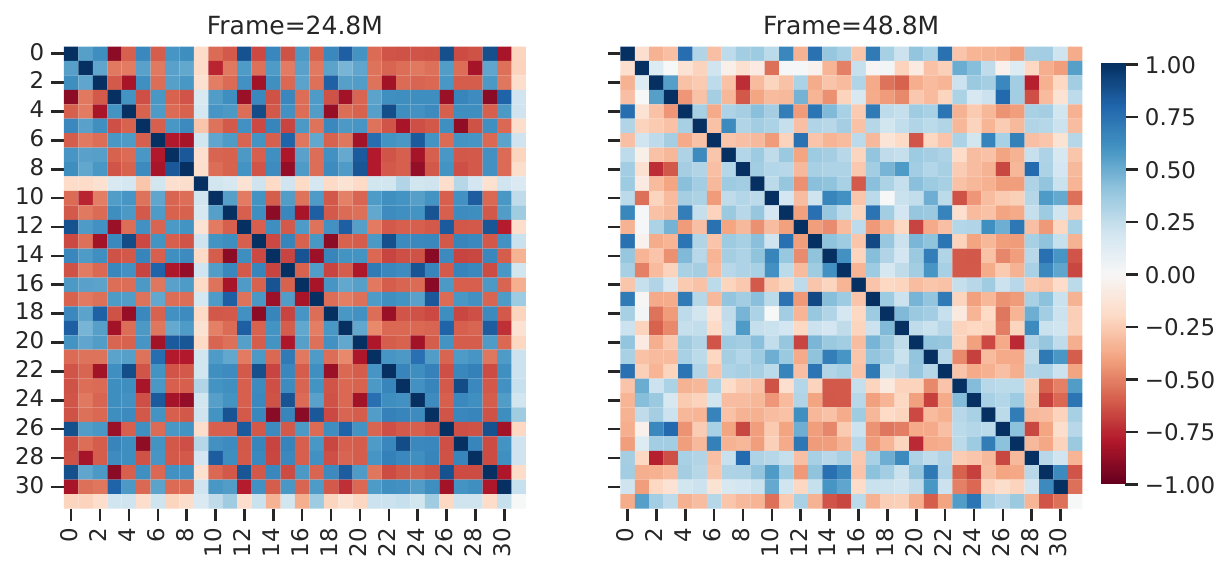}}
\subcaptionbox{$\fta$}{\includegraphics[width=\figwidthtwo]{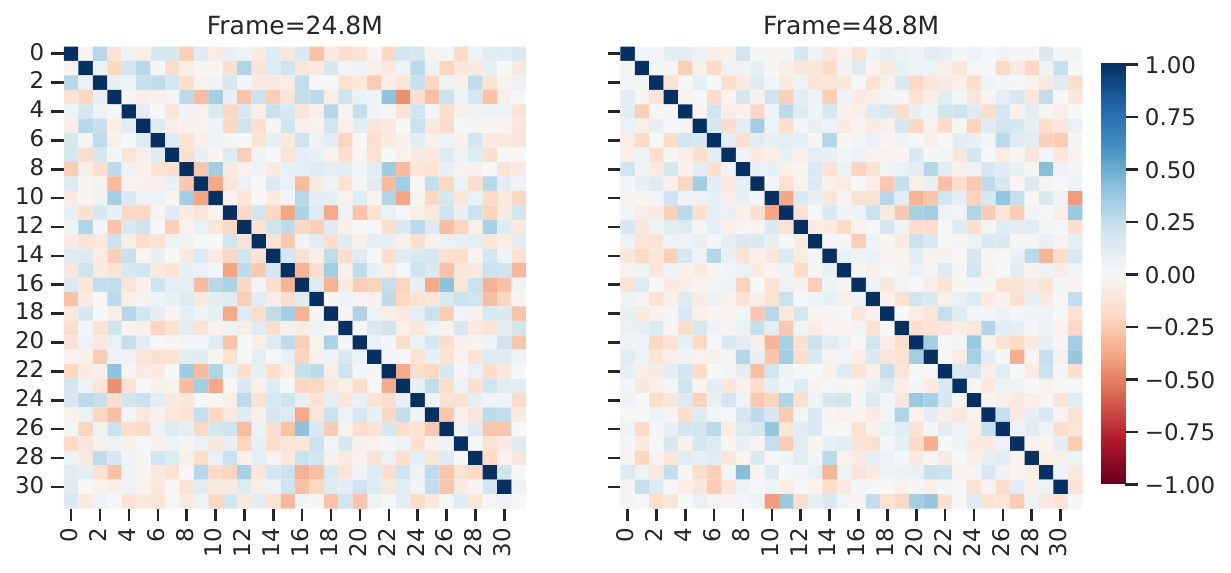}}
\subcaptionbox{$\elephant$}{\includegraphics[width=\figwidthtwo]{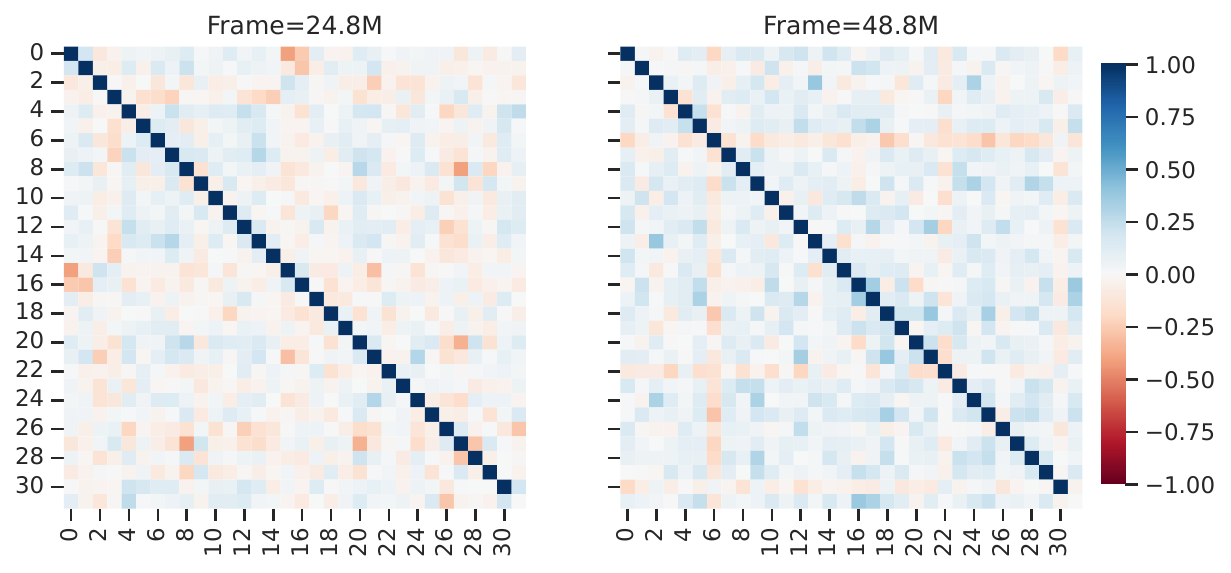}}
\hfill
\caption{Heatmaps of gradient covariance matrices for training DQN in Name This Game.}
\label{fig:atari_grad_dqn:NameThisGame}
\hfill \\
\subcaptionbox{$\relu$}{\includegraphics[width=\figwidthtwo]{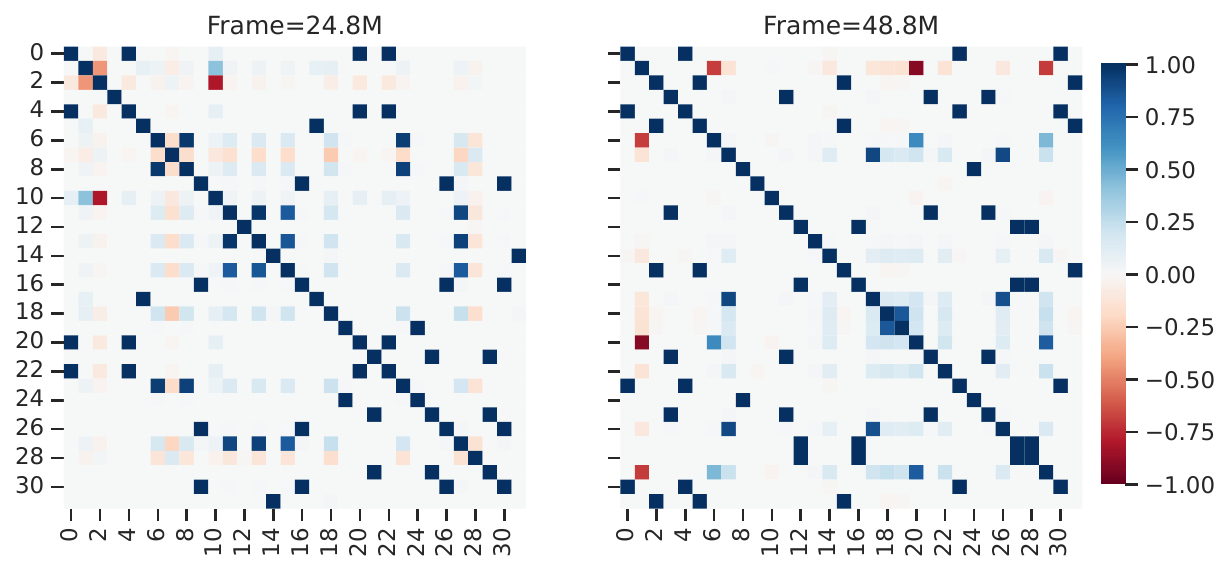}}
\subcaptionbox{$\tanh$}{\includegraphics[width=\figwidthtwo]{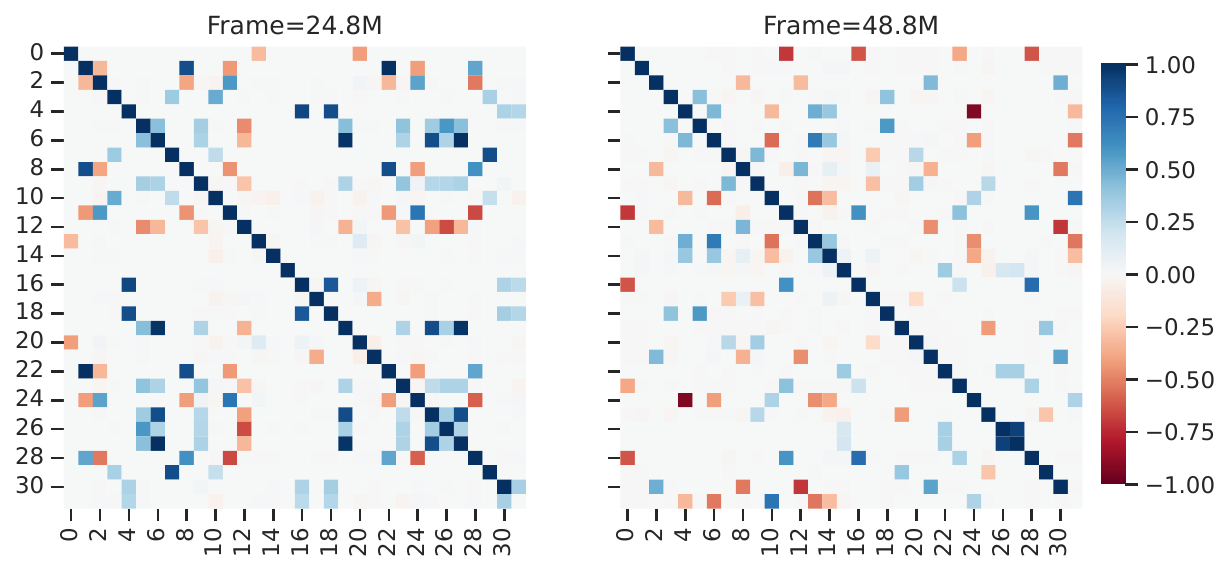}}
\subcaptionbox{$\maxout$}{\includegraphics[width=\figwidthtwo]{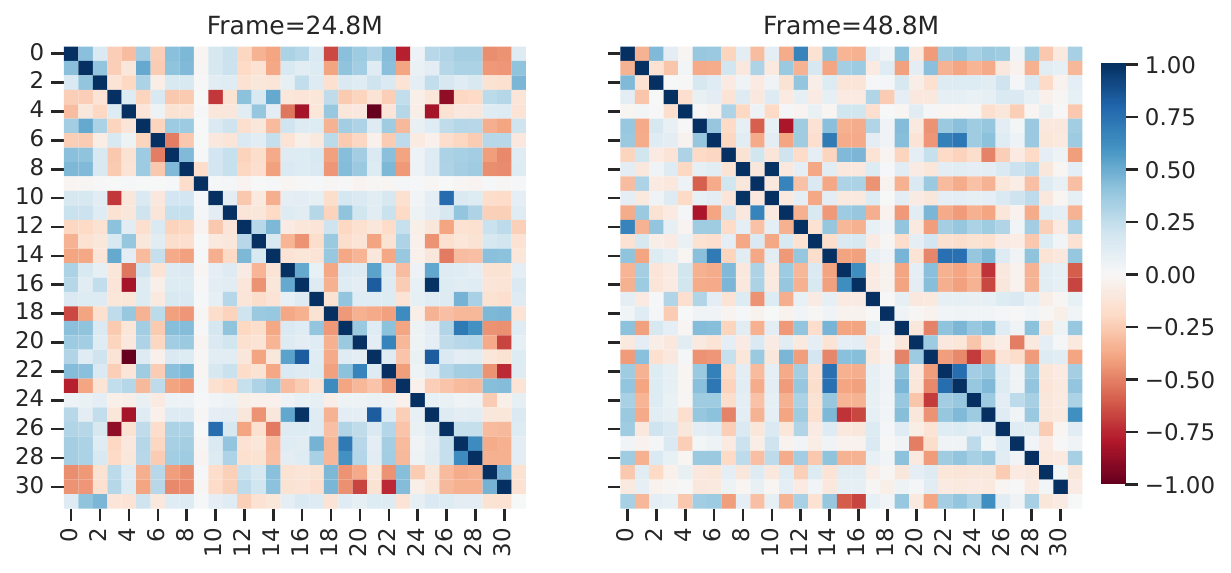}}
\subcaptionbox{$\lwta$}{\includegraphics[width=\figwidthtwo]{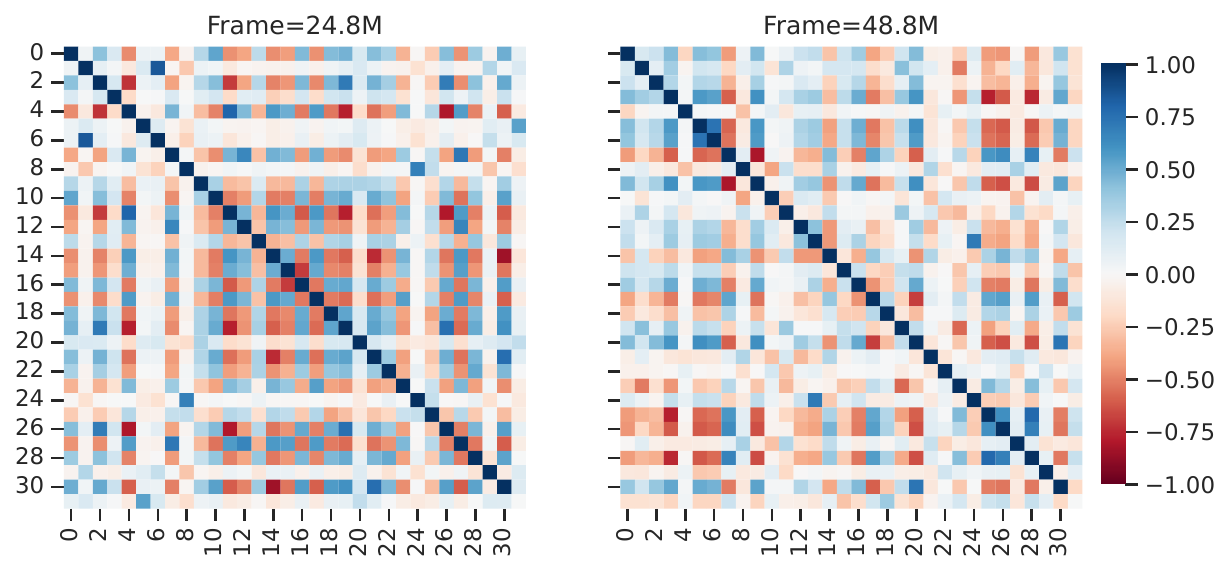}}
\subcaptionbox{$\fta$}{\includegraphics[width=\figwidthtwo]{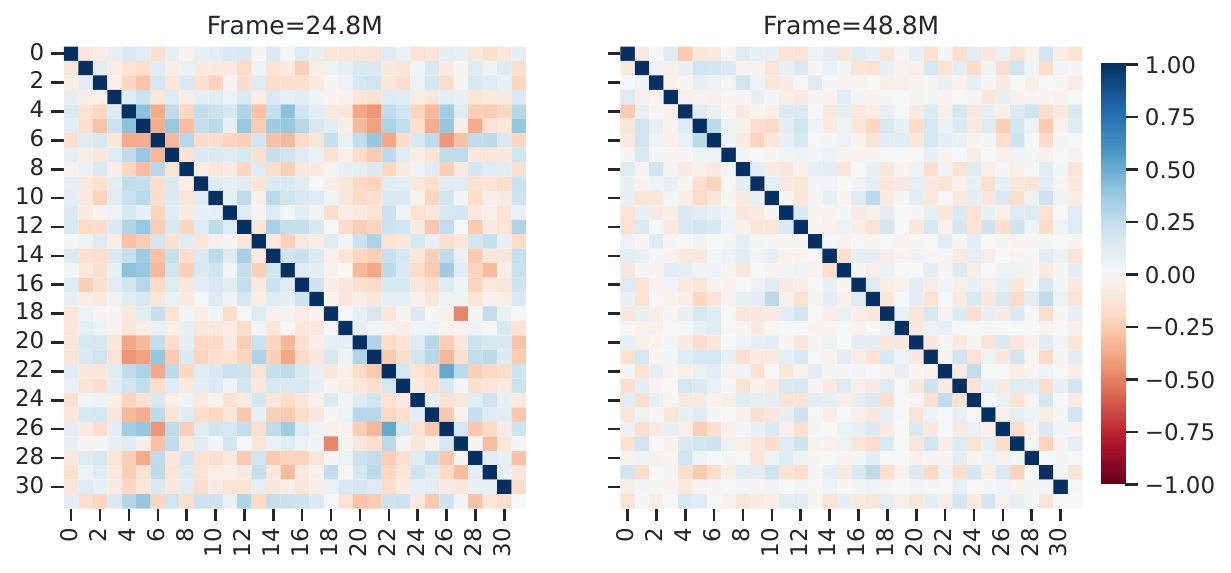}}
\subcaptionbox{$\elephant$}{\includegraphics[width=\figwidthtwo]{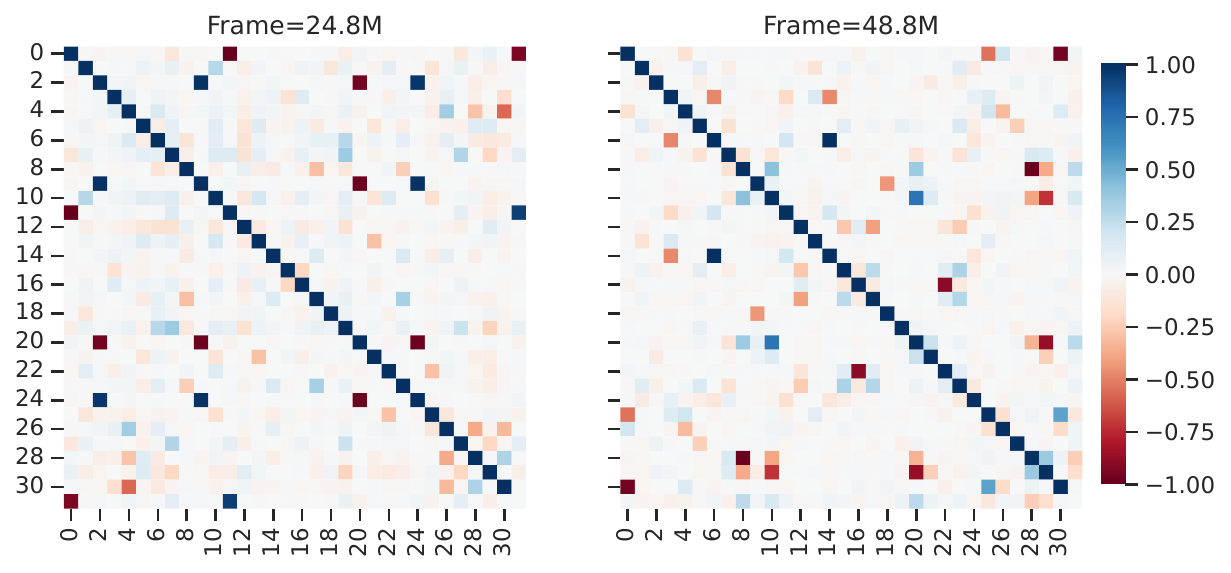}}
\hfill
\caption{Heatmaps of gradient covariance matrices for training DQN in Phoenix.}
\label{fig:atari_grad_dqn:Phoenix}
\end{figure}

\begin{figure}[htbp]
\vspace{-2em}
\centering
\subcaptionbox{$\relu$}{\includegraphics[width=\figwidthtwo]{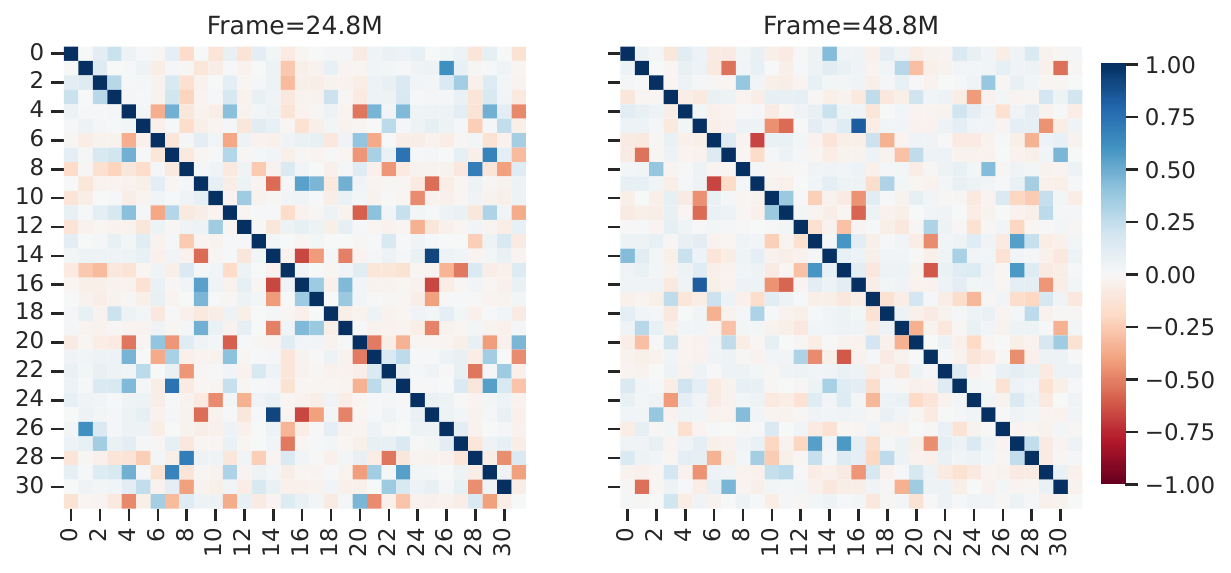}}
\subcaptionbox{$\tanh$}{\includegraphics[width=\figwidthtwo]{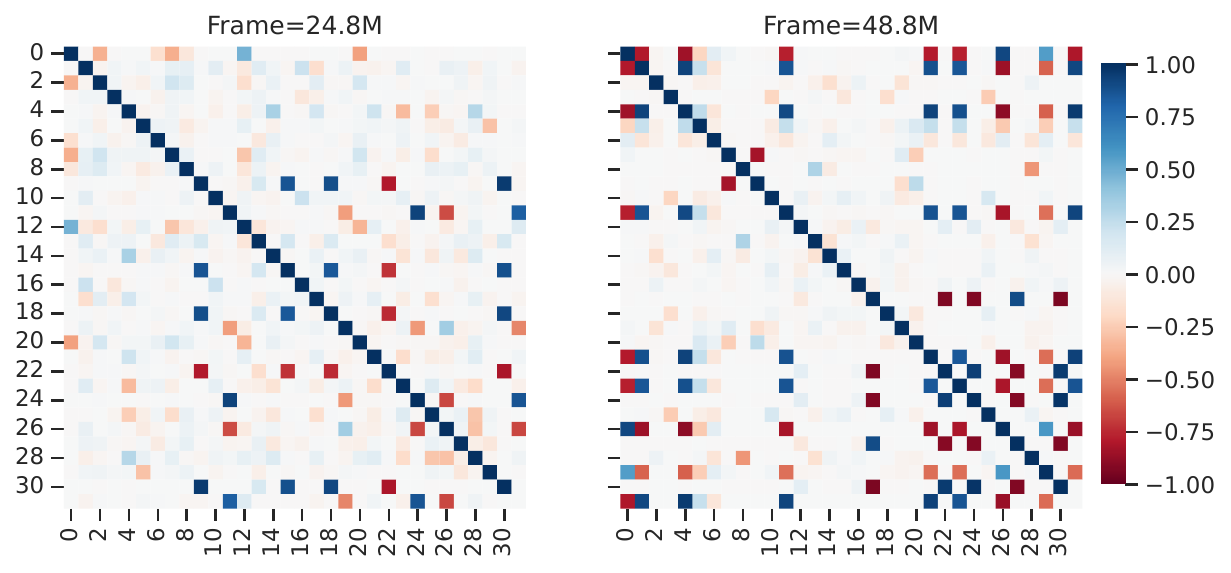}}
\subcaptionbox{$\maxout$}{\includegraphics[width=\figwidthtwo]{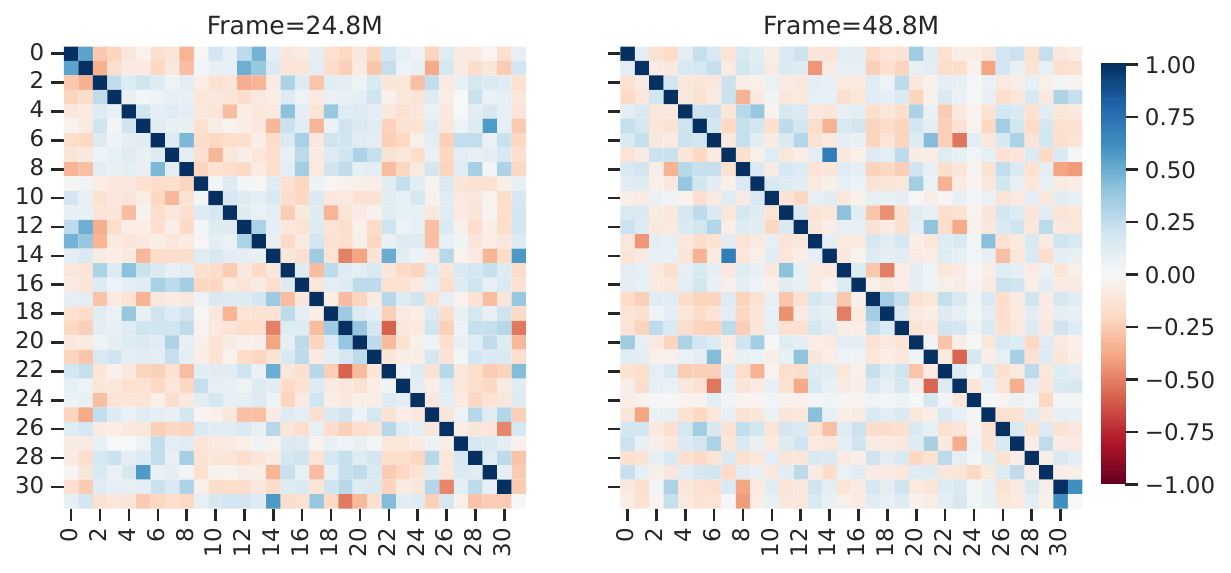}}
\subcaptionbox{$\lwta$}{\includegraphics[width=\figwidthtwo]{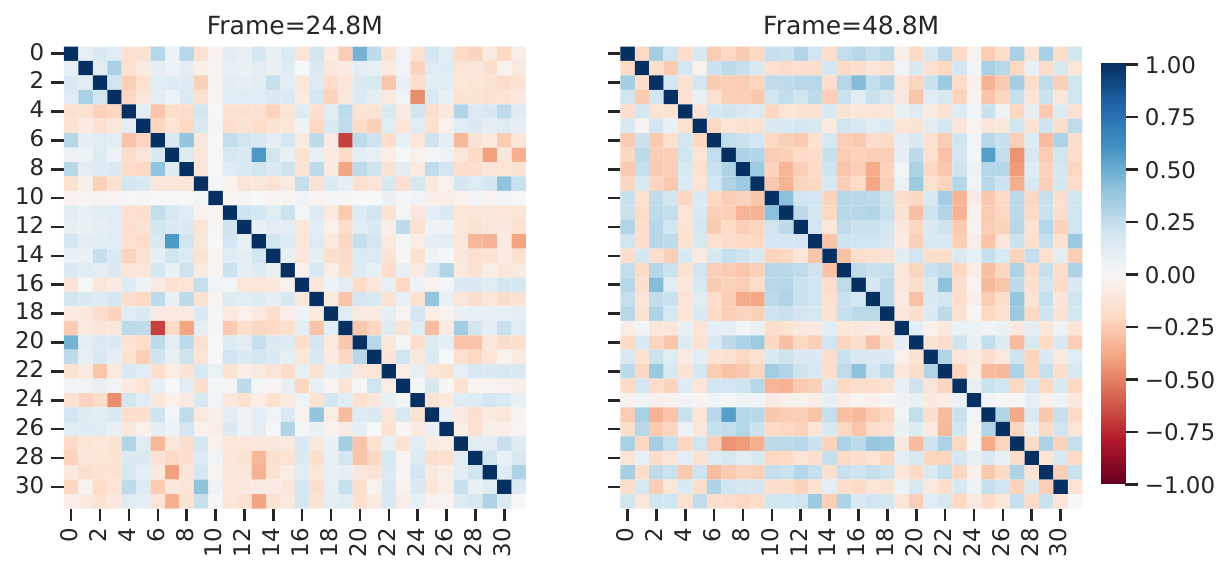}}
\subcaptionbox{$\fta$}{\includegraphics[width=\figwidthtwo]{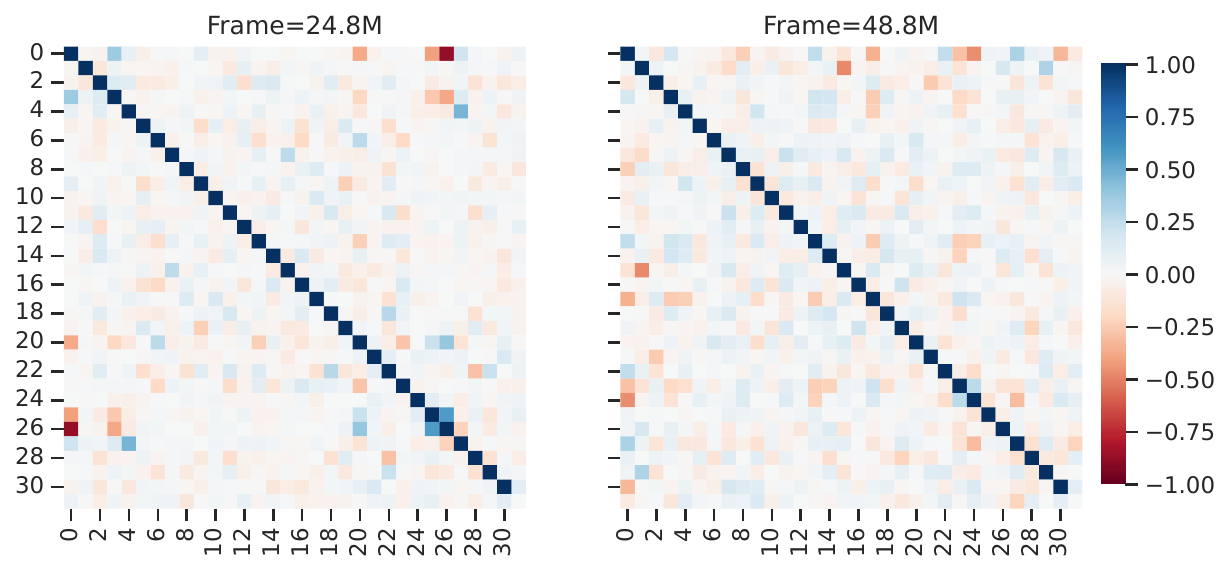}}
\subcaptionbox{$\elephant$}{\includegraphics[width=\figwidthtwo]{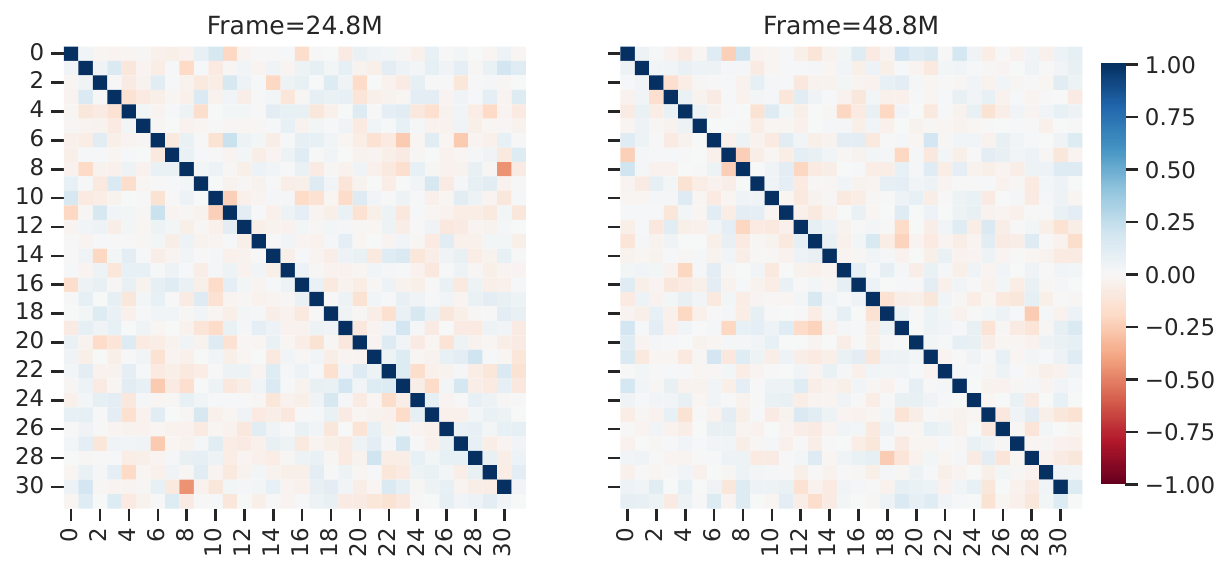}}
\hfill
\caption{Heatmaps of gradient covariance matrices for training DQN with in Q*bert.}
\label{fig:atari_grad_dqn:Qbert}
\hfill \\
\subcaptionbox{$\relu$}{\includegraphics[width=\figwidthtwo]{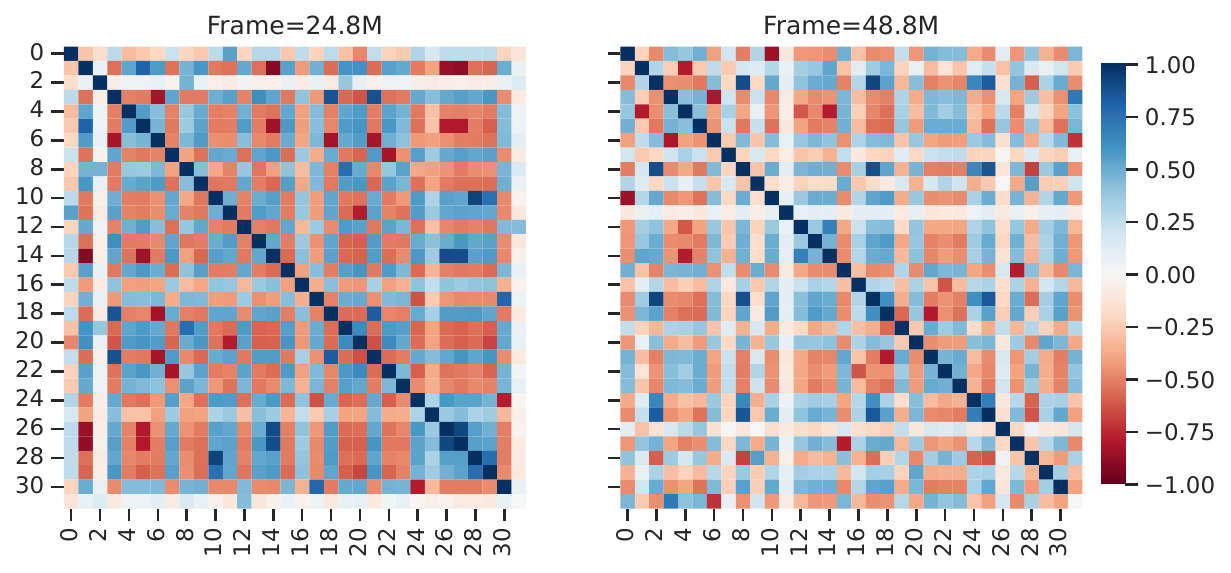}}
\subcaptionbox{$\tanh$}{\includegraphics[width=\figwidthtwo]{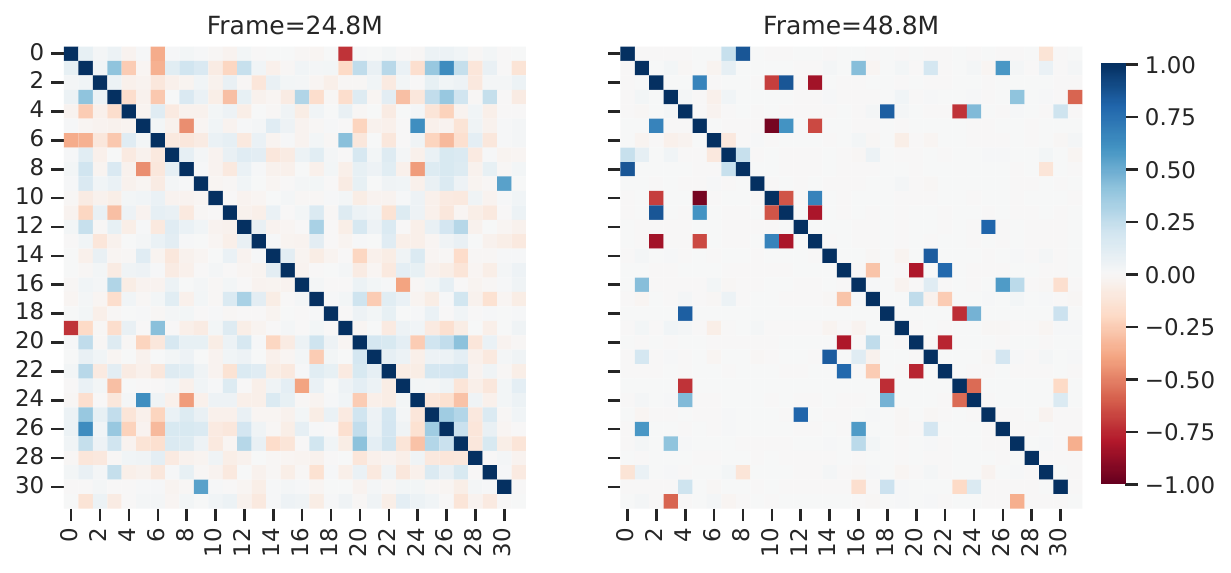}}
\subcaptionbox{$\maxout$}{\includegraphics[width=\figwidthtwo]{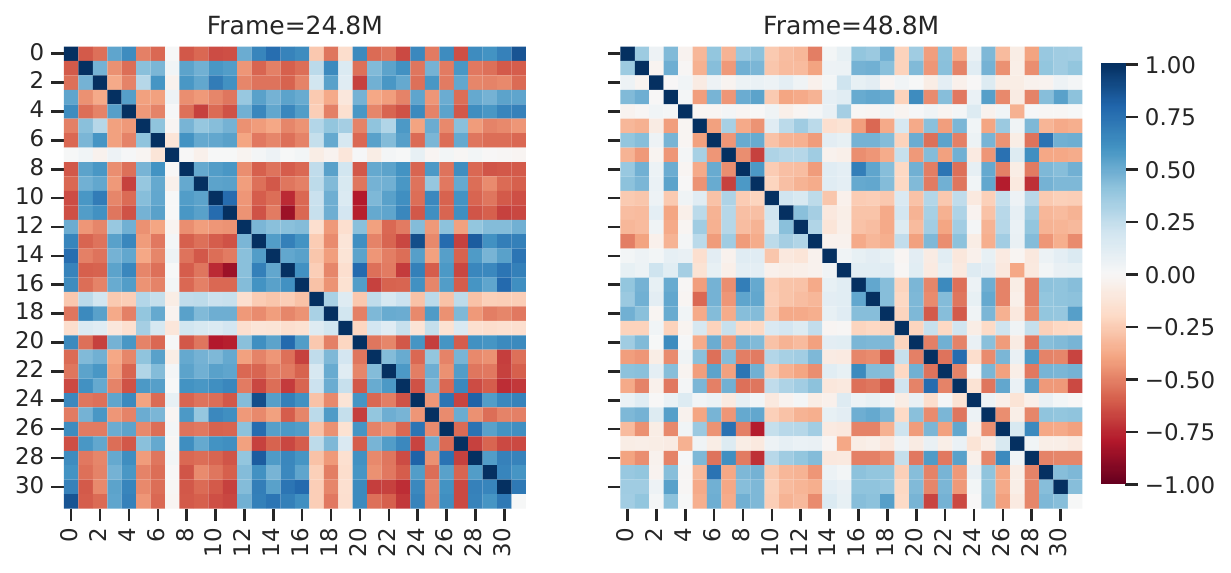}}
\subcaptionbox{$\lwta$}{\includegraphics[width=\figwidthtwo]{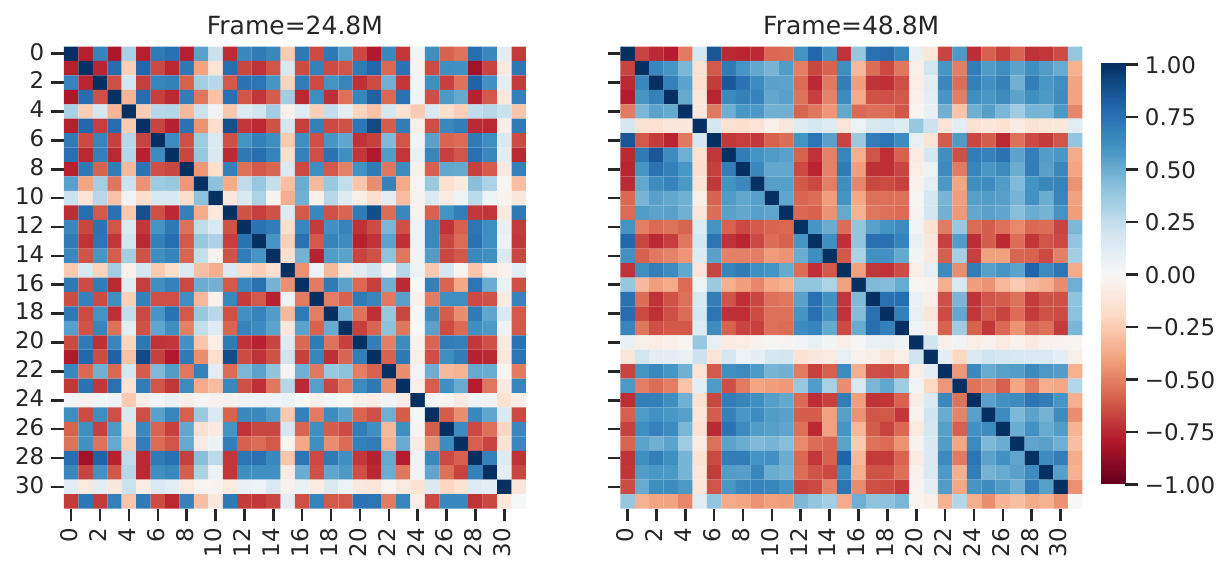}}
\subcaptionbox{$\fta$}{\includegraphics[width=\figwidthtwo]{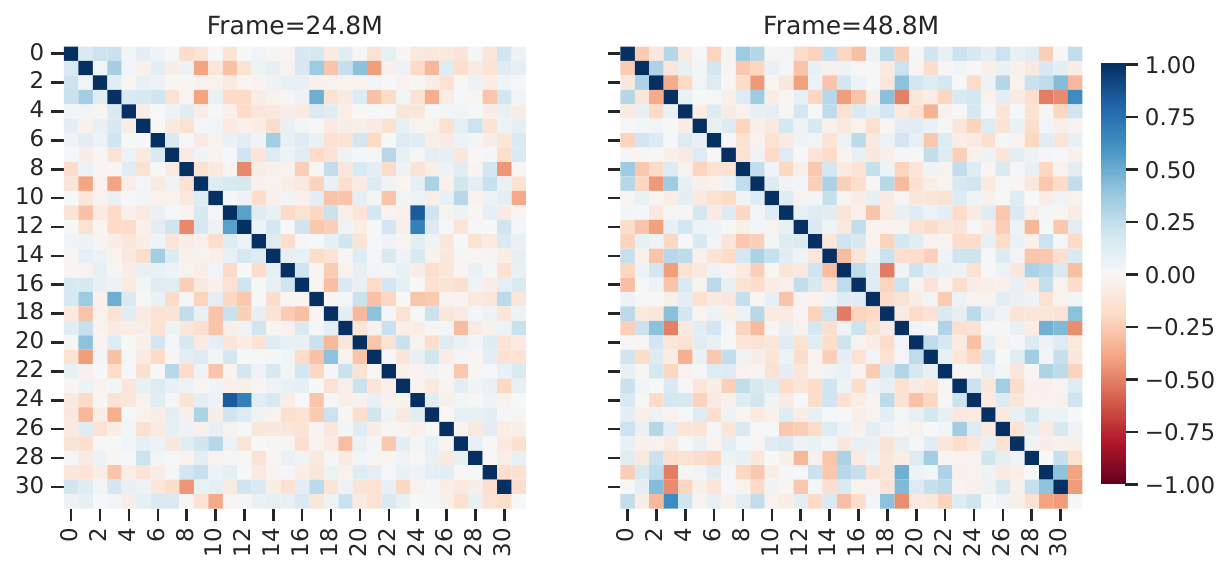}}
\subcaptionbox{$\elephant$}{\includegraphics[width=\figwidthtwo]{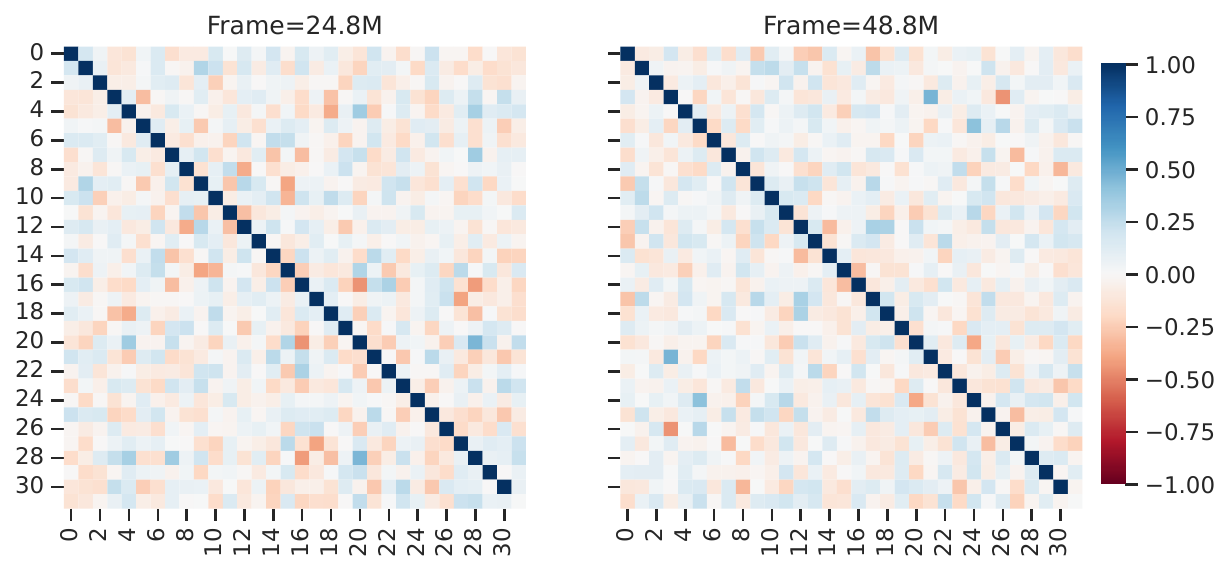}}
\hfill
\caption{Heatmaps of gradient covariance matrices for training DQN with in River Raid.}
\label{fig:atari_grad_dqn:Riverraid}
\end{figure}

\section{Additional Experiments}

\newpage
\subsection{Class Incremental Learning}\label{appendix:clari}

In addition to RL, our method can be applied to classification tasks as well by simply replacing classical activation functions with elephant activation functions in a classification model.
Though our model is agnostic to data distributions, we test it in class incremental learning in order to compare it with previous methods.
Moreover, we adopt a stricter variation of the continual learning setting by adding the following restrictions:
(1) same as streaming learning, each sample only occurs once during training,
(2) task boundaries are not provided or inferred~\citep{aljundi2019task},
(3) neither pre-training nor a fixed feature encoder is allowed~\citep{wolfe2022cold},
and (4) no buffer is allowed to store old task information in any form, such as training samples and gradients.

Surprisingly, we find no methods are designed for or have been tested in the above setting.
As a variant of EWC~\citep{kirkpatrick2017overcoming}, Online EWC~\citep{schwarz2018progress} almost meets these requirements, although it still requires task boundaries.
To overcome this issue, we propose Streaming EWC as one of the baselines, which updates the fisher information matrix after every training sample.
Streaming EWC can be viewed as a special case of Online EWC, treating each training sample as a new task.
Besides Streaming EWC, we consider SDMLP~\citep{bricken2023sparse} and FlyModel~\citep{shen2021algorithmic} as two strong baselines, although they require either task boundary information or multiple data passes. 
Finally, two naive baselines are included, which train MLPs and CNNs without any techniques to reduce forgetting.

We test various methods on several standard datasets --- Split MNIST~\citep{deng2012mnist}, Split CIFAR10~\citep{krizhevsky2009learning}, Split CIFAR100~\citep{krizhevsky2009learning}, and Split Tiny ImageNet~\citep{le2015tiny}.
For both MNIST and CIFAR10, we split them into five sub-datasets, and each of them contains two of the classes.

For SDMLP and FlyModel, we take results from~\citet{bricken2023sparse} directly.
For MLP and CNN, we use $\relu$ by default.
The CNN model consists of a convolution layer, a max pooling operator, and a linear layer.
For our methods, we apply $\elephant$ with $d=4$ in an MLP with one hidden layer and a simple CNN.
The resulting neural networks are named EMLP and ECNN, respectively.

\begin{table}
\vspace{-1em}
\caption{The test accuracy of various methods in class incremental learning on \textit{Split MNIST}. Higher is better. The number of neurons refers to the size of the last hidden layer, i.e., the feature dimension.}
\label{tb:mnist}
\centering
\begin{tabular}{lcccc}
\toprule
Method & Neurons & Dataset Passes & Task Boundary & Test Accuracy \\
\midrule
MLP               &  1K &  1  & \XSolidBrush & 0.665$\pm$0.014 \\
MLP+Streaming EWC &  1K &  1  & \XSolidBrush & 0.708$\pm$0.008 \\
SDMLP             &  1K & 500 & \XSolidBrush & 0.69  \\
FlyModel          &  1K &  1  &  \Checkmark  & \textbf{0.77}  \\
\textbf{EMLP (ours)} &  1K &  1  & \XSolidBrush & 0.723$\pm$0.006 \\
CNN               &  1K &  1  & \XSolidBrush & 0.659$\pm$0.016 \\
CNN+Streaming EWC &  1K &  1  & \XSolidBrush & 0.716$\pm$0.024 \\
\textbf{ECNN (ours)} &  1K &  1  & \XSolidBrush & 0.732$\pm$0.007 \\
ConvMixer            &  1K &  1  & \XSolidBrush & 0.110$\pm$0.003 \\
EConvMixer (ours)    &  1K &  1  & \XSolidBrush & 0.105$\pm$0.003 \\
\midrule
MLP               & 10K &  1  & \XSolidBrush & 0.621$\pm$0.010 \\
MLP+Streaming EWC & 10K &  1  & \XSolidBrush & 0.609$\pm$0.013 \\
SDMLP             & 10K & 500 & \XSolidBrush & 0.53  \\
FlyModel          & 10K &  1  &  \Checkmark  & \textbf{0.91} \\
\textbf{EMLP (ours)} & 10K &  1  & \XSolidBrush & 0.802$\pm$0.002 \\
CNN               & 10K &  1  & \XSolidBrush & 0.769$\pm$0.011 \\
CNN+Streaming EWC & 10K &  1  & \XSolidBrush & 0.780$\pm$0.010 \\
\textbf{ECNN (ours)} & 10K &  1  & \XSolidBrush & 0.850$\pm$0.004 \\
ConvMixer            & 10K &  1  & \XSolidBrush & 0.107$\pm$0.003 \\
EConvMixer (ours)    & 10K &  1  & \XSolidBrush & 0.104$\pm$0.003 \\
\bottomrule
\end{tabular}
\end{table}

\begin{table}
\caption{The test accuracy of various methods in class incremental learning on \textit{Split CIFAR10}. Higher is better. The number of neurons refers to the size of the last hidden layer, i.e., the feature dimension.}
\label{tb:cifar}
\centering
\begin{tabular}{lcc}
\toprule
Method & Neurons & Test Accuracy \\
\midrule
MLP               &  1K & 0.151$\pm$0.008 \\
MLP+Streaming EWC &  1K & 0.158$\pm$0.005 \\
\textbf{EMLP (ours)} &  1K & \textbf{0.197$\pm$0.003} \\
CNN               &  1K & 0.151$\pm$0.001 \\
CNN+Streaming EWC &  1K & 0.147$\pm$0.010 \\
\textbf{ECNN (ours)} &  1K & \textbf{0.192$\pm$0.004} \\
ConvMixer            &  1K & 0.100$\pm$0.0027 \\
EConvMixer (ours)    &  1K & 0.100$\pm$0.0001 \\
\midrule
MLP               & 10K & 0.173$\pm$0.004 \\
MLP+Streaming EWC & 10K & 0.169$\pm$0.003 \\
\textbf{EMLP (ours)} & 10K & \textbf{0.239$\pm$0.002} \\
CNN               & 10K & 0.151$\pm$0.001 \\
CNN+Streaming EWC & 10K & 0.179$\pm$0.007 \\
\textbf{ECNN (ours)} & 10K & \textbf{0.243$\pm$0.002} \\
ConvMixer            & 10K & 0.102$\pm$0.0015 \\
EConvMixer (ours)    & 10K & 0.100$\pm$0.0001 \\
\bottomrule
\end{tabular}
\end{table}

\begin{figure}[tbp]
\centering
\subcaptionbox{Split MNIST}{
\includegraphics[width=\figwidthtwo]{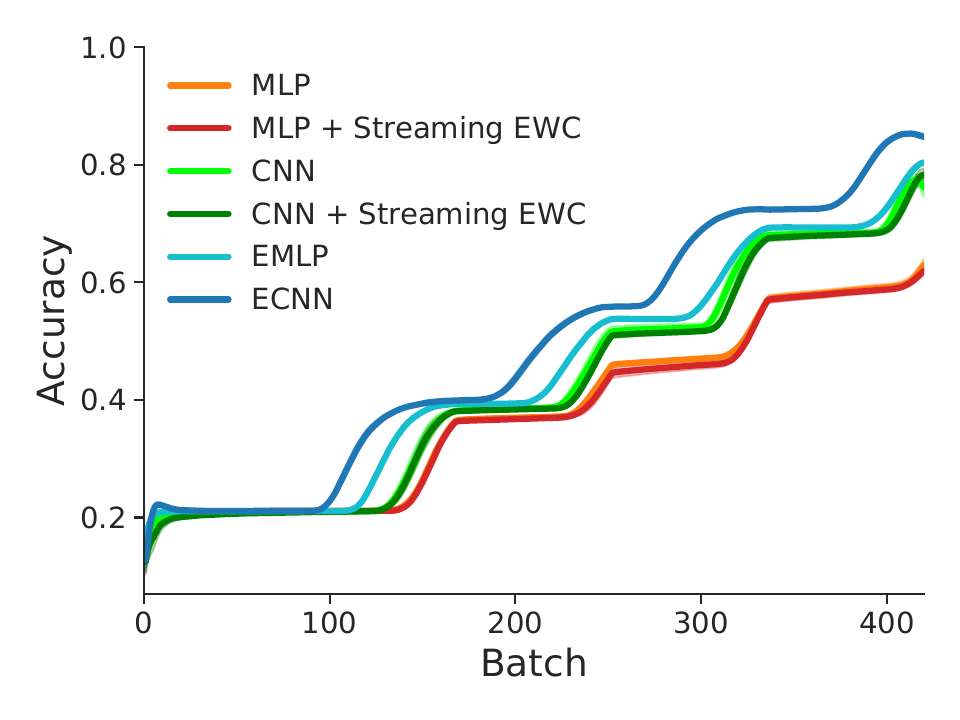}}
\subcaptionbox{Split CIFAR10}{
\includegraphics[width=\figwidthtwo]{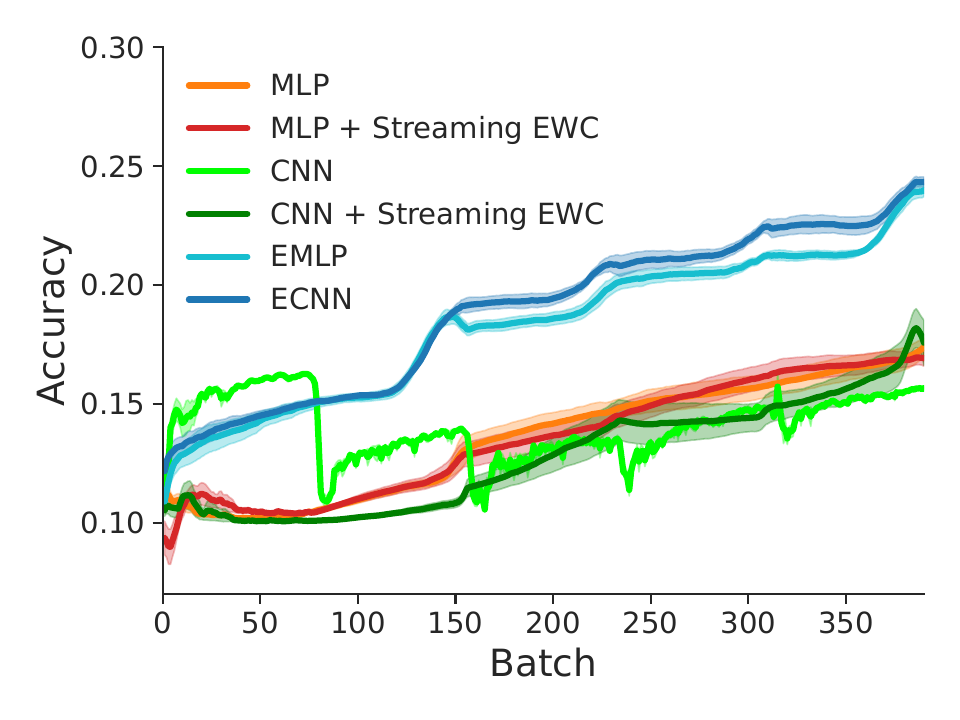}}
\caption{The test accuracy curves during training. The x-axis shows the number of training mini-batches. All results are averaged over $5$ runs, and the shaded regions represent standard errors.}
\label{fig:mnist_clari}
\end{figure}

The test accuracy is used as the performance metric.
The performance summaries of different methods on Split MNIST and Split CIFAR10 are shown in~\cref{tb:mnist} and~\cref{tb:cifar}, correspondingly. 
In~\cref{fig:mnist_clari}, we plot the test accuracy curves during training on Split MNIST and Split CIFAR10 when the number of neurons is $10K$.
All results are averaged over $5$ runs, reported with standard errors.
The number of neurons refers to the size of the last hidden layer in a neural network, i.e., the feature dimension.

Besides simple CNNs, we test ConvMixer~\citep{trockman2022patches} which can achieve around $92.5\%$ accuracy in just $25$ epochs in classical supervised learning~\footnote{\url{https://github.com/locuslab/convmixer-cifar10}}.
However, we find that ConvMixer completely fails in class incremental learning setting, as shown in~\cref{tb:mnist} and~\cref{tb:cifar}.
Same as~\citet{mirzadeh2022architecture}, the results show that using more advanced models does not necessarily lead to better performance in class incremental learning.
We also tried CNNs with more convolution layers but found the performance was worse and worse as we increased the number of convolution layers.
Overall, FlyModel performs the best, although it requires additional task boundary information.
Our methods (EMLP and ECNN) are the second best without utilizing task boundary information by training for a single pass.
Other findings are summarized in the following, reaffirming the findings in~\citet{mirzadeh2022wide,mirzadeh2022architecture}:
\begin{itemize}[leftmargin=1em]
\item Wider neural networks forget less by using more neurons.
\item Neural network architectures can significantly impact learning performance: CNN (ECNN) is better than MLP (EMLP), especially with many neurons.
\item Architectural improvement could be larger than algorithmic improvement: using a better architecture (e.g., EMLP and ECNN) is more beneficial than incorporating Streaming EWC.
\end{itemize}

Overall, we conclude that applying elephant activation functions significantly reduces forgetting and boosts performance in class incremental learning under strict constraints.

\paragraph{Combining with Pre-Training}
To show that our method can be combined with pre-training, we relax our experimental assumptions by adding the pre-training technique.
Specifically, for CIFAR10 and CIFAR100, we use 256-dimensional latent embeddings, which are provided by~\citet{bricken2023sparse}, taken from the last layer of a frozen ConvMixer~\citep{trockman2022patches} that is pre-trained on ImageNet32~\citep{chrabaszcz2017downsampled}.
The corresponding embedding datasets are called Embedding CIFAR10 and Embedding CIFAR100, respectively.
We split Embedding CIFAR10 into $5$ sub-datasets while Embedding CIFAR100 is split into $50$ sub-datasets; each sub-dataset contains two of the classes.
We test different methods on the two datasets and present results in~\cref{tb:embedcifar10} and~\cref{tb:embedcifar100}.
Note that the the results of SDMLP and FlyModel are from~\citet{bricken2023sparse}.

For Tiny ImageNet, we use 768-dimensional latent embeddings, which are taken from the last layer of a frozen ConvMixer~\citep{trockman2022patches} that is pre-trained on ImageNet-1k~\citep{imagenet15russakovsky}, provided by Hugging Face~\footnote{\url{https://huggingface.co/timm/convmixer_768_32.in1k}}.
The corresponding embedding dataset is called Embedding Tiny ImageNet.
We then split Embedding Tiny ImageNet into 100 sub-datasets; each of them contains two of the classes.
The experimental results are presented in~\cref{tb:embedtiny}.

Overall, the performance of EMLP is greatly improved with the help of pre-training.
For example, the test accuracy on CIFAR10 is boosted from 25\% to 75\%.
Moreover, EMLP still outperforms MLP significantly, showing that our method can be combined with common practices in class incremental learning, such as pre-training.

\begin{table}
\caption{The test accuracy of various methods in class incremental learning on \textit{Split Embedding CIFAR10}, averaged over $5$ runs. Higher is better. Standard errors are also reported.}
\label{tb:embedcifar10}
\centering
\begin{tabular}{lcccc}
\toprule
Method & Neurons & Dataset Passes & Task Boundary & Test Accuracy \\
\midrule
MLP               &  1K &  1  & \XSolidBrush & 0.662$\pm$0.006 \\
SDMLP             &  1K & 2000 & \XSolidBrush & 0.56  \\
FlyModel          &  1K &  1  &  \Checkmark  & 0.69  \\
\textbf{EMLP (ours)} &  1K &  1  & \XSolidBrush & \textbf{0.726$\pm$0.008} \\
\midrule
MLP               & 10K &  1  & \XSolidBrush & 0.697$\pm$0.002 \\
SDMLP             & 10K & 2000 & \XSolidBrush & 0.77  \\
FlyModel          & 10K &  1  &  \Checkmark  & \textbf{0.82} \\
\textbf{EMLP (ours)} & 10K &  1  & \XSolidBrush & 0.755$\pm$0.002 \\
\bottomrule
\end{tabular}
\end{table}

\begin{table}
\caption{The test accuracy of various methods in class incremental learning on \textit{Split Embedding CIFAR100}, averaged over $5$ runs. Higher is better. Standard errors are also reported.}
\label{tb:embedcifar100}
\centering
\begin{tabular}{lcccc}
\toprule
Method & Neurons & Dataset Passes & Task Boundary & Test Accuracy \\
\midrule
MLP               &  1K &  1  & \XSolidBrush & 0.157$\pm$0.001 \\
SDMLP             &  1K & 500 & \XSolidBrush & 0.39  \\
FlyModel          &  1K &  1  &  \Checkmark  & 0.36  \\
\textbf{EMLP (ours)} &  1K &  1  & \XSolidBrush & \textbf{0.391$\pm$0.003} \\
\midrule
MLP               & 10K &  1  & \XSolidBrush & 0.425$\pm$0.001 \\
SDMLP             & 10K & 500 & \XSolidBrush & 0.43  \\
FlyModel          & 10K &  1  &  \Checkmark  & \textbf{0.58} \\
\textbf{EMLP (ours)} & 10K &  1  & \XSolidBrush & 0.449$\pm$0.002 \\
\bottomrule
\end{tabular}
\end{table}

\begin{table}
\caption{The test accuracy of various methods in class incremental learning on \textit{Split Embedding Tiny ImageNet}, averaged over $5$ runs. Higher is better. Standard errors are also reported.}
\label{tb:embedtiny}
\centering
\begin{tabular}{lcccc}
\toprule
Method & Neurons & Dataset Passes & Task Boundary & Test Accuracy \\
\midrule
MLP               &  10K &  1  & \XSolidBrush & \textbf{0.249$\pm$0.003} \\
\textbf{EMLP (ours)} &  10K &  1  & \XSolidBrush & 0.241$\pm$0.002 \\
\midrule
MLP               & 100K &  1  & \XSolidBrush & 0.264$\pm$0.002 \\
\textbf{EMLP (ours)} & 100K &  1  & \XSolidBrush & \textbf{0.350$\pm$0.002} \\
\bottomrule
\end{tabular}
\end{table}

\paragraph{Hyper-parameter Settings}
We list the (swept) hyper-parameters for Split MNIST and Split CIFAR10 in~\cref{hyper:mnist_cifar}.
The (swept) hyper-parameters for Split Embedding CIFAR10, Split Embedding CIFAR100, and Split Embedding Tiny ImageNet are listed in~\cref{hyper:embed_cifar} and \cref{hyper:embed_tiny}.
Among them, $d$, $a$, and $\sigma_{bias}$ are specific hyper-parameters for ENNs, where $\gamma$ and $\lambda$ are two hyper-parameters used in (streaming) EWC~\citep{kirkpatrick2017overcoming}.
For each mini-batch data, we do $E$ optimization steps.

\begin{table}[htbp]
\caption{The (swept) hyper-parameters for \textit{Split MNIST} and \textit{Split CIFAR10}.}
\label{hyper:mnist_cifar}
\centering
\begin{tabular}{lc}
\toprule
Hyper-parameter & Value \\
\midrule
Optimizer & RMSProp with decay=$0.999$ \\
learning rate & $\{3e-6, 1e-6, 3e-7, 1e-7, 3e-8, 1e-8\}$ \\
mini-batch size & $125$ \\
$E$ & $\{1, 2\}$ \\
$d$ & $4$ \\
$a$ & $\{0.02, 0.04, 0.08, 0.16, 0.32\}$ \\
$\sigma_{bias}$ & $\{0.04, 0.08, 0.16, 0.32, 0.64\}$ \\
EWC $\gamma$ & $\{0.5, 0.8, 0.9, 0.95, 0.99, 0.999\}$ \\
EWC $\lambda$ & $\{1e1, 1e2, 1e3, 1e4, 1e5\}$ \\
\bottomrule
\end{tabular}
\end{table}

\begin{table}[htbp]
\caption{The (swept) hyper-parameters for \textit{Split Embedding CIFAR10} and \textit{Split Embedding CIFAR100}.}
\label{hyper:embed_cifar}
\centering
\begin{tabular}{lc}
\toprule
Hyper-parameter & Value \\
\midrule
Optimizer & RMSProp with decay=$0.999$ \\
learning rate & $\{1e-4, 3e-5, 1e-5, 3e-6, 1e-6, 3e-7, 1e-7\}$ \\
mini-batch size & $125$ \\
$E$ & $\{1, 2, 4\}$ \\
$d$ & $4$ \\
$a$ & $\{0.02, 0.04, 0.08, 0.16, 0.32\}$ \\
$\sigma_{bias}$ & $\{0.04, 0.08, 0.16, 0.32, 0.64\}$ \\
\bottomrule
\end{tabular}
\end{table}

\begin{table}[htbp]
\caption{The (swept) hyper-parameters for \textit{Split Embedding Tiny ImageNet}.}
\label{hyper:embed_tiny}
\centering
\begin{tabular}{lc}
\toprule
Hyper-parameter & Value \\
\midrule
Optimizer & RMSProp with decay=$0.999$ \\
learning rate & $\{1e-5, 3e-6, 1e-6, 3e-7, 1e-7\}$ \\
mini-batch size & $250$ \\
$E$ & $\{2, 4, 8\}$ \\
$d$ & $4$ \\
$a$ & $\{0.08, 0.16, 0.32, 0.64\}$ \\
$\sigma_{bias}$ & $\{0.32, 0.64, 1.28, 2.56\}$ \\
\bottomrule
\end{tabular}
\end{table}

\subsection{Policy Gradient Methods}\label{appendix:pg}

In this section, we test $\elephant$ by incorporating it into policy gradient methods.
We consider other activation functions as baselines, such as $\relu$, $\tanh$, $\maxout$, $\lwta$, and $\fta$.
Different activation functions are tested for 10 runs in 6 MuJoCo tasks~\citep{towers2023gymnasium}: HalfCheetah-v4, Hopper-v4, Walker2d-v4, Ant-v4, Reacher-v4, and Swimmer-v4.

We consider two representative policy gradient algorithms --- proximal policy optimization (PPO)~\citep{schulman2017proximal} and soft actor-critic (SAC)~\citep{haarnoja2018soft}.
Specifically, PPO is an on-policy policy gradient method and SAC is an off-policy policy gradient method.
We implement PPO and SAC with Jax~\citep{jax2018github} from scratch based on popular open-sourced implementations~\citep{jaxrl,huang2022cleanrl}.

For PPO, the default network is an MLP with one hidden layer of size $1,000$ in all tasks.
Following~\citet{huang2022cleanrl}, we apply Adam~\citep{tieleman2012rmsprop} with linearly decreasing learning rate starting from $1e-4$ to $0$ and global gradient norm $0.5$.
We set the number of the collected trajectory steps to $2,048$, which is also the buffer size in PPO. 
The discount factor $\gamma=0.99$. The mini-batch size is $64$.
For $\elephant$, for a fair comparison, we use the same hyper-parameters ($d=4$, $h=1$, $a=0.8$, and $\sigma_{bias}=1$) for all MuJoCo tasks, although tuning them for each task could further improve the performance.
For $\maxout$ and $\lwta$, we set $k=5$.
For $\fta$, we set $k=20$ and $[l,u]=[-20,20]$ following~\citet{pan2020fuzzy}.
Similar to the DQN case, we adjust the width of the hidden layer so that the total number of parameters is close to each other when different activation functions are applied.

For SAC, the default network is an MLP with hidden layers $[256, 256]$ in all tasks.
We apply Adam~\citep{tieleman2012rmsprop} with learning rate $1e-3$ for optimization.
For buffer sizes, we consider values in $\{1e4, 1e5, 1e6\}$ where the default buffer size is $1e6$.
The discount factor $\gamma=0.99$. The mini-batch size is $256$.
For $\elephant$, for a fair comparison, we use the same hyper-parameters ($d=4$, $h=1$, $a=1$, and $\sigma_{bias}=2$) for all MuJoCo tasks, although tuning them for each task and buffer size could further improve the performance.
For $\maxout$ and $\lwta$, we set $k=4$.
Following~\citet{pan2020fuzzy}, when apply $\fta$ only to the last hidden layer; and we set $k=20$ and $[l,u]=[-20,20]$.
We also adjust the width of the hidden layer so that the total number of parameters is close to each other when different activation functions are applied.

\subsubsection{Evaluation of Memory Efficiency}

First, we test SAC under various buffer sizes.~\footnote{The buffer size in PPO is already relatively small, hence it is not further tested.}
The agents' performance of different activation functions and buffer sizes is shown in~\cref{fig:mujoco_buffer}.
All results are averaged over $10$ runs and the shaded areas represent standard errors. 
There is no clear winner in general, and $\elephant$ either matches or outperforms $\relu$ in most tasks.
Notably, $\elephant$ shows its strong robustness to buffer sizes in Swimmer-v4 --- while other activations are negatively impacted by size changes, its performance stays high across buffer sizes.

\begin{figure}[htbp]
\centering
\begin{minipage}[c]{\textwidth}
    \centering
    \includegraphics[width=\figwidthone]{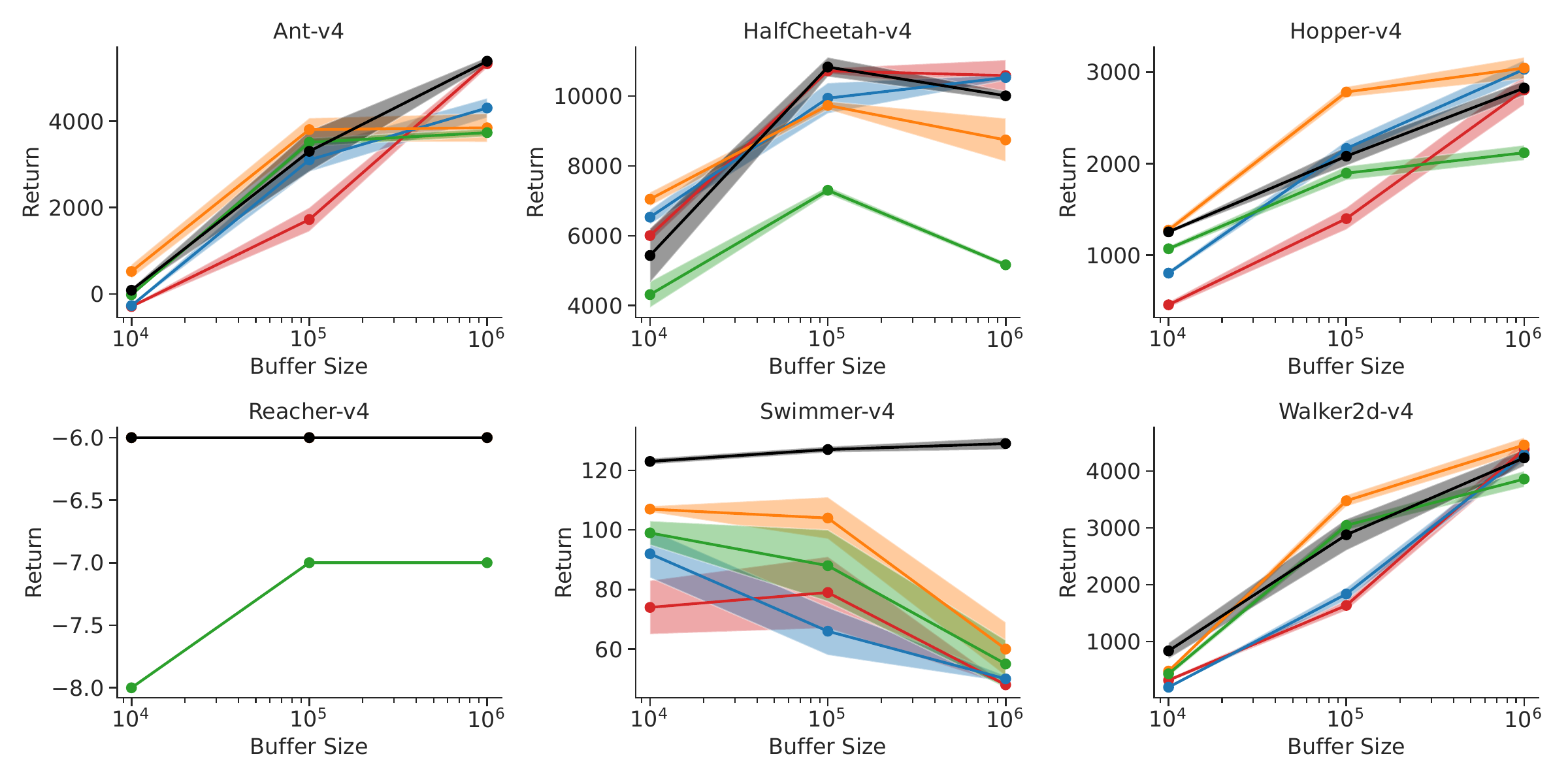}
    \includegraphics[width=0.55\textwidth]{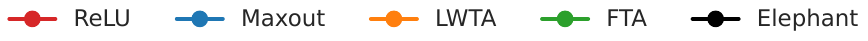}
\end{minipage}
\caption{The performance of DQN and SAC with different activation functions under various buffer sizes, measured as the average return of the last 10\% episodes. In Reacher-v4, the curves of $\maxout$, $\lwta$, and $\relu$ are presented as well but covered by $\elephant$.}
\label{fig:mujoco_buffer}
\end{figure}

\subsubsection{Evaluation of Sample Efficiency}

We then test activation functions in PPO and SAC with a default (large) buffer, shown in \cref{fig:mujoco}.
All results are averaged over $10$ runs and the shaded areas represent standard errors. 
For PPO, $\elephant$ achieves better performance than the default activation $\tanh$ in most tasks.
SAC with the default activation $\relu$ is already quite good, achieving similar or better performance than most other activations in most tasks.
Specifically, the results between $\elephant$ and $\relu$ are quite mixed: $\elephant$ matches the performance of $\relu$ in Hopper-v4, Walker2d-v4, Ant-v4, and Reacher-v4, slightly underperforms $\relu$ in HalfCheetah-v4, and significantly outperforms $\relu$ in Swimmer-v4.

\begin{figure}[htbp]
\centering
\subcaptionbox{PPO in 6 MuJoCo tasks \vspace{1em}}{
\begin{minipage}{\textwidth}
    \centering
    \includegraphics[width=\figwidthone]{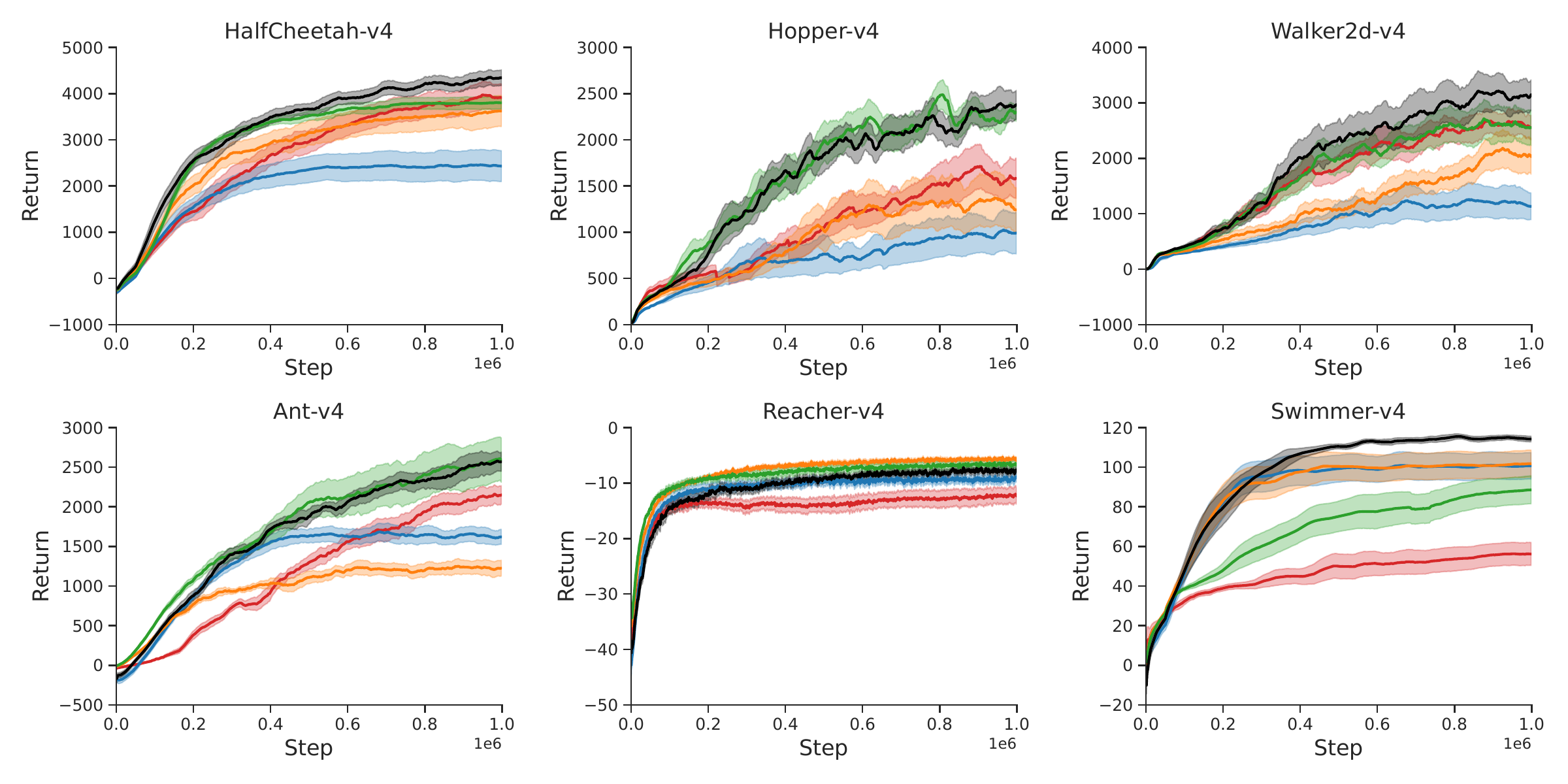}
    \includegraphics[width=0.55\textwidth]{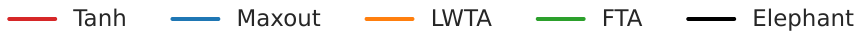}
\end{minipage}
}
\subcaptionbox{SAC in 6 MuJoCo tasks}{
\begin{minipage}{\textwidth}
    \centering
    \includegraphics[width=\figwidthone]{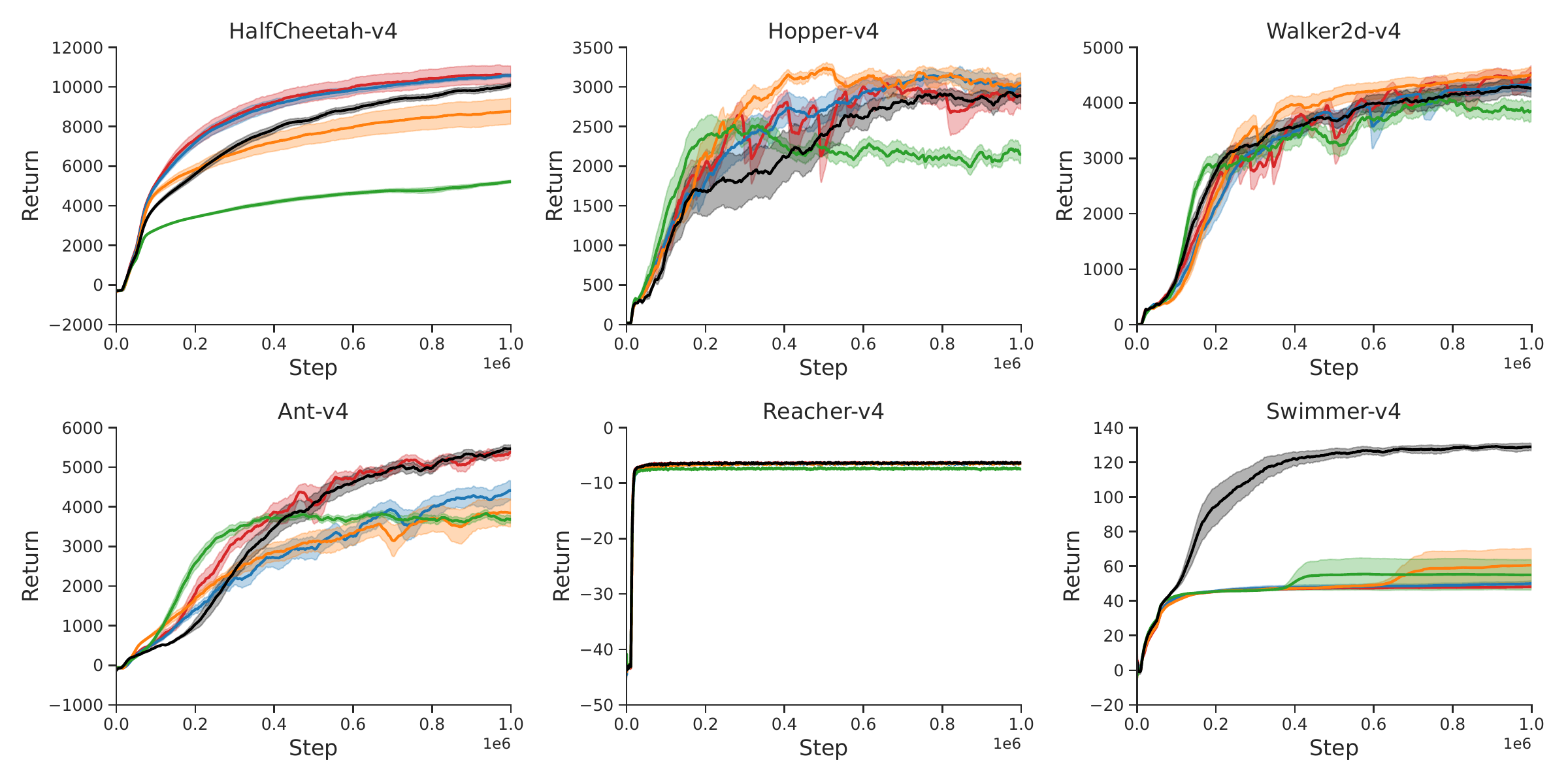}
    \includegraphics[width=0.55\textwidth]{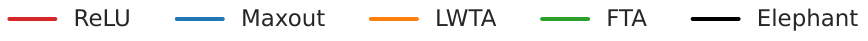}
\end{minipage}
}
\caption{The return curves of PPO and SAC in 6 MuJoCo tasks with different activation functions. A default large buffer is used. All results are averaged over $10$ runs and the shaded areas represent standard errors.}
\label{fig:mujoco}
\end{figure}

\subsubsection{Summary}

In summary, $\elephant$ does not show significant advantage over other activation functions for policy gradient methods, especially in SAC experiments --- it performs similar to $\relu$ in most tasks.
Notably, these results are consistent with the findings from~\citet{ceron2024value}, who report that increasing network sparsity when training PPO and SAC in MuJoCo tasks does not lead to performance improvement in most cases.
We leave a deeper investigation of this phenomenon for future work.

\subsection{Evaluation in A Real-Robot Task}

Finally, we test $\elephant$ in a real-robot task---UR-Reacher-2, introduced by~\cite{mahmood2018benchmarking}.
This task resembles the Reacher-v4 task in MuJoCo using a UR5 robot.
Specifically, the robot controls the angular speeds of the second and third joints from the base, within a range of $[-0.3, +0.3]$ rad/s.
The observation vector comprises joint angles, joint velocities, the previous action, and the vector difference between the target and the fingertip coordinates.
The goal of the learning agent is to reach arbitrary target positions on a 2D plane.
The experiment is conducted on a workstation with an AMD Ryzen Threadripper 2950 processor, an NVidia 2080Ti GPU, and 128GB memory.

Our implementation of SAC for the robot reacher task is adapted from~\cite{pytorch_sac}.
The networks are MLPs with hidden layers $[256, 256]$.
For every $2$ steps, we apply Adam~\citep{tieleman2012rmsprop} with learning rate $3e-4$ for optimization.
The discount factor $\gamma=0.99$.
The mini-batch size is $64$.
For $\elephant$, we set $d=4$, $h=1$, $a=1$, and $\sigma_{bias}=1.6$.

We apply SAC in this task and compare the performance between $\elephant$ and the default activation function $\relu$.
Moreover, we consider two different buffer sizes---a big buffer with size $200K$ and a small buffer with size $2K$.
All results are averaged over $5$ runs, and the shaded areas represent standard errors, as shown in~\cref{fig:robot}.
We observe that $\elephant$ achieves a similar performance as $\relu$ when a big buffer is used, while $\elephant$ achieves a higher sample efficiency than $\relu$ when the buffer is small.

\begin{figure}[htbp]
\centering
\includegraphics[width=0.7\textwidth]{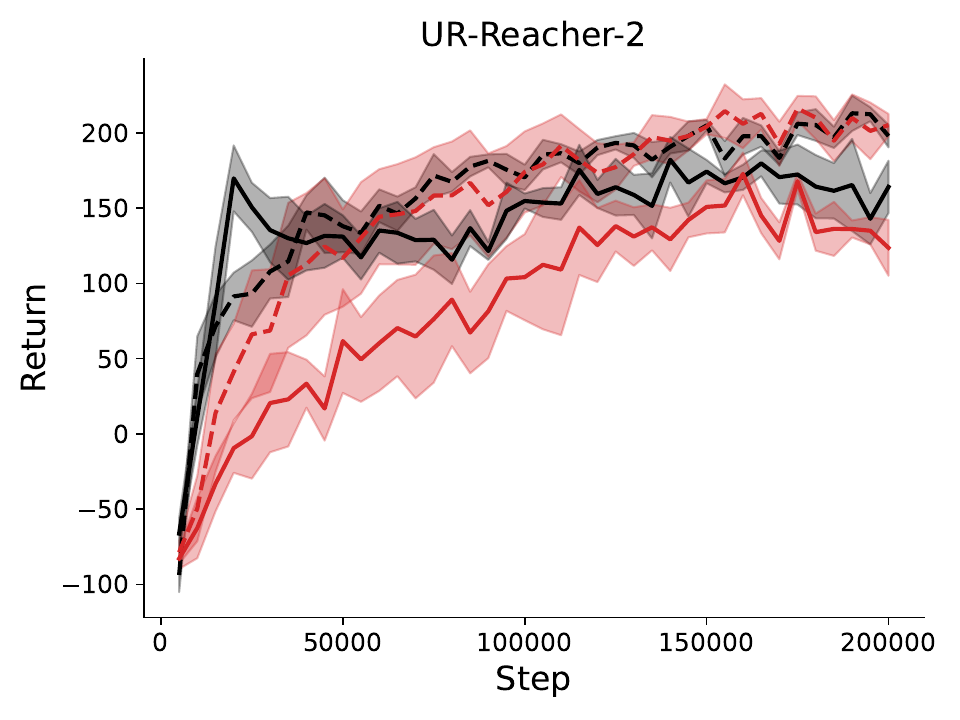}
\includegraphics[width=0.7\textwidth]{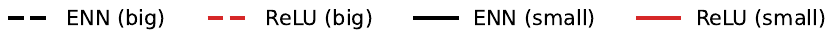}
\caption{The performance of SAC in the robot task UR-Reacher-2 with big and small buffers.}
\label{fig:robot}
\end{figure}

\end{document}